\documentclass[10pt, a4paper, twoside]{article}
\usepackage[utf8]{inputenc}
\usepackage[english]{babel}
\usepackage{csquotes}
\usepackage{amsmath}
\usepackage{amssymb}
\usepackage{amsthm}
\usepackage{amsfonts}
\usepackage{graphicx}
    \graphicspath{ {./figures/} }
\usepackage[left=2.5cm, right=2.5cm, top=3cm, bottom=3cm, twoside]{geometry}
\usepackage{float}
\usepackage{subcaption}
\usepackage[labelfont=bf,font=small,labelsep=space]{caption}
\usepackage{sourcecodepro}
\usepackage{sourcesanspro}
\usepackage{sourceserifpro}
\usepackage{xfrac}

\usepackage{xcolor}
\usepackage[notextcomp]{stix}
    
\usepackage{bm}
	
\usepackage[mathcal]{eucal}
\usepackage{tikz}
    \usetikzlibrary{cd}
    \tikzstyle{every picture}+=[remember picture]
\usepackage{enumitem}
\usepackage[many]{tcolorbox}
\usepackage{hyperref}
    \hypersetup{
    colorlinks=true,
    linkcolor=blue,
    filecolor=magenta,      
    urlcolor=blue,
    citecolor=violet,
    breaklinks=true,
    }
\usepackage{enumitem}
    \setlist{nosep}
\usepackage{fancyhdr}
\usepackage[toc,title]{appendix}

\usepackage[backend=bibtex, sorting=none, maxbibnames=99]{biblatex}
\newtheorem{theorem}{Theorem}[section]
\newtheorem{proposition}[theorem]{Proposition}
\newtheorem{lemma}[theorem]{Lemma}
\newtheorem{corollary}[theorem]{Corollary}

\theoremstyle{definition}
\newtheorem{definition}[theorem]{Definition}

\theoremstyle{remark}
\newtheorem{remark}[theorem]{Remark}

\tcolorboxenvironment{definition}{
	blanker,
	before skip=\topsep,
	after skip=\topsep,
	enhanced jigsaw,
	colback=black!5!white,
	left=5pt,
	right=5pt, 
	top=5pt,
	bottom=5pt,
	rounded corners=all,
	arc=0pt,
	outer arc=0pt
}

\tcolorboxenvironment{remark}{
	blanker,
	before skip=1.5em,
	after skip=1.5em,
	borderline west={1.5pt}{1.5pt}{black},
	breakable,
	left=8pt,
	right=8pt, 
}

\DeclareMathOperator*{\argmin}{arg\,min}

\DeclareMathOperator*{\cut}{cut}

\DeclareMathOperator{\Lip}{Lip}
\DeclareMathOperator{\Len}{Len}
\renewcommand\d[1]{\mathop{}\!\mathit{d}\nobreak\hspace{-0.1em}#1}

\DeclareMathOperator{\morph}{\scalebox{0.7}{\ensuremath\square}}

\providecommand{\keywords}[1]
{
  \small	
  \textbf{\textit{Keywords---}} #1
}

\makeatletter
\newcommand*{\bigcdot}{}
\DeclareRobustCommand*{\bigcdot}{%
  \mathbin{\mathpalette\bigcdot@{}}%
}
\newcommand*{\bigcdot@scalefactor}{.7}
\newcommand*{\bigcdot@widthfactor}{1.15}
\newcommand*{\bigcdot@}[2]{%
  \sbox0{$#1\vcenter{}$}
  \sbox2{$#1\cdot\m@th$}%
  \hbox to \bigcdot@widthfactor\wd2{%
    \hfil
    \raise\ht0\hbox{%
      \scalebox{\bigcdot@scalefactor}{%
        \lower\ht0\hbox{$#1\bullet\m@th$}%
      }%
    }%
    \hfil
  }%
}
\makeatother

\newcommand{\tikznode}[3][inner sep=0pt]{\tikz[remember picture,baseline=(#2.base)]{\node(#2)[#1]{$#3$};}}

\newcommand{\footremember}[2]{%
    \footnote{#2}
    \newcounter{#1}
    \setcounter{#1}{\value{footnote}}%
}
\newcommand{\footrecall}[1]{%
    \footnotemark[\value{#1}]%
} 

\newcommand{\R}{\mathbb{R}}
\title{\vspace{-2em}PDE-based Group Equivariant Convolutional Neural Networks}

\author{Bart M.N. Smets\footremember{casa}{CASA, Department of Mathematics and Computer Science, Eindhoven University of Technology, email: \href{mailto:b.m.n.smets@tue.nl}{b.m.n.smets@tue.nl}} \and Jim Portegies\footrecall{casa} \and Erik J. Bekkers\footnote{Machine Learning Lab, Informatics Institute, University of Amsterdam} \and Remco Duits\footrecall{casa}}

\date{\today}

\begin{document}

\captionsetup{width=0.7\linewidth}

\maketitle
\vspace{-1em}


\begin{abstract}
We present a PDE-based framework that generalizes Group equivariant Convolutional Neural Networks (G-CNNs). In this framework, a network layer is seen as a set of PDE-solvers where geometrically meaningful PDE-coefficients become the layer's trainable weights. Formulating our PDEs on homogeneous spaces allows these networks to be designed with built-in symmetries such as rotation in addition to the standard translation equivariance of CNNs.

Having all the desired symmetries included in the design obviates the need to include them by means of costly techniques such as data augmentation. We will discuss our PDE-based G-CNNs (PDE-G-CNNs) in a general homogeneous space setting while also going into the specifics of our primary case of interest: roto-translation equivariance.
    
We solve the PDE of interest by a combination of linear group convolutions and non-linear morphological group convolutions with analytic kernel approximations that we underpin with formal theorems. Our kernel approximations allow for fast GPU-implementation of the PDE-solvers, we release our implementation with this article in the form of the LieTorch extension to PyTorch, available at \url{https://gitlab.com/bsmetsjr/lietorch}. Just like for linear convolution a morphological convolution is specified by a kernel that we train in our PDE-G-CNNs. 
In PDE-G-CNNs we do not use non-linearities such as max/min-pooling and ReLUs as they are already subsumed by morphological convolutions.
    
We present a set of experiments to demonstrate the strength of the proposed PDE-G-CNNs in increasing the performance of deep learning based imaging applications with far fewer parameters than traditional CNNs.


\end{abstract}


\section{Introduction}

In this work we introduce \emph{PDE-based Group CNNs}. The key idea is to replace the typical trifecta of convolution, pooling and ReLUs found in CNNs with a Hamilton-Jacobi type evolution PDE, or more accurately a solver for a Hamilton-Jacobi type PDE. This substitution is illustrated in Fig.~\ref{fig:traditional} where we retain (channel-wise) affine combinations as the means of composing feature maps.

The PDE we propose to use in this setting comes from the geometric image analysis world \cite{welk2019pde,Fadili2015,Peyre2010,Kimmel2015,Burger2016,Cremers2016,Weickert2016,SapiroBook,Sethian,WeickertEE,morel1995variational,duits2007scale}. It was chosen based on the fact that it exhibits similar behaviour on images as traditional CNNs do through convolution, pooling and ReLUs. Additionally it can be formulated on Lie groups to yield equivariant processing, which makes our PDE approach compatible with Group CNNs \cite{cohen2016group,dieleman2016exploiting,dieleman2015rotation,winkels20183d,worrall2018cubenet,bekkers2017template,oyallon2015deep,bekkers2018roto,weiler2018learning,cohen2019general,worrall2017harmonic,kondor2018generalization,esteves2018learning}. Finally an approximate solver for our PDE can be efficiently implemented on modern highly parallel hardware, making the choice a practical one as well.

Our solver uses the operator splitting method to solve the PDE in a sequence of steps, each step corresponding to a term of the PDE. The sequence of steps for our PDE is illustrated in Fig.~\ref{fig:cdde-layer}. The morphological convolutions that are used to solve for the non-linear terms of the PDE are a key aspect of our design. Normally, morphological convolutions are considered on $\mathbb{R}^d$ \cite{akian1994bellman,burgeth2004morphological}, but when extended to Lie groups such as $SE(d)$ they have many benefits in applications (e.g. crossing-preserving flow \cite{duits2013morphological} or tracking \cite{bekkers2015pde,haije2014sharpening}). Using morphological convolutions allows our network to have trainable non-linearities instead of the fixed non-linearities in (G-)CNNs.

The theoretical contribution of this paper consists of providing good analytical approximations to the kernels that go in the linear and morphological convolutions that solve our PDE. 
On $\mathbb{R}^n$ the formulation of these kernels is reasonably straightforward, but in order to achieve group equivariance we need to generalize them on homogeneous spaces.

Instead of training kernel weights our goal is training the coefficients of the PDE. The coefficients of our PDE have the benefit of yielding geometrically meaningful parameters from a image analysis point of view. Additionally we will need (much) less PDE parameters than kernel weights to achieve a given level of performance in image segmentation and classification tasks; arguably the greatest benefit of our approach.

\begin{figure}[ht]
    \centering
    \includegraphics[width=0.5\linewidth]{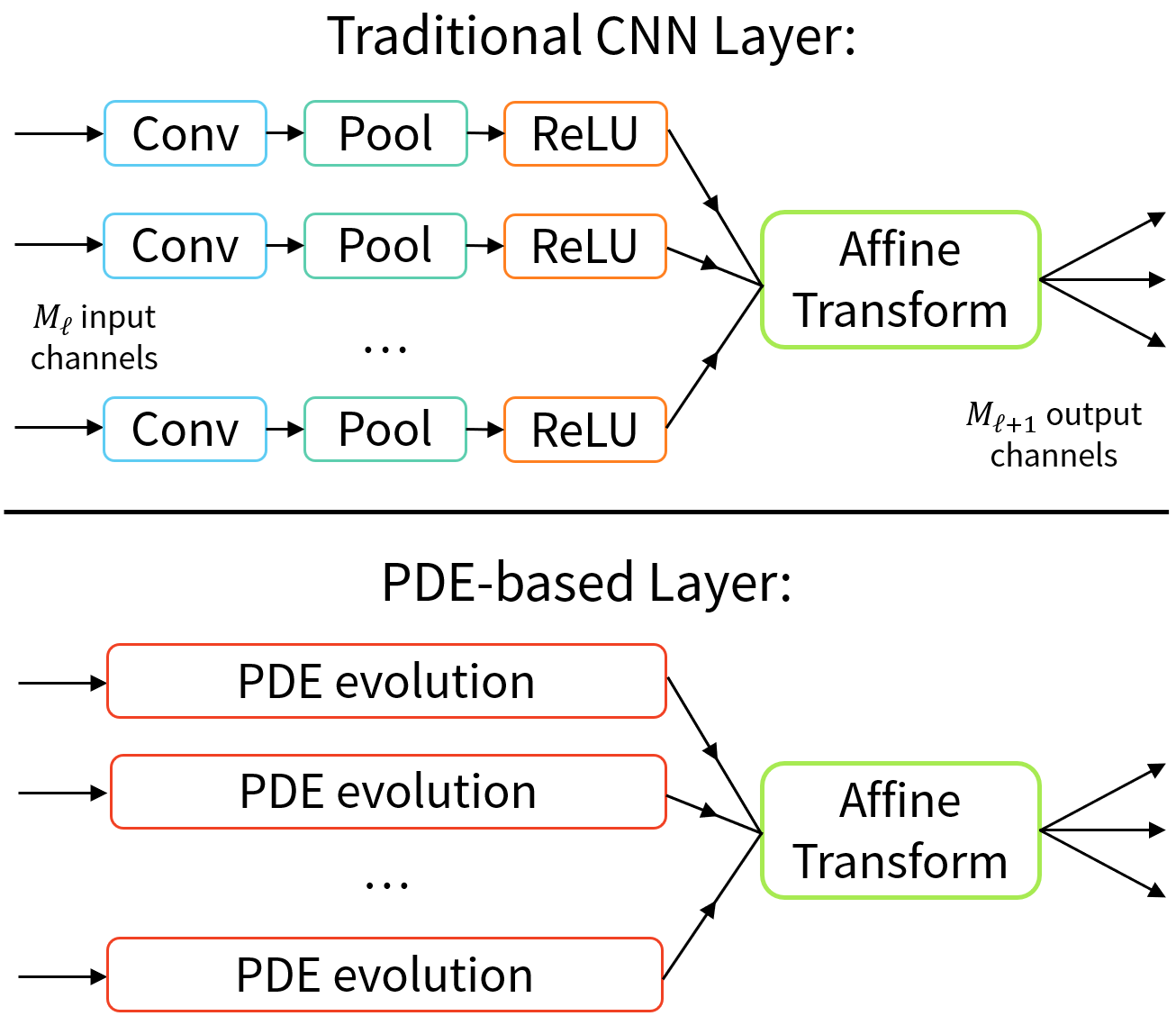}
    \caption{In a PDE-based CNN we replace the traditional convolution, pooling and ReLU operations by a PDE solver. The inputs of a given layer serve as initial conditions for a set of evolution PDEs, the outputs consist of affine combinations of the solutions of those PDEs at a fixed point in time. The parameters of the PDE become the trainable weights (alongside the affine parameters) over which we optimize.}
    \label{fig:traditional}
\end{figure}

\begin{figure}[ht]
    \centering
    \includegraphics[width=0.7\linewidth]{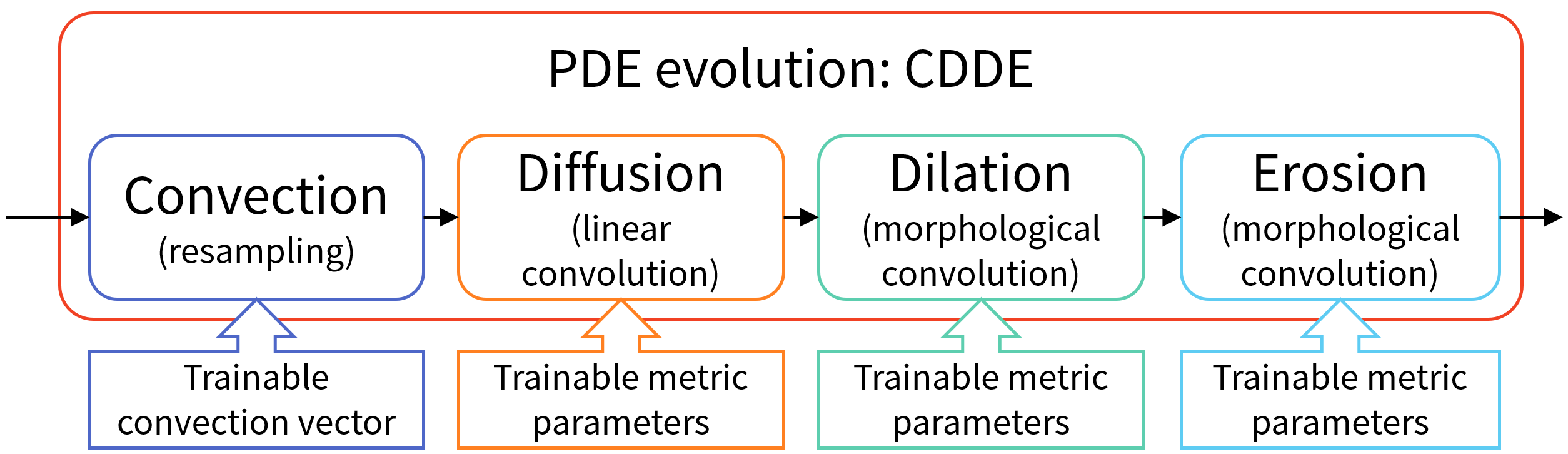}
    \caption{Our Hamilton-Jacobi type PDE of choice contains a convection, diffusion, dilation and erosion term (CDDE for short). Through operator splitting we solve for these terms separately by using resampling (for convection), linear convolution (for diffusion) and morphological convolution (for dilation and erosion). }
    \label{fig:cdde-layer}
\end{figure}

This paper is a substantially extended journal version of \cite{duits2021equivariant} presented at the SSVM 2021 conference.

\subsection{Structure of the Article}

The structure of the article is as follows. We first place our work in its mathematical and deep learning context in Section~\ref{sec:context}. Then we introduce the needed theoretical preliminaries from Lie group theory in Section \ref{sec:prelims} where we also define the space of positions and orientations $\mathbb{M}_d$ that will allow us to construct roto-translation equivariant networks. 

In Section \ref{sec:architecture} we give the overall architecture of a PDE-G-CNN and the ancillary tools that are needed to support a PDE-G-CNN. We propose an equivariant PDE that models commonly used operations in CNNs.

In Section \ref{sec:pde_solver} we detail how our PDE of interest can be solved using a process called operator splitting.
Additionally, we give tangible approximations to the fundamental solutions (kernels) of the PDEs that are both easy to compute and sufficiently accurate for practical applications. We use them extensively in the PDE-G-CNNs GPU-implementations in PyTorch that can be downloaded from the GIT-repository: \\   {\color{blue}\url{https://gitlab.com/bsmetsjr/lietorch}}.

Section \ref{sec:G-CNN-generalization} is dedicated to showing how common CNN operations such as convolution, max-pooling, ReLUs and skip connections can be interpreted in terms of PDEs.

We end our paper with some experiments showing the strength of PDE-G-CNNs in Section \ref{sec:experiments}, and concluding remarks in Section \ref{sec:concluding}.

The framework we propose covers transformations and CNNs on homogeneous spaces in general and as such we develop the theory in an abstract fashion. To maintain a bridge with practical applications we give details throughout the article on what form the abstractions take explicitly in the case of roto-translation equivariant networks acting on $\mathbb{M}_d$, specifically in 2D (i.e. $d=2$).


\section{Context}
\label{sec:context}

As this article touches on disparate fields of study we use this section to discuss context and highlight some closely related work.

\subsection{Drawing Inspiration from PDE-based Image Analysis}

Since the Partial Differential Equations that we use are well-known in the context of geometric image analysis~\cite{welk2019pde,Fadili2015,Peyre2010,Kimmel2015,Burger2016,Cremers2016,Weickert2016,SapiroBook,Sethian,WeickertEE,morel1995variational}, the layers also get an interpretation in terms of classical image-processing operators. This allows intuition and techniques from geometric PDE-based image analysis to be carried over to neural networks. 

In geometric PDE-based image processing it can be beneficial to include mean curvature or other geometric flows \cite{citti2016sub,Chambolle2011,chambolle2019total,smets2021total} as regularization and our framework provides a natural way for such flows to be included into neural networks. In the PDE-layer from Fig.~\ref{fig:cdde-layer} we only mention diffusion as a means of regularization, but mean curvature flow could easily be integrated by replacing the diffusion sub-layer with a mean curvature flow sub-layer. This would require replacing the linear convolution for diffusion by a median filtering approximation of mean curvature flow \cite{welk2019pde}.

\subsection{The Need for Lifting Images}

In geometric image analysis it is often useful to \emph{lift} images from a 2D picture to a 3D orientation score as in Fig.~\ref{fig:lift_project} and do further processing on the orientation scores \cite{duits2007invertible}. A typical image processing task in which such a lift is beneficial is that of the segmentation of blood vessels in a medical image. Algorithms based on processing the 2D picture directly, usually fail around points where two blood vessels cross, but algorithms that lift the image to an orientation score manage to decouple the blood vessels with different orientations as is illustrated in the bottom row of Fig. \ref{fig:lift_project}.

To be able to endow image-processing neural networks with the added capabilities (such as decoupling orientations and guaranteeing equivariance) that result from lifting data to an extended domain, we develop our theory for the more general CNNs defined on \emph{homogeneous spaces}, rather than just the prevalent CNNs defined on Euclidean space. One can then choose which homogeneous space to use based on the needs of one's application (such as needing to decouple orientations). A homogeneous space is, given subgroup $H$ of a group $G$, the manifold of left cosets, denoted by $G/H$. In the above image-analysis example, the group $G$ would be the special Euclidean group $G = SE(d)$, the subgroup $H$ would be the stabilizer subgroup of a fixed reference axis, and the corresponding homogeneous space $G/H$ would be the space of positions and orientations $\mathbb{M}_d \equiv \mathbb{R}^d \times S^{d-1}$, which is the lowest dimensional homogeneous space able to decouple orientations. By considering convolutional neural networks on homogeneous spaces such as $\mathbb{M}_d$ these networks have access to the same benefits of decoupling structures with different orientations as was highly beneficial for geometric image processing~\cite{InvertibleOrientationScores3D,MashtakovRecent,Citti,DuitsAMS1,DuitsAMS2,ZhangDuits,Gauthier,Prandi,DuitsACHA,Barbieri,pinwheel,Felsberg2,SavadjievPNAS2012,duits2019fourier,SiddiqiX}.

\begin{remark}[Generality of the architecture]
    Although not considered here, for other Lie groups applications (e.g. frequency scores \cite{duits2013evolution,boscain2021bio}, velocity scores, scale-orientation scores \cite{baspinar2018geometric}) the same structure applies, therefore we keep our theory in the general setting of homogeneous spaces $G/H$.
    This generality was also important in non-PDE based learning \cite{cohen2019general}, but also for PDE-based learning it is again beneficial.
\end{remark}

\begin{figure}[ht]
    \centering
    \includegraphics[width=0.7\linewidth]{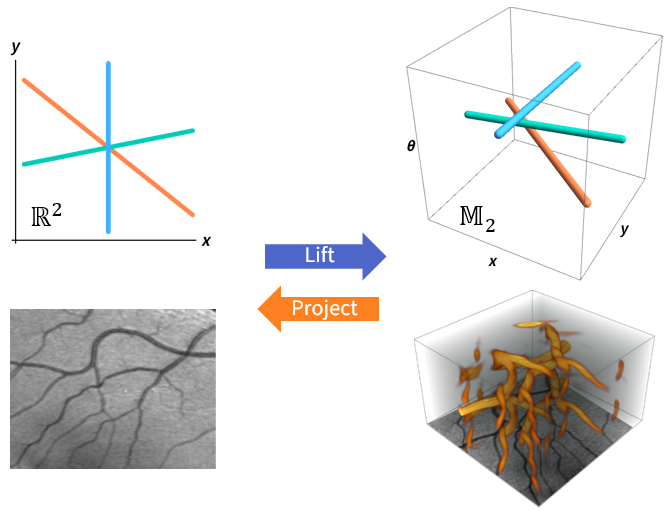}
    \caption[Illustrating the process of lifting and projecting]{Illustrating the process of lifting and projecting, in this case the advantage of lifting an image from $\mathbb{R}^2$ to the 2D space of positions and orientations $\mathbb{M}_2$ derives from the disentanglement of the lines at the crossings.}
    \label{fig:lift_project}
\end{figure}

\subsection{The Need for Equivariance}

We require the layers of our network to be \emph{equivariant}: a transformation of the input should lead to a corresponding transformation of the output, in other words: first transforming the input and then applying the network or first applying the network and then transforming the output should yield the same result.
A particular example, in which the output transformation is trivial (i.e. the identity transformation), is that of \emph{invariance}: in many classification tasks, such as the recognition of objects in pictures, an apple should still be recognized as an apple even if it is shifted or otherwise transformed in the picture as illustrated in Fig. \ref{fig:apples}. By guaranteeing equivariance of the network, the amount of data necessary or the need for data augmentation are reduced as the required symmetries are intrinsic to the network and need not be trained.

\begin{figure}[ht]
    \centering
    \includegraphics[width=0.7\linewidth]{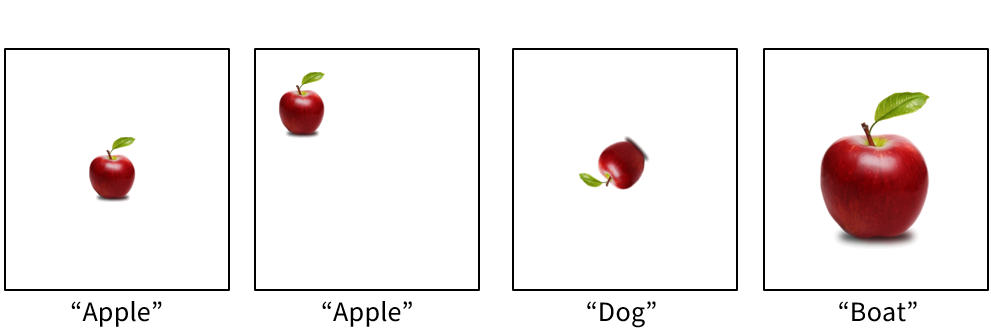}
    \caption[Equivariance]{Spatial CNNs, as used for image classification for example, are translation equivariant but not necessarily equivariant with respect to rotation, scaling and other transformations as the illustrative tags of the differently transformed apples images suggest. Building a G-CNN with the appropriately chosen group confers the network with all the equivariances appropriate for the chosen application. Our PDE-based approach is compatible with the group CNN approach \cite{cohen2019general} and so can confer the same symmetries.}
    \label{fig:apples}
\end{figure}

\begin{figure*}[ht]
    \centering
    \includegraphics[width=1.0\linewidth]{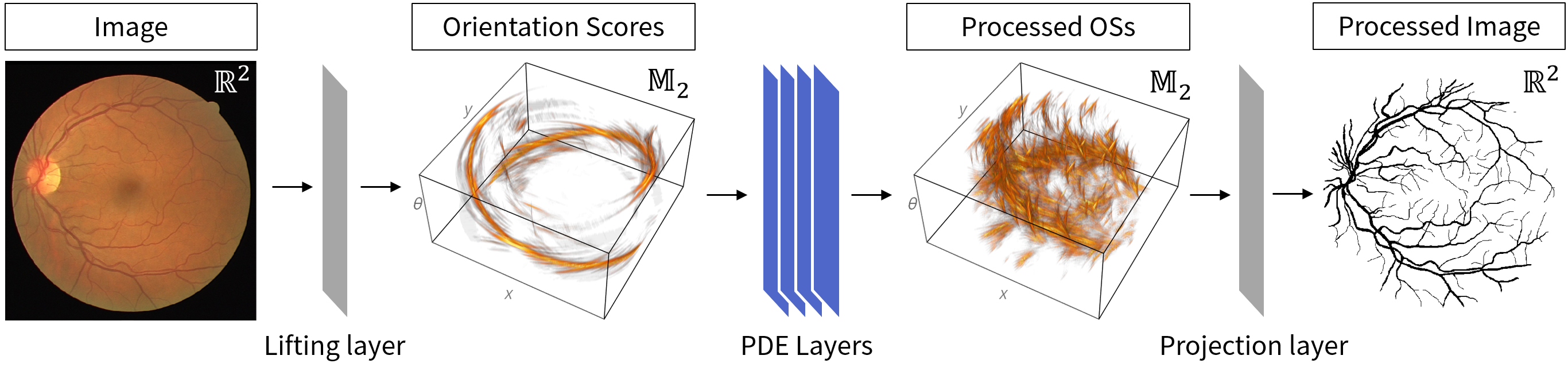}
    \caption[Illustrating the overall architecture of a PDE-G-CNN.]{Illustrating the overall architecture of a PDE-G-CNN (example: retinal vessel segmentation). An input image is lifted to a homogeneous space from which point on it can be fed through subsequent PDE layers (each PDE layer follow the structure of Fig.\ref{fig:cdde-layer}) that replace the convolution layers in conventional CNNs. Finally the result is projected back to the desired output space.}
    \label{fig:overall_architecture}
\end{figure*}

\subsection{Related Work}

\paragraph{G-CNNs}
After the introduction of G-CNNs by Cohen \& Welling \cite{cohen2016group} in the field of machine and deep learning, G-CNNs became popular. This resulted in many articles on showing the benefits of G-CNNs over classical spatial CNNs.
These works can be roughly categorised as
\begin{samepage}
\begin{itemize}
    \item discrete G-CNNs \cite{cohen2016group,dieleman2016exploiting,dieleman2015rotation,winkels20183d,worrall2018cubenet},
    \item regular continuous G-CNNs \cite{bekkers2015pde,oyallon2015deep,bekkers2018roto,weiler2018learning,bekkers2019bspline,finzi2020generalizing},
    \item and steerable continuous G-CNNs \cite{cohen2019general,worrall2017harmonic,kondor2018generalization,esteves2018learning,weiler2019general} that rely on Fourier transforms on homogeneous spaces \cite{chirikjian2000operational,duits2019fourier}.
\end{itemize}
\end{samepage}
Both regular and steerable G-CNNs naturally arise from linear mappings between functions on homogeneous spaces that are placed under equivariance constraints \cite{cohen2019general,kondor2018generalization,weiler2019general,bekkers2019bspline}. Regular G-CNNs explicitly extend the domain and lift feature maps to a larger homogeneous space of a group, whereas steerable CNNs extend the co-domain by generating fiber bundles in which a steerable feature vector is assigned to each position in the base domain. 
Although steerable operators have clear benefits in terms of computational efficiency and accuracy \cite{franken2006efficient,reisert2008group}, working with steerable representations puts constraints on non-linear activations within the networks which limits the representation power of G-CNNs \cite{weiler2019general}. Like regular G-CNNs, the proposed PDE-G-CNNs do not suffer from this. 
In our proposed PDE-G-CNN framework it is essential that we adopt the domain-extension viewpoint, 
as this allows to naturally and transparently construct scale space PDEs via left-invariant vector fields \cite{duits2007scale}.
In general this viewpoint entails that the domain of images is extended from the space of positions only, to a higher dimensional homogeneous space, and originates from coherent state theory \cite{ali2000coherent}, orientation score theory \cite{duits2007invertible}, cortical perception models \cite{Citti}, G-CNNs 
\cite{cohen2016group,bekkers2018roto}, and rigid-body motion scattering \cite{sifre2014rigid}.

The proposed PDE-G-CNNs form a new, unique class of equivariant neural networks, and we show in Section \ref{sec:G-CNN-generalization} how regular continuous G-CNNs arise as a special case of our PDE-G-CNNs. 

\paragraph{Probabilistic-CNNs}

Our geometric PDEs relate to $\alpha$-stable L\'evy processes \cite{duits2019fourier} and cost-processes akin to \cite{akian1994bellman}, but then on $\mathbb{M}_{d}$ rather than $\R^d$. This relates to probabilistic equivariant numerical neural networks 
\cite{finzi2020probabilistic} that use 
anisotropic convection-diffusions on $\R^d$. 

In contrast to these networks, the PDE-G-CNNs that we propose allow for \emph{simultaneous} spatial and angular diffusion on $\mathbb{M}_d$.
Furthermore we include nonlinear 
Bellman processes \cite{akian1994bellman} for max pooling over Riemannian balls.


\paragraph{KerCNNs}

An approach to introducing horizontal connectivity in CNNs that does not require a Lie group structure was proposed by Montobbio et al.\cite{montobbio2019metric,montobbio2019kercnns} in the form of \mbox{KerCNNs}. In this biologically inspired metric model a diffusion process is used to achieve intra-layer connectivity. 

While our approach does require a Lie group structure it is not restricted to diffusion and also includes dilation/erosion.

\paragraph{Neural Networks and Differential Equations}

The connection between neural networks and differential equations became widely known in 2017, when Weinan E \cite{weinan2017proposal} explicitly explained the connection between neural networks and dynamical systems especially in the context of the ultradeep ResNet \cite{he2016deep}. This point of view was further expanded by Lu et al. \cite{lu2017beyond}, showing how many ultradeep neural networks can be viewed as discretizations of ordinary differential equations. The somewhat opposite point of view was taken by Chen et al. \cite{chen2018neural}, who introduced a new type of neural network which no longer has discrete layers, them being replaced by a field parameterized by a continuous time variable. 
Weinan~E also indicated a relationship between CNNs and PDEs, or rather with evolution equations involving a nonlocal operator. 
Implicitly, the connection between neural networks and differential equations was also explored by the early works of Chen et al. \cite{chen2015learning} who learn parameters in a reaction-diffusion equation. This connection between neural networks and PDEs was then made explicit and more extensive by Long et al. who made it possible to learn a much wider class of PDEs \cite{long2017pde} with their PDE-Net.
More recent work in PDE inspired neural networks includes \cite{ruthotto2020deep, shen2020pdo}.

Basing neural network computations on PDEs formulated on manifolds also makes the processing independent with respect to the choice of coordinates on the manifold in the fashion of Weiler et al. \cite{weiler2021coordinate}.

More recent work in this direction includes integrating equivariant partial differential operators in steerable CNNs \cite{jenner2021steerable}, drawing a strong analogy between deep learning and physics.

A useful aspect of the connection between neural networks and differential equations is the observation that the stability of the differential equation can give into the stability and generalization ability of the neural network \cite{haber2017stable}.
Moreover, there are intriguing analogies with numerical PDE-approximations and specific network architectures (e.g.~ResNets), as can be seen in the comprehensive overview article by Alt et al.\cite{alt2021translating}.

The main contribution of our work in the field of PDE-related neural networks, is that we implement and analyze geometric PDEs on homogeneous spaces, to obtain general \emph{ group equivariant PDE-based CNNs} whose implementations just require linear and morphological convolutions with new analytic approximations of scale space kernels. 

\section{Equivariance: Groups \& Homogeneous Spaces}
\label{sec:prelims}

We want to design the PDE-G-CNN, and its layers, in such a way that they are \emph{equivariant}. Equivariance essentially means that one can either transform the input and then feed it through the network, or first feed it through the network and then transform the output, and both give the same result. We will give a precise definition after introducing general notation.

\subsection{The General Case}

A layer in a neural network (or indeed the whole network) can be viewed as an operator from a space of real-valued functions defined on a space $X$ to a space of real-valued functions defined on a space $Y$. It may be helpful to think of these function spaces as spaces of images.

We assume that the possible transformations form a \emph{connected Lie group} $G$. Think for instance of a group of translations which shift the domain into different directions.
The Lie group being connected excludes transformations such as reflections, which we want to avoid for the sake of simplicity.
We further assume that the Lie group $G$ acts smoothly on both spaces $X$ and $Y$, which means that there are smooth maps 
 $\rho_X : G \times X \to X$ and $\rho_Y: G \times Y \to Y$ such that for all $g, h \in G$,
 \begin{align*}
 \rho_X(g h, x) = \rho_X(g ,\rho_X(h, x))
 \intertext{and}
 \rho_Y(g h, x) = \rho_Y(g ,\rho_Y(h, x)),
 \end{align*}
making $\rho_X$ and $\rho_Y$ group actions on their respective spaces.

Additionally we will assume that the group $G$ acts \emph{transitively} on the spaces, meaning that for any two elements of these spaces there exists a transformation in $G$ that maps one to the other. This has as the consequence that $X$ and $Y$ can be seen as \emph{homogeneous spaces} \cite{koda2009introduction}. In particular, this means that after selecting a reference element $x_0 \in X$ we can make the following isomorphism:
 \begin{equation}
     X \ \equiv\  G/\textrm{Stab}_G ( x_0 )
 \end{equation}
 using the mapping
 \begin{equation}
    \label{eq:isomap}
     x \mapsto \left\{ g \in G \,\vert\, \rho_X(g,x_0) = x \right\}
     ,
 \end{equation}
which is a bijection due to transitivity and the fact that
\begin{equation*}
\textrm{Stab}_G(x_0):=\left\{ g \in G \,\middle\vert\, \rho_X(g,x_0)=x_0 \right\}
\end{equation*}
is a closed subgroup of $G$.
Because of this we will represent a homogeneous space as the quotient $G/H$ for some choice of closed subgroup $H = \textrm{Stab}_G(x_0)$ since all homogeneous spaces are isomorphic to such a quotient by the above construction.

In this article we will restrict ourselves to those homogeneous spaces that correspond to those quotients $G/H$ where the subgroup $H$ is compact and connected. Restricting ourselves to compact and connected subgroups simplifies many constructions and still covers several interesting cases such as the rigid body motion groups $SE(d)$.

The elements of the quotient $G/H$ consist of subsets of $G$ which we will denote by the letter $p$, these subsets are know as left cosets of $H$ since every one of them consists of the set $p=gH$ for some $g \in G$, the left cosets are a partition of $G$ under the equivalence relation
\begin{equation*}
    g_1 \sim g_2 
    \iff g_1^{-1} g_2 \in H.
    \iff g_1 H = g_2 H.
\end{equation*}

Under this notation the group $G$ consists of the disjoint union
\begin{equation}
    G = \coprod_{p\in G/H} p.
\end{equation}
The left coset that is associated with the reference element $x_0 \in X$ is $H$ and for that reason we also alias it by $p_0 := H$ when we want to think of it as an atomic entity rather than a set in its own right.

We will denote quotient map from $G$ to $G/H$ with $\pi$:
\begin{equation}
    \label{eq:quotientmap}
    \pi(g) := g p_0 := gH.
\end{equation}

\begin{remark}[Principal homogeneous space]
    Observe that by choosing $H=\left\{ e \right\}$ we get $G/H \equiv G$, i.e. the Lie group is a homogeneous space of itself. This is called the principal homogeneous space. In that case the group action is equivalent to the group composition.
    The numerical experiments we perform in this paper are on the principal homogeneous space $\mathbb{R}^2 \times S^1$ of $SE(2)$.
\end{remark}

We will denote the group action/left-multiplication by an element $g \in G$ by the operator $L_g: G/H \to G/H$ given by
\begin{equation}
    L_g p := g p \quad \text{for all}\quad p \in G / H.
\end{equation}
In addition, we denote the left-regular representation of $G$ on functions $f$ defined on $G/H$ by $\mathcal{L}_g$ defined by
\begin{equation}
    \left(\mathcal{L}_g f\right)(p)
    :=
    f \left( g^{-1} p \right)
    .
\end{equation}

A neural network layer is itself an operator (from functions on $G/H_X$ to functions on $G/H_Y$), and we require the operator to be equivariant with respect to the actions on these function spaces.

\begin{definition}[Equivariance]
    \label{def:equivariance}
    Let $G$ be a Lie group with homogeneous spaces $G/H_X$ and $G/H_Y$. Let $\Phi$ be an operator from functions (of some function class) on $G/H_X$ to functions on $G/H_Y$, then we say that $\Phi$ is equivariant with respect to $G$ if for all functions $f$ (of that class) we have that:
    \begin{equation}
        \forall g \in G
        : 
        \boxed{
        \left(\Phi \ \circ\ \mathcal{L}_g\right) f = \left(\mathcal{L}_g \ \circ\ \Phi\right) f,
        }
        \label{eq:equivariance}
    \end{equation}
\end{definition}
or in words: the neural network commutes with transformations.

Most of the time we will have $H_X = H_Y$ in our proposed neural networks, only the initial lifting layer and the final projection layer will be between different homogeneous spaces, as we will see later on.

\subsection{Vector and Metric Tensor Fields}

The particular operators that we will base our framework on are vector and tensor fields, if these basic building blocks are equivariant then our processing will be equivariant. We explain what left invariance means for these objects next.

For $g \in G$ and $p \in G/ H$, let $T_p(G/H)$ be the tangent space at point $p$ then the pushforward 
\begin{equation*}
    \left( L_{g} \right)_* : T_p \left( G/H \right) \to T_{g p} \left( G/H \right)
\end{equation*}
of the group action $L_g$ is defined by the condition that for all smooth functions $f$ on $G/H$ and all $\bm{v} \in T_p(G/H)$ we have that
\begin{equation}
    \label{eq:pushforward}
    \left( \left(L_g\right)_* \bm{v} \right) f := \bm{v} \left( f \,\circ\, L_g  \right) 
    .
\end{equation}

\begin{remark}[Tangent vectors as differential operators]
    \label{remark:tangent_vectors}
    Other than the usual geometric interpretation of tangent vectors as being the velocity vectors $\dot{\gamma}(t)$ tangent to some differentiable curve $\gamma:\mathbb{R} \to G/H$ we will simultaneously use them as differential operators acting on functions as we did in \eqref{eq:pushforward}. This algebraic viewpoint defines the action of the tangent vector $\dot{\gamma}(t)$ on a differentiable function $f$ as
    \begin{equation*}
        \dot{\gamma}(t) f := \frac{\partial}{\partial s} f \left( \gamma(s) \right)
        \big\vert_{s=t}
        .
    \end{equation*}
    In the flat setting of $G=\left(\mathbb{R}^d,+\right)$, where the tangent spaces are isomorphic to the base manifold $\mathbb{R}^d$, when we have a tangent vector $\bm{c} \in \mathbb{R}^d$ its application to a function is the familiar directional derivative:
    \begin{equation*}
        \bm{c}f = \bm{c} \cdot \nabla f = \d f (\bm{c}).
    \end{equation*}
    See \cite[\S 2.1.1]{lee2009manifolds} for details on this double interpretation.
\end{remark}

Vector fields that have the special property that the push forward $(L_g)_*$ maps them to themselves in the sense that
\begin{equation}
    \label{eq:left_invariant_vector_field_2}
    \forall g \in G, \forall p \in G/H
    :
    \bm{v} \left( p \right) f
    =
    \bm{v} \left( g p \right) \left[ \mathcal{L}_g f \right]
    ,
\end{equation}
for all differentiable functions $f$ and where $\bm{v}:p \mapsto T_p \left( G/H \right)$ is a vector field, are referred to as $G$-invariant.

\begin{definition}[$G$-invariant vector field on a homogeneous space]
    A vector field $\bm{v}$ on $G/H$ is invariant with respect to $G$ if it satisfies
      \begin{equation}
        \label{eq:left_invariant_vector_field}
        \forall g \in G,\, \forall p \in G/H
        :
        \bm{v} \left( g p \right) = \left( L_g \right)_* \bm{v} \left( p \right)
        .
    \end{equation}
\end{definition}
It is straightforward to check that \eqref{eq:left_invariant_vector_field_2} and \eqref{eq:left_invariant_vector_field} are equivalent and that these imply the following.
\begin{corollary}[Properties of $G$-invariant vector fields]
    \label{cor:livf_properties}
    On a homogeneous space $G/H$ a $G$-invariant vector field $\bm{v}$ has the following properties:
    \begin{enumerate}
            \setlength{\itemsep}{0.5em}
        \item it is fully determined by its value $\bm{v}\vert_H \in T_H(G/H)$ in $H$,
        \item $\forall h \in H: \left( L_h \right)_* \bm{v}\vert_H = \bm{v}\vert_H$.
    \end{enumerate}
\end{corollary}

We also introduce $G$-invariant metric tensor fields.

\begin{definition}[$G$-invariant metric tensor field on $G/H$]
    \label{def:g-invariant-metric-tensor-field}
    A $(0,2)$-tensor field $\mathcal{G}$ on $G/H$ is $G$-invariant if and only if
    \begin{align}
        \nonumber
        \forall g \in G,\, \forall p \in G/H,\, \forall \bm{v},\bm{w} \in T_p \left( G/H \right): \\[0.5em]
        \mathcal{G}\big\vert_p \left( \bm{v},\bm{w} \right)
        =
        \mathcal{G}\big\vert_{g p}
        \left( \vphantom{\big|} \left( L_{g} \right)_* \bm{v} ,\, \left( L_{g} \right)_* \bm{w} \right).
        \label{eq:left_invariant_metric_tensor}
    \end{align}
\end{definition}
Recall that $L_g p := g p$ and so the push-forward $\left( L_{g} \right)_*$ maps tangent vector from $T_p(G/H)$ to $T_{g p}(G/H)$. Again it follows immediately from this definition that a $G$-invariant metric tensor field has similar properties as a $G$-invariant vector field.
\begin{corollary}[Properties of $G$-invariant metric tensor fields]
    \label{cor:limtf_properties}
    On a homogeneous space $G/H$, a $G$-invariant metric tensor field $\mathcal{G}$ has the following properties:
    \begin{enumerate}
        \item it is fully determined by its metric tensor $\mathcal{G}\vert_{p_0}$ at $p_0=H$,
        \item $\forall h \in H,\, \forall \bm{v}, \bm{w} \in T_{p_0} \left( G/H \right) :
        \\[0.5em]
        \mathcal{G}\big\vert_{p_0} \left( \bm{v}, \bm{w} \right) = \mathcal{G}\big\vert_{p_0} \left( \left(L_h\right)_* \bm{v}, \left( L_h  \right)_* \bm{w} \right)$.
    \end{enumerate}
\end{corollary}
Or in words, the metric (tensor) has to be symmetric with respect to the subgroup $H$.

A (positive definite) metric tensor field yields a Riemannian metric in the usual manner, as we recall next.
\begin{definition}[Metric on $G/H$]
    \label{def:metric_G/H}
    Let $p_1,p_2 \in G/H$ then:
    \begin{equation}
        \begin{split}
        &d_{\mathcal{G}}(p_1,p_2) := d_{G/H,\mathcal{G}}(p_1,p_2)
        :=
        \\
        &\inf_{\substack{\beta \in \Lip([0,1],\ G/H) \\
        \nonumber
        \beta(0)=p_1 ,\ \beta(1)=p_2 }}
        \int_{0}^1 \sqrt{\mathcal{G} \vert_{\beta(t)} \left( \dot{\beta}(t),  \dot{\beta}(t) \right)} \ dt.
        \end{split}
    \end{equation}
\end{definition}

As metrics and their smoothness play a role in our construction we need to take into account where that smoothness fails.
\begin{definition}
\label{def:cutlocus}
The cut locus $\cut(p) \subset G/H$ or $\cut(g) \subset G$ is the set of points respectively group elements where the distance map from $p$ resp. $g$ is not smooth (excluding the point $p$ and group element $g$ themselves). 
\end{definition}
As long as we stay away from the cut locus the infimum from Def.~ \ref{def:metric_G/H} gives a unique geodesic.

Being derived from a $G$-invariant tensor field gives the metric $d_{\mathcal{G}}$ the same symmetries.
\begin{proposition}[$G$-invariance of the metric on $G/H$]
    Let $p_1,p_2 \in G/H$ away from each other's cut locus, then we have:
    \begin{equation*}
        \forall g \in G: d_{\mathcal{G}}(p_1,p_2)=d_{\mathcal{G}}(g p_1,g p_2).
    \end{equation*}
\end{proposition}

\begin{proof}
    We observe that we can make a bijection from the set of Lipschitz curves between $p_1$ and $p_2$ and between $g p_1$ and $g p_2$ simply by left multiplication by $g$ one way and $g^{-1}$ the other way. 
    Due to \eqref{eq:left_invariant_metric_tensor} multiplying a curve with a group element preserves its length, hence if $\gamma:[0,1] \to G/H$ is the geodesic from $p_1$ to $p_2$ then $g \gamma$ is the geodesic from $g p_1$ to $g p_2$, both having the same length.
    \qed
\end{proof}

A metric tensor field on the homogeneous space has a natural counterpart on the group.
\begin{definition}[Pseudometric tensor field on $G$]
    \label{def:pseudometric_tensor_field}
    A $G$-invariant metric tensor field $\mathcal{G}$ on $G/H$ induces a (pullback) pseudometric tensor field $\tilde{\mathcal{G}}$ on $G$ that is left-invariant:
    \begin{equation}
        \label{eq:pseudometric_tensor_field}
        \tilde{\mathcal{G}} := \pi^* \mathcal{G},
    \end{equation}
    where $\pi^*$ is the pullback of the quotient map $\pi$ from \eqref{eq:quotientmap}. This is equivalent to saying that for all $\bm{v},\bm{w} \in T_g G$:
    \begin{equation*}
        \tilde{\mathcal{G}} \big\vert_g \left( \bm{v}, \bm{w} \right)
        :=
        \mathcal{G}\big\vert_{\pi(g)} \left( \pi_* \bm{v}, \pi_* \bm{w} \right)
        ,
    \end{equation*}
    where $\pi_*$ is the pushforward of $\pi$.
\end{definition}
This tensor field $\tilde{\mathcal{G}}$ is left-invariant by virtue of $\mathcal{G}$ being $G$-invariant. It is also degenerate in the direction of $H$ and so yields a seminorm on $TG$.
\begin{definition}[Seminorm on $TG$]
    \label{def:seminorm}
    Let $\bm{v} \in T_g G$. Then the metric tensor field $\mathcal{G}$ on $G/H$ induces the following seminorm:
    \begin{equation}
        \left\| \bm{v}  \right\|_{\tilde{\mathcal{G}}}
        :=
        \sqrt{\tilde{\mathcal{G}} \vert_{g} \left( \bm{v}, \bm{v} \right)}
        :=
        \sqrt{\mathcal{G} \vert_{g p_0}  \left( \pi_* \bm{v}, \pi_* \bm{v} \right)}
        .
    \end{equation}
\end{definition}

In the same fashion we have an induced pseudometric on $G$ from the pseudometric tensor field on $G$.
\begin{definition}[Pseudometric on $G$]
Let $g_1,g_2 \in G$. Then we define:
    \begin{equation}
    \begin{split}
        &d_{\tilde{\mathcal{G}}}(g_1,g_2) := d_{G,\tilde{\mathcal{G}}}(g_1,g_2)
        :=
        \\
        &\inf_{\substack{\gamma \in \Lip([0,1],\ G) \\
        \gamma(0)=g_1,\ \gamma(1)=g_2}}
        \int_{0}^1 \sqrt{\tilde{\mathcal{G}} \vert_{\gamma(t)} \left( \dot{\gamma}(t), \dot{\gamma}(t) \right)} \ dt
        .
    \end{split}
    \end{equation}
\end{definition}
This pseudometric has the property that $d_{\tilde{\mathcal{G}}}(h_1,h_2)=0$ for all $h_1,h_1 \in H$, in fact for all $p \in G/H$ we have that $d_{\tilde{\mathcal{G}}}(g_1,g_2)=0$ for all $g_1,g_2 \in p$.

By requiring $G$ and $H$ to be connected we get the following strong correspondence between the metric structure on the homogeneous space and the pseudometric structure on the group.

\begin{lemma}
    \label{lem:metric_correspondence}
    Let $g_1,g_2 \in G$ so that $\pi(g_2)$ is away from the cut locus of $\pi(g_1)$, then:
    \begin{equation*}
        d_{\tilde{\mathcal{G}}}(g_1,g_2) = d_{\mathcal{G}}(\pi(g_1), \pi(g_2)).
    \end{equation*}
    Moreover if $\gamma$ is a minimizing geodesic in the group $G$ connecting $g_1$ with $g_2$ then $\pi \circ \gamma$ is the unique minimizing geodesic in the homogeneous space $G/H$ that connects $\pi(g_1)$ with $\pi(g_2)$. 
\end{lemma}

\begin{proof}
    Assuming it exists, let $\gamma \in \Lip([0,1],G)$ be a minimizing geodesic connecting $\gamma(0)=g_1$ with $\gamma(1)=g_2$ and let $\beta \in Lip([0,1],G/H)$ be the unique minimizing geodesic connecting $\beta(0)=\pi(g_1)$ with $\beta(1)=\pi(g_2)$. 
    Because of the pseudometric on $G$, minimizing geodesics are not unique, i.e. $\gamma$ is not unique. On $G/H$ we have a full metric and so staying away from the cut locus means $\beta$ is both unique and minimizing.
        
    Denote the length functionals with:
    \begin{gather*}
        \Len_G(\gamma) 
        := 
        \int_0^1 \sqrt{\tilde{\mathcal{G}}\vert_{\gamma(t)} \left( \dot{\gamma}(t), \dot{\gamma}(t) \right)} \d t
        ,
        \\
        \Len_{G/H}(\beta) 
        := 
        \int_0^1 \sqrt{\mathcal{G}\vert_{\beta(t)} \left( \dot{\beta}(t), \dot{\beta}(t) \right)} \d t
        .
    \end{gather*}
    Observe that by construction of the pseudometric tensor field $\tilde{\mathcal{G}}$ on $G$ we have: $\Len_G(\gamma)=\Len_{G/H}(\pi \circ \gamma)$.
    
    Now we assume $\pi \circ \gamma \neq \beta$.
    Then since $\beta$ is the unique geodesic we have
    \begin{equation*}
        \Len_{G/H}(\beta) < \Len_{G/H}(\pi \circ \gamma) = \Len_G(\gamma).
    \end{equation*}
    But then we can find some $\gamma_{\mathrm{lift}} \in \Lip([0,1],G)$ that is a preimage of $\beta$, i.e. $\pi \circ \gamma_{\mathrm{lift}} = \beta$. 
    The potential problem is that while $\gamma_{\mathrm{lift}}(0) \in \pi(g_1)$ and $\gamma_{\mathrm{lift}}(1) \in \pi(g_2)$, $\gamma_{\mathrm{lift}}$ does not necessarily connect $g_1$ to $g_2$.
    But since the coset $\pi(g_1)$ is connected we can find a curve wholly contained in it that connects $g_1$ with $\gamma_{\mathrm{lift}}(0)$, call this curve $\gamma_{\mathrm{head}} \in \Lip([0,1], \pi(g_1))$. 
    Similarly we can find a $\gamma_{\mathrm{tail}} \in \Lip([0,1], \pi(g_2))$ that connects $\gamma_{\mathrm{lift}}(1)$ to $g_2$.
    Both these curves have zero length since $\pi$ maps them to a single point on $G/H$, i.e. $\Len_G(\gamma_{\mathrm{head}})=\Len_G(\gamma_{\mathrm{tail}})=0$.
    
    Now we can compose these three curves:
    \begin{equation*}
        \gamma_{\mathrm{new}}(t) := 
        \begin{cases}
            \gamma_{\mathrm{head}}(3t) & \text{ if } t \in [0,\sfrac{1}{3}],
            \\
            \gamma_{\mathrm{lift}}(3t-1) & \text{ if } t \in [\sfrac{1}{3},\sfrac{2}{3}],
            \\
            \gamma_{\mathrm{tail}}(3t-2) & \text{ if } t \in [\sfrac{2}{3},1].
        \end{cases}
    \end{equation*}
    This new curve is again in $\Lip([0,1],G)$ and connects $g_1$ with $g_2$, but also:
    \begin{equation*}
        \Len_G(\gamma_{\mathrm{new}}) = \Len_G(\gamma_{\mathrm{lift}}) = \Len_{G/H}(\beta) < \Len_G(\gamma),
    \end{equation*}
    which is a contradiction since $\gamma$ is a minimizing geodesic between $g_1$ and $g_2$.
    We conclude $\pi \circ \gamma = \beta$ and thereby:
    \begin{equation*}
        d_{\tilde{\mathcal{G}}}(g_1,g_2) 
        = \Len_G(\gamma)
        = \Len_{G/H}(\beta)
        = d_{\mathcal{G}}(\pi(g_1), \pi(g_2)).
    \end{equation*}
    \qed
\end{proof}

This result allows us to more easily translate results from Lie groups to homogeneous spaces.

We end our theoretical preliminaries by introducing the space of positions and orientations $\mathbb{M}_d$.

\subsection{Example: The Group \texorpdfstring{$SE(d)$}{SE(d)} and the Homogeneous Space \texorpdfstring{$\mathbb{M}_d$}{Md}}

Our main example and Lie group of interest is the \emph{Special Euclidean} group $SE(d)$ of the rotations and translations of $\mathbb{R}^d$, in particular for $d \in \left\{ 2,3 \right\}$. When we take $H=\{ 0 \} \times SO(d-1)$ we obtain the space of positions and orientations 
\begin{equation} \label{eq:Mspecial}
\mathbb{M}_d=SE(d)/\left( \{ 0 \} \times SO(d-1) \right). 
\end{equation}
This homogeneous space and group will enable the construction of roto-translation equivariant networks. 
For the experiments in this paper we restrict ourselves to $d=2$ but we include the case $d=3$ in some of our theoretical preliminaries.

As a set we identify $\mathbb{M}_d$ with $\mathbb{R}^d \times S^{d-1}$ and choose some reference direction $\bm{a} \in S^{d-1} \subset \mathbb{R}^d$ as the forward direction so that we can set the reference point of the space as $p_0 = \left( \bm{0}, \bm{a} \right)$. We can then see that elements of $H$ are those rotations that map $\bm{a}$ to itself, i.e. rotations with the forward direction as their axis.

If we denote elements of $SE(d)$ as translation/rotation pairs $\left(\bm{y},R\right) \in \mathbb{R}^d \times SO(d)$ then group multiplication is given by
\begin{gather*}
    g_1 = \left( \bm{y}_1, R_1 \right),\,
    g_2 = \left( \bm{y}_2, R_2 \right) \in G
    :
    \\[0.4em]
    g_1 g_2 = \left( \bm{y}_1, R_1 \right) \left( \bm{y}_2, R_2 \right)
    =
    \left( \bm{y}_1 + R_1 \bm{y}_2, R_1 R_2 \right)
    ,
\end{gather*}
and the group action on elements $p=\left( \bm{x}, \bm{n} \right) \in  \mathbb{R}^d \times S^{d-1} \equiv \mathbb{M}_d$ is given as
\begin{equation}
    g p = \left( \bm{y}, R \right) \left( \bm{x}, \bm{n} \right)
    =
    \left( \bm{y} + R\bm{x}, R\bm{n} \right).
\end{equation}

What the G-invariant vector field and metric tensor fields look like on $\mathbb{M}_d$ is different for $d=2$ than for $d>2$. We first look at $d>2$. 

\begin{proposition} \label{prop:livmd}
    Let $d>2$ and let $\partial_{\bm{a}} \in T_{p_0}\left( \mathbb{M}_d \right)$ be the tangent vector in the reference point in the main direction $\bm{a} \in S^{d-1}$, specifically:
\begin{equation*}
    \partial_{\bm{a}} f
    :=
    \lim_{t \to 0}
    \frac{
    f \left( (t\bm{a},\bm{a}) \right)
    -
    f \left( (\bm{0},\bm{a}) \right)
    }{t},
\end{equation*}
where $f : \mathbb{M}_d \to \mathbb{R}$ is smooth in an open neighborhood of $p_0=(0,\bm{a})$,
then all $SE(d)$-invariant vector fields are spanned by the vector field:
\begin{equation}
    \label{eq:livf_md}
    p \mapsto
    \mathcal{A}_1 \big\vert_{p}
    :=
    \left( L_{g_p} \right)_* \partial_{\bm{a}},
\end{equation}
with $g_p \in p \in \mathbb{M}_d$. 
\end{proposition}

\begin{proof}
    For $d > 3$ we can see that \eqref{eq:livf_md} are the only left-invariant vector fields since for all $h \in H$ we have $\left(g_p h\right)p_0=p$ and so in order to be well-defined we must require $\left( L_h \right)_* \bm{v} = \bm{v}$ on $T_{p_0} \left( \mathbb{M}_d \right)$, and this is true for $\partial_{\bm{a}}$ (and its scalar multiples) but not true for any other tangent vectors at $T_{p_0} \left( \mathbb{M}_d \right)$.
\qed
\end{proof}

\begin{proposition} \label{prop:metricdlarger2}
    For $d>2$ the only Riemannian metric tensor fields on $\mathbb{M}_d$ that are $SE(d)$-invariant are of the form:
    \begin{multline}
        \label{eq:md_metric}
        \mathcal{G} \big\vert_{\left(\bm{x},\bm{n}\right)}
        \left(
            \vphantom{\big\vert}
            \left( \dot{\bm{x}}, \dot{\bm{n}} \right)
            ,
            \left( \dot{\bm{x}}, \dot{\bm{n}} \right)
        \right)
        =
        \\
        w_M
        \left| \dot{\bm{x}} \bigcdot \bm{n} \right|^2
        +
        w_L
        \left\Vert \dot{\bm{x}} \wedge \bm{n} \right\Vert^2
        +
        w_A
        \left\Vert \dot{\bm{n}} \right\Vert^2,
    \end{multline}
    with $w_M, w_L, w_A > 0$ weighing the main, lateral and angular motion respectively and where the inner product, outer product and norm are the standard Euclidean constructs.
\end{proposition}
\begin{proof}
    It follows that to satisfy the second condition of Cor.~\ref{cor:limtf_properties} at the tangent space $T_{(\bm{x},\bm{n})}$ of a particular $(\bm{x},\bm{n})$ the metric tensor needs to be symmetric with respect to rotations about $\bm{n}$ both spatially and angularly (i.e. we require isotropy in all angular and lateral directions) which leads to the three degrees of freedom contained in \eqref{eq:md_metric} irrespective of $d$.
    \qed
\end{proof}

For $d=2$ we represent the elements of $\mathbb{M}_2$ with $(x,y,\theta) \in \mathbb{R}^3$ where $x,y$ are the usual Cartesian coordinates and $\theta$ the angle with respect to the $x$-axis, so that $\bm{n}=(\cos\theta,\sin\theta)^T$. 
The reference element is then simply denoted by $(0,0,0)$. 

It may be counter-intuitive but decreasing the number of dimensions to $2$ gives more freedom to the $G$-invariant vector and metric tensor fields compared to $d>2$.
This is a consequence of the subgroup $H$ being trivial and so the symmetry conditions from Cor.~\ref{cor:livf_properties} and \ref{cor:limtf_properties} also become trivial.
The $SE(2)$-invariant vector fields are given as follows.

\begin{proposition}
\label{prop:livf_m2}
On $\mathbb{M}_2$ the $SE(2)$-invariant vector fields are spanned by the following basis:
\begin{equation}
    \label{eq:livf_m2}
    \begin{cases}
        \mathcal{A}_1 \big\vert_{(x,y,\theta)}
        &=
        \cos{\theta}\ \partial_x \big\vert_{(x,y,\theta)} + \sin{\theta}\ \partial_y \big\vert_{(x,y,\theta)}
        ,
        \\[1em]
        \mathcal{A}_2 \big\vert_{(x,y,\theta)}
        &=
        -\sin{\theta}\ \partial_x \big\vert_{(x,y,\theta)} + \cos{\theta}\ \partial_y \big\vert_{(x,y,\theta)}
        ,
        \\[1em]
        \mathcal{A}_3 \big\vert_{(x,y,\theta)}
        &=
        \partial_\theta \big\vert_{(x,y,\theta)}
        .
    \end{cases}
\end{equation}
\end{proposition}

\begin{proof}
For $d=2$ we have $\mathbb{M}_2 \equiv SE(d)$ and the group invariant vector fields on $\mathbb{M}_2$ are exactly the left-invariant vector fields on $SE(2)$ given by \eqref{eq:livf_m2}.
\qed
\end{proof}

In a similar manner $SE(2)$-invariant metric tensors are then given as follows.
\begin{proposition} \label{prop:metricm2}
On $\mathbb{M}_2$ the $SE(2)$-invariant metric tensor fields are given by:
\begin{equation*}
    \mathcal{G} \big\vert_{\left(x,y,\theta\right)} 
    \left( \bm{v}, \bm{w} \right)
    =
    \mathcal{G} \big\vert_{\left(0,0,0\right)}
    \left( \left(L^{-1}_{(x,y,\theta)}\right)_* \bm{v},  \left(L^{-1}_{(x,y,\theta)}\right)_* \bm{w}\right)
    ,
\end{equation*}
for any choice of inner product $\mathcal{G} \big\vert_{\left(0,0,0\right)}$ at $e$. 
\end{proposition}

\begin{proof}
    Since $SE(2) \equiv \mathbb{M}_2$ the $G$-invariant metric tensor fields are again exactly the left-invariant metric tensor fields.
    \qed
\end{proof}

This gives $SE(2)$-invariant metric tensor fields 6 degrees of freedom and hence 6 trainable parameters on $\mathbb{M}_2$.
Remarkably, the case $d=2$ allows for more degrees of freedom than the case $d=3$ where Proposition~\ref{prop:metricdlarger2} applies.
In our experiments so far we have restricted ourselves to those metric tensors that are diagonal with respect to the frame from Prop.~\ref{prop:livf_m2}. A diagonal metric tensor would have just 3 degrees of freedom and have the same general form as \eqref{eq:md_metric}, specifically:
\begin{align} 
    \begin{split}
    \mathcal{G} \big\vert_{\left(x,y,\theta\right)}
    &
        \left(
            \vphantom{\big\vert}
            \left( \dot{x}, \dot{y}, \dot{\theta} \right)
            ,
            \left( \dot{x}, \dot{y}, \dot{\theta} \right)
        \right)
        =
        \\
        &\qquad \ w_M
        \left| \dot{x} \cos{\theta} 
        +\dot{y} \sin{\theta} \right|^2
        \\
        &\qquad +
        w_L
        \left| -\dot{x} \sin{\theta} 
        +\dot{y} \cos{\theta} \right|^2
        \\
        &\qquad +
        w_A
         |\dot{\theta}|^2
        .
    \end{split}
    \label{eq:metric_m2}
\end{align}

We will expand into non-diagonal metric tensors in future work.

\section{Architecture}
\label{sec:architecture}

\subsection{Lifting \& Projecting}

The key ingredient of what we call a PDE-G-CNN is the PDE layer that we detail in the next section, however to make a complete network we need more. Specifically we need a layer that transforms the network's input into a format that is suitable for the PDE layers and a layer that takes the output of the PDE layers and transforms it to the desired output format.
We call this input and output transformation \emph{lifting} respectively \emph{projection}, this yields the overall architecture of a PDE-G-CNN as illustrated in Fig.~\ref{fig:overall_architecture}. 

As our theoretical preliminaries suggest we aim to do processing on homogeneous spaces but the input and output of the network do not necessarily live on that homogeneous space. Indeed in the case of images the data lives on $\mathbb{R}^2$ and not on $\mathbb{M}_2$ where we propose to do processing.

This necessitates the addition of \emph{lifting} and \emph{projection} layers to first transform the input to the desired homogeneous space and end with transforming it back to the required output space.
Of course for the entire network to be equivariant we require these transformation layers to be equivariant as well. In this paper we focus on the design of the PDE layers, details on appropriate equivariant lifting and projection layers in the case of $SE(2)$ can be found in \cite{bekkers2018roto,smets2019msc}.

\begin{remark}[General equivariant linear transformations between homogeneous spaces]
\label{remark:general_lift_project}
    A general way to lift and project from one homogeneous space to another in a trainable fashion is the following.
    Consider two homogeneous spaces $G/H_1$ and $G/H_2$ of a Lie group $G$, let $f: G/H_1 \to \mathbb{R}$ and $k : G/H_2 \to \mathbb{R}$ with the following property:
    \begin{equation*}
        \forall h \in H_1,\, q \in G/H_2: k \left( h q \right) = k(q)
        ,
    \end{equation*}
    where $H_1$ is compact.
    Then the operator $\mathcal{T}$ defined by
    \begin{equation}
        \label{eq:hs_to_hs_convolution}
        \forall q \in G/H_2:
        \left( \mathcal{T} f \right)(q) 
        :=
        \int_G k \left( g^{-1} q \right) \, f\left( g H_1 \right) \, \d\mu_G(g)
    \end{equation}
    transforms $f$ from a function on $G/H_1$ to a function on $G/H_2$ in an equivariant manner (assuming $f$ and $k$ are such that the integral exists). Here the kernel $k$ is the trainable part and $\mu_G$ is the left-invariant Haar measure on the group.
    
    Moreover it can be shown via the Dunford-Pettis\cite{arendt1994integral} theorem that (under mild restrictions) all linear transforms between homogeneous spaces are of this form.
\end{remark}

\begin{remark}[Lifting and projecting on $\mathbb{M}_2$]
    Lifting an image (function) on $\mathbb{R}^2$ to $\mathbb{M}_2$ can either be performed by a non-trainable \emph{Invertible Orientation Score Transform} \cite{duits2007invertible} or a trainable lift \cite{bekkers2018roto} in the style of Remark \ref{remark:general_lift_project}.

    Projecting from $\mathbb{M}_2$ back down to $\mathbb{R}^2$ can be performed by a simple maximum projection: let \mbox{$f:\mathbb{M}_2 \to \mathbb{R}$} then
    \begin{equation}
        \label{eq:max_projection}
        (x,y) \mapsto \max_{\theta \in \left[0,2\pi\right)} f(x,y,\theta)
    \end{equation}
    is a roto-translation equivariant projection as used in \cite{bekkers2018roto}. A variation on the above projection is detailed in \cite[Ch. 3.3.3]{smets2019msc}.
\end{remark}

\subsection{PDE Layer}
\label{sec:layer}

A PDE layer operates by taking its inputs as the initial conditions for a set of evolution equations, hence there will be a PDE associated with each input feature. The idea is that we let each of these evolution equations work on the inputs up to a fixed time $T>0$. Afterwards, we take these solutions at time $T$ and take affine combinations (really batch normalized linear combinations in practice) of them to produce the outputs of the layer and as such the initial conditions for the next set of PDEs. 

If we index network layers (i.e. the depth of the network) with $\ell$ and denote the width (i.e. the number of features or channels) at layer $\ell$ with $M_{\ell}$ then we have $M_{\ell}$ PDEs and take $M_{\ell+1}$ linear combinations of their solutions.
We divide a PDE layer into the PDE solvers that each apply the PDE evolution to their respective input channel and the affine combination unit. This design is illustrated in Fig. \ref{fig:traditional}, but let us formalize it. 

Let $\left(U_{\ell,c}\right)_{c=1}^{M_\ell}$ be the inputs of the $\ell$-th layer (i.e. some functions on $G/H$), let $a_{\ell i j}$ and $b_{\ell i} \in \mathbb{R}$ be the coefficients of the affine transforms for \mbox{$i=1 \ldots M_{\ell+1}$} and $j=1 \ldots M_{\ell}$. Let each PDE be parametrized by a set of parameters $\theta_{\ell j}$. Then the action of a PDE layer is described as:
\begin{equation}
    \label{eq:pde-layer}
    U_{\ell+1,i} = \sum_{j=1}^{M_\ell} a_{\ell i j} \Phi_{T,\theta_{\ell j}} \left( U_{\ell j} \right) + b_{\ell i}
    ,
\end{equation} where $\Phi_{T,\theta}$ is the evolution operator of the PDE at time $T\geq 0$ and parameter set $\theta$. We define the operator $\Phi_{t,\theta}$ so that $(p,t) \mapsto \left( \Phi_{t,\theta} U \right)(p)$ satisfies the Hamilton-Jacobi type PDE that we introduce in just a moment. In this layer formula the parameters $a_{\ell i j}$, $b_{\ell i}$ and $\theta_{\ell j}$ are the trainable weights, but the evolution time $T$ we keep fixed.

It is essential that we require the network layers, and thereby all the PDE units, to be \emph{equivariant}. This has consequences for the class of PDEs that is allowed.

The PDE solver that we will consider in this article, illustrated in Fig.~\ref{fig:cdde-layer}, computes the approximate solution to the PDE
\begin{flalign}
    \label{eq:cnn_pde}
    \begin{cases}
        \begin{aligned}[b]
        \frac{\partial W}{\partial t} (p,t)
        = \ 
        &\tikznode[inner sep=1pt, fill=blue!8]{N1}{\vphantom{\bigg|}\  - \bm{c} W(p,t) \ \ }
        &\tikznode[inner sep=1pt]{tag1}{\text{(convection)}}
        \\
        &\tikznode[inner sep=1pt, fill=orange!8]{N2}{\vphantom{\bigg|}- \left( - \Delta_{\mathcal{G}_1} \right)^{\alpha} W (p,t)\ }
        &\tikznode[inner sep=1pt]{tag2}{\text{(diffusion)}}
        \\
        &\tikznode[inner sep=1pt, fill=green!8]{N3}{\vphantom{\bigg|}+ \left\Vert \nabla_{\mathcal{G}_2^+} W (p,t) \right\Vert^{2 \alpha}_{\mathcal{G}_2^+}\ }
        &\tikznode[inner sep=1pt]{tag3}{\text{(dilation)}}
        \\
        &\tikznode[inner sep=1pt, fill=cyan!8]{N4}{\vphantom{\bigg|}- \left\Vert \nabla_{\mathcal{G}_2^-} W (p,t) \right\Vert^{2 \alpha}_{\mathcal{G}_2^-}\ }
        &\tikznode[inner sep=1pt]{tag4}{\text{(erosion)}}
        \end{aligned}
        \\
        \hfill\text{for } p \in G/H,\ t \geq 0,
        \\
        W(p,0) = U(p) \hfill\text{for } p \in G/H.
    \end{cases}
\end{flalign}
\begin{tikzpicture}[overlay,remember picture,very thick]
    \draw[blue!50] (tag1.south west) -- (tag1.south east);
    %
    \draw[orange!50] (tag2.south west) -- (tag2.south east);
    %
    \draw[green!50] (tag3.south west) -- (tag3.south east);
    %
    \draw[cyan!50] (tag4.south west) -- (tag4.south east);
\end{tikzpicture}
Here, $\bm{c}$ is a $G$-invariant vector field on $G/H$ (recall \eqref{eq:livf_md} and our use of tangent vectors as differential operators per Remark \ref{remark:tangent_vectors}), $\alpha \in \left[\sfrac{1}{2}, 1 \right]$, $\mathcal{G}_1$ and $\mathcal{G}_2^\pm$ are $G$-invariant metric tensor fields on $G/H$, $U$ is the initial condition and $\Delta_{\mathcal{G}}$ and $\Vert \cdot \Vert_{\mathcal{G}}$ denote the Laplace-Beltrami operator and norm induced by the metric tensor field $\mathcal{G}$.
As the labels indicate, the four terms have distinct effects:
\begin{itemize}
    \item convection: moving data around,
    \item (fractional) diffusion: regularizing data (which relates to subsampling by destroying data),
    \item dilation: pooling of data,
    \item erosion: sharpening of data.
\end{itemize}
This is also why we refer to a layer using this PDE as a CDDE layer. Summarized the parameters of this PDE are given by $\theta=\left( \bm{c},\ \mathcal{G}_{1},\  \mathcal{G}^+_{2},\  \mathcal{G}^-_{2} \right)$.
The geometric interpretation of each of the terms in \eqref{eq:cnn_pde} is illustrated in Fig. \ref{fig:interpretation_r2} for $G = \mathbb{R}^2$ and in Fig.~\ref{fig:interpretation_m2} for $G=\mathbb{M}_2$.

Since the convection vector field $\bm{c}$ and the metric tensor fields $\mathcal{G}_1$ and $\mathcal{G}_2^{\pm}$ are $G$-invariant, the PDE unit, and so the network layer, is automatically equivariant.

\begin{figure*}[ht]
    \centering
    \includegraphics[width=0.9\linewidth]{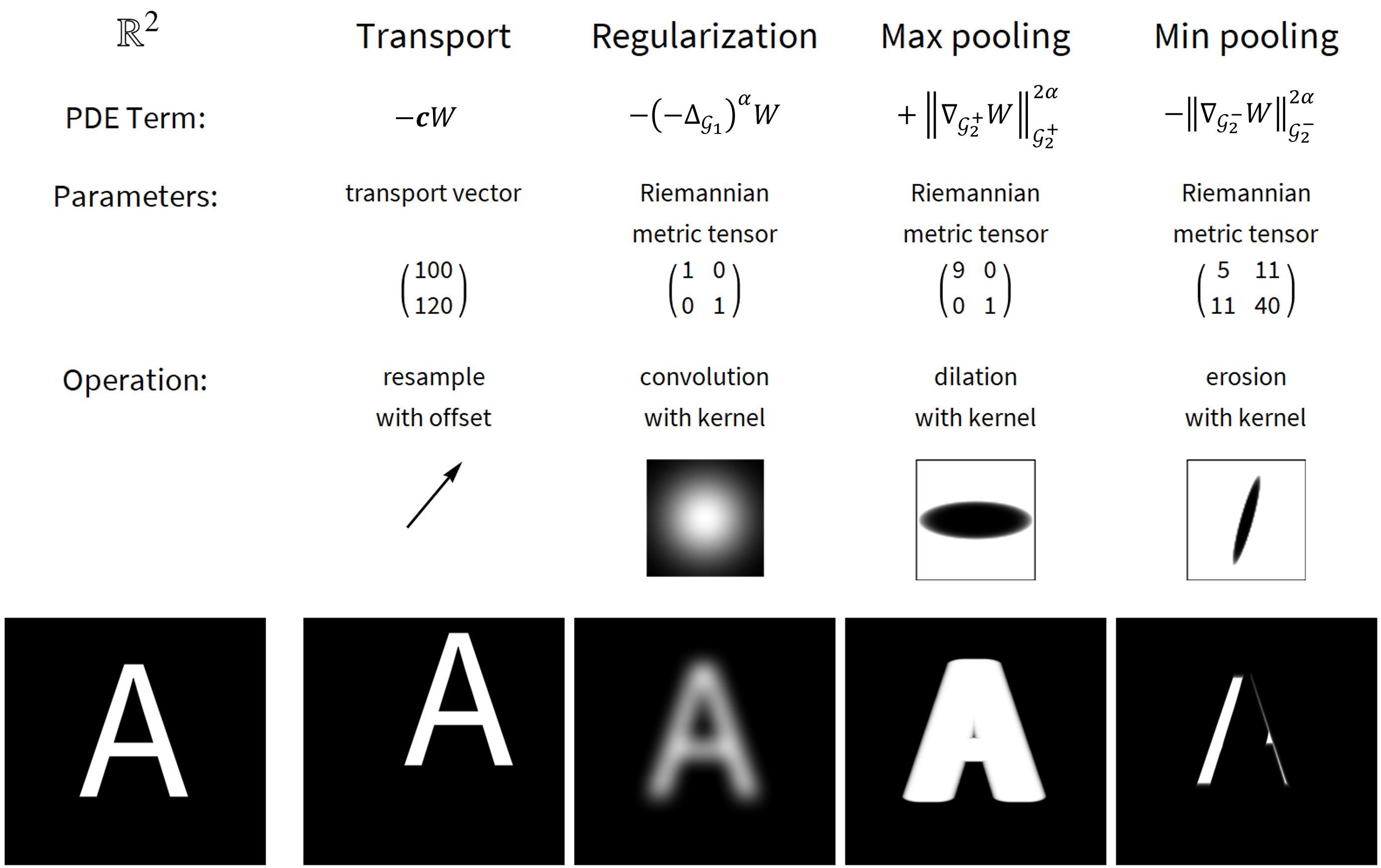}
    \caption[Geometric interpretation of the PDE]{Geometric interpretation of the terms of the PDE \eqref{eq:cnn_pde} illustrated for $\mathbb{R}^2$. In this setting the $G$-invariant vector field $\bm{c}$ is the constant vector field given by two translation parameters. For the other terms we use Riemannian metric tensors parametrized by a positive definite $2 \times 2$ matrix in the standard basis. The kernels used in the diffusion, dilation and erosion terms are functions of the distance-map induced by the metric tensors.}
    \label{fig:interpretation_r2}
\end{figure*}

\begin{figure*}[ht]
    \centering
    \includegraphics[width=0.9\linewidth]{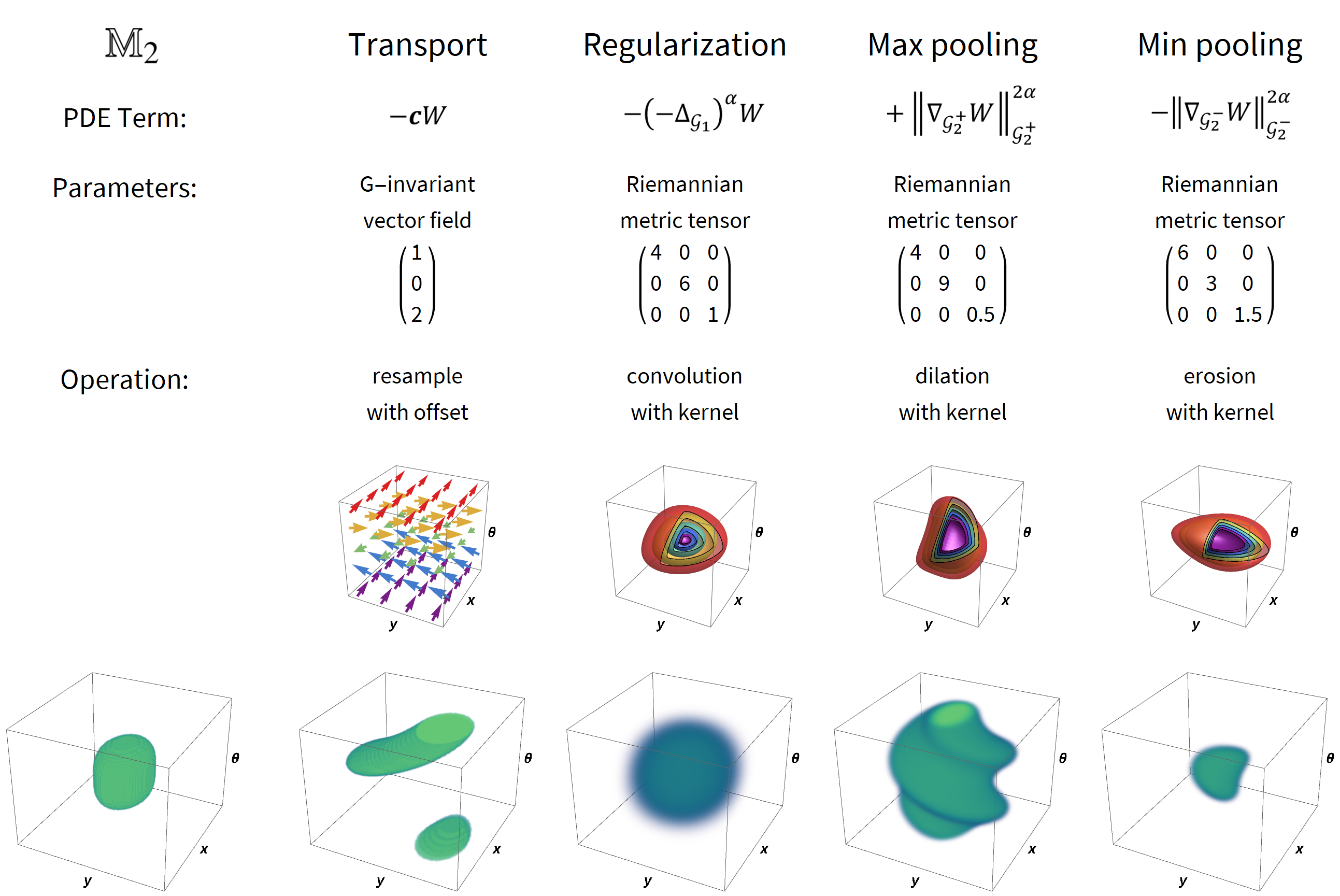}
    \caption[Geometric interpretation of the PDE]{Geometric interpretation of the terms of the PDE \eqref{eq:cnn_pde} illustrated for $\mathbb{M}_2$. In this setting the $G$-invariant vector field $\bm{c}$ is a left-invariant vector field given by two translation and one rotation parameter. For the other terms we use Riemannian metric tensors parametrized by a positive definite $3 \times 3$ matrix in the left-invariant basis (the matrix does not need to be diagonal but we keep that for future work). The kernels used in the diffusion, dilation and erosion terms are functions of the distance-map induced by the metric tensors and are visualized by partial plots of their level sets.}
    \label{fig:interpretation_m2}
\end{figure*}

\subsection{Training}

Training the PDE layer comes down to adapting the parameters in the PDEs in order to minimize a given loss function (the choice of which depends on the application and we will not consider in this article). In this sense, the vector field and the metric tensors are analogous to the weights of this layer.

Since we required the convection vector field and the metric tensor fields to be $G$-invariant, the parameter space is finite-dimensional as a consequence of Cor. \ref{cor:livf_properties} and \ref{cor:limtf_properties} if we restrict ourselves to Riemannian metric tensor fields.

For our main application on $\mathbb{M}_2$ each PDE unit would have the following $12$ trainable parameters:
\begin{itemize}
    \item 3 parameters to specify the convection vector field as a linear combination of \eqref{eq:livf_m2},
    \item 3 parameters to specify the fractional diffusion metric tensor field $\mathcal{G}_1$,
    \item and 3 parameters each to specify the dilation and erosion metric tensor fields $\mathcal{G}_2^{\pm}$,
\end{itemize}
where the metric tensor fields are of the form \eqref{eq:metric_m2} that are diagonal with respect to the frame from Prop.~\ref{prop:livf_m2}.

Surprisingly for higher dimensions $\mathbb{M}_d$ has less trainable parameters than for $d=2$. This is caused by the $SE(d)$-invariant vector fields on $\mathbb{M}_d$ for $d \geq 3$ being spanned by a single basis element (per Proposition~\ref{prop:livmd}) instead of the three \eqref{eq:livf_m2} basis elements available for $d=2$. Since the left-invariant metric tensor fields are determined by only $3$ parameters irrespective of dimensions we count a total of $7$ parameters for each PDE unit for applications on $\mathbb{M}_d$ for $d \geq 3$.

In our own experiments we always use some form of stochastic gradient descent (usually ADAM) with a small amount of $L^2$ regularization applied uniformly over all the parameters. Similarly we stick to a single learning rate for all the parameters. Given that in our setting different parameters have distinct effects treating all of them the same is likely far from optimal, however we leave that topic for future investigation.

\section{PDE Solver}
\label{sec:pde_solver}

Our PDE solver will consist of an iteration of time step units, each of which is a composition of convection, diffusion, dilation and erosion substeps. These units all take their input as an initial condition of a PDE, and produce as output the solution of a PDE at time $t=T$. 

The convection, diffusion and dilation/erosion steps are implemented with respectively a shifted resample, linear convolution, and two morphological convolutions, as illustrated in Fig.~\ref{fig:operator_splitting}.
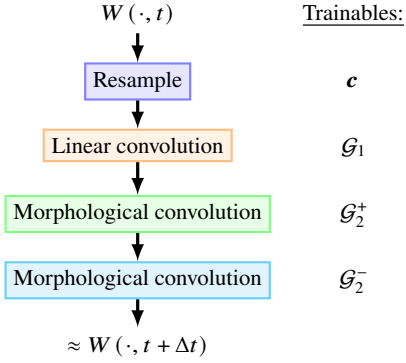
\begin{figure}[ht]
\begin{center}
\renewcommand*{\arraystretch}{2.5}
\begin{tabular}{cc}
    \tikznode[inner sep=3pt]{A}{W(\cdot,t)} 
    & 
    \underline{Trainables:}
    \\
    \tikznode[inner sep=3pt, fill=blue!10, draw=blue!50, thick]{B}{\textrm{Resample}}
    &
    $\bm{c}$
    \\
    \tikznode[inner sep=3pt, fill=orange!10, draw=orange!50, thick]{C}{\textrm{Linear convolution}}
    &
    $\mathcal{G}_1$
    \\
    \tikznode[inner sep=3pt, fill=green!10, draw=green!50, thick]{D}{\textrm{Morphological convolution}}
    &
    $\mathcal{G}_2^+$
    \\
    \tikznode[inner sep=3pt, fill=cyan!10, draw=cyan!50, thick]{E}{\textrm{Morphological convolution}}
    &
    $\mathcal{G}_2^-$
    \\
    \tikznode[inner sep=3pt]{F}{\approx W(\cdot,t+\Delta t)}
    &
\end{tabular}
\begin{tikzpicture}[overlay,remember picture,very thick]
    \draw[-latex, black] (A.south) --(B);
    \draw[-latex, black] (B.south) --(C);
    \draw[-latex, black] (C.south) --(D);
    \draw[-latex, black] (D.south) --(E);
    \draw[-latex, black] (E.south) --(F);
\end{tikzpicture}
\end{center}
\caption{Evolving the PDE through operator splitting, each operation corresponds to a term of \eqref{eq:cnn_pde}.}
\label{fig:operator_splitting}
\end{figure}
The composition of the substeps does not  solve \eqref{eq:cnn_pde} exactly, but for small $\Delta t$, it approximates the solution by a principle called \emph{operator splitting}.

We will now discuss each of these substeps separately.

\subsection{Convection}

The convection step has as input a function $U_1:G/H \to \mathbb{R}$ and takes it as initial condition of the PDE
\begin{equation}
\label{eq:convection_pde}
    \begin{cases}
    \frac{\partial W_1}{\partial t}(p,t)
    =
    - \bm{c}(p)  W_1(\,\cdot\, ,t) \ &\text{for } p \in G/H, t \geq 0,
    \\[1em]
    W_1(p,0) = U_1(p) & \text{for } p \in G/H.
    \end{cases}
\end{equation}
The output of the layer is the solution of the PDE evaluated at time $t=T$, i.e. the output is the function $p \mapsto W_1(p,T)$.

\begin{proposition}[Convection solution]
The solution of the convection PDE \eqref{eq:convection_pde} is found by the method of characteristics, and is given by
\begin{align} 
    \nonumber
    W^{1}(p,t) 
    &= \left(\mathcal{L}_{g_{p}^{-1}} U_1\right)\left( \gamma_{\bm{c}}(t)^{-1} p_0 \right) 
    \\
    &= U_1 \left( g_p \, \gamma_{\bm{c}}(t)^{-1}  p_0 \right)
    \\
    &= U_1 \left( g_p \, \gamma_{-\bm{c}}(t)  p_0 \right),
    \label{eq:convection_solution}
\end{align}
where $g_p \in p$ (i.e. $g_p p_0=p$) and $\gamma_{\bm{c}}:\mathbb{R} \to G$ is the exponential  curve that satisfies $\gamma_{\bm{c}}(0)=e$ and
\begin{equation}
    \frac{\partial}{\partial t} \left( \gamma_{\bm{c}}(t)  p \right)(t)
    =
    \bm{c} \left( \gamma_{\bm{c}}(t)  p \right),
\end{equation}
i.e. $\gamma_{\bm{c}}$ is the exponential curve in the group $G$ that induces the integral curves of the $G$-invariant vector field $\bm{c}$ on $G/H$ when acting on elements of the homogeneous space. 
\end{proposition}
Note that this exponential curve existing is a consequence of the vector field $\bm{c}$ being $G$-invariant, such exponential curves do not exist for general convection vector fields.

\begin{proof}
\begin{align*}
    \frac{\partial W_1}{\partial t} (p,t)
    &=
    \lim_{h \to 0}
    \frac{W_1(p,t+h)-W_1(p,t)}{h}
    \\
    &\hspace{-3.5em}=
    \lim_{h \to 0}
    \frac{
        U_1 \left( g_p \, \gamma_{\bm{c}}(t+h)^{-1} p_0 \right)
        -
        U_1 \left( g_p \, \gamma_{\bm{c}}(t)^{-1} p_0 \right)
    }{h}
    \\
    &\hspace{-3.5em}=
    \lim_{h \to 0}
    \frac{
        U_1 \left(  g_p \, \gamma_{\bm{c}}(t)^{-1} \, \gamma_{\bm{c}}(h)^{-1} p_0 \right)
        -
        U_1 \left( g_p \, \gamma_{\bm{c}}(t)^{-1} p_0 \right)
    }{h}
    ,
\intertext{now let $\bar{U}:=\mathcal{L}_{\gamma_{\bm{c}}(t)\,g_p^{-1}} U_1$, then}
    &=
    \lim_{h \to 0}
    \frac{
        \bar{U} \left(  \gamma_{\bm{c}}(h)^{-1} p_0 \right)
        -
        \bar{U} \left(p_0 \right)
    }{h}
    \\
    &=
    - \bm{c}(p_0) \, \bar{U}
    \\
    &=
    - \left( L_{g_p} \right)_* \bm{c}(p_0)
    \ \mathcal{L}_{g_p} \bar{U}
\intertext{due to the $G$-invariance of $\bm{c}$ this yields}
    &=
    - \bm{c}(p) \, \mathcal{L}_{g_p} \, \mathcal{L}_{\gamma_{\bm{c}}(t)\,g_p^{-1}} U_1
    \\
    &=
    - \bm{c}(p) \,
    \left[
        p \mapsto
        U_1 \left(
            g_p \gamma_{\bm{c}}(t)^{-1} g_p^{-1} p
        \right)
    \right]
    \\
    &=
    - \bm{c}(p) \,
    \left[
        p \mapsto
        U_1 \left(
            g_p \gamma_{\bm{c}}(t)^{-1} p_0
        \right)
    \right]
    \\
    &=
    - \bm{c}(p) \, W_1 (\cdot, t).
\end{align*}
\qed
\end{proof}

In our experiments equation \eqref{eq:convection_solution} is numerically implemented as a resampling operation with trilinear interpolation to account for the off-grid coordinates.

\subsection{Fractional Diffusion}

The (fractional) diffusion step solves the PDE
\begin{equation}
\label{eq:diffusion_system}
    \begin{cases}
    \frac{\partial W_2}{\partial t}
    =
    - \left( - \Delta_{\mathcal{G}_1} \right)^{\alpha} W_2 (p,t) \ 
    & \text{for } p \in G/H, t \geq 0, 
    \\[1em]
    W_2(p,0) = U_2(p) 
    & \text{for } p \in G/H.
    \end{cases}
\end{equation}
As with (fractional) diffusion on $\mathbb{R}^n$, there exists a smooth function
\[
K_{\cdot}^{\alpha} : (0,\infty) \times (G/H) \to [0,\infty),
\]
called the fundamental solution of the $\alpha$-diffusion equation, such that for every initial condition $U_2$, the solution to the PDE \eqref{eq:diffusion_system} is given by the convolution of the function $U_2$ with the fundamental solution $K_{t}^{\alpha}$:
\begin{equation} \label{eq:diffusion_solution}
    W^{2}(p,t) = \left( K_{t}^{\alpha} *_{G/H} U_2  \right)(p).
\end{equation}
The convolution $*_{G/H}$ on a homogeneous space $G/H$ is specified by the following definition.

\begin{definition}[Linear group convolution]
    \label{def:linear_convolution}
    Let $p_0=H$ be compact, let $f \in L^2 \left( G/H \right)$ and $k \in L^1 \left( G/H \right)$ such that:
    \begin{equation*}
        \forall h \in H,\ p \in G/H: k\left( h  p \right)=k\left( p \right)
        \tag{kernel compatibility}
    \end{equation*}
    then we define:
    \begin{equation}
        \label{eq:linear_group_convolution}
        \left( k *_{G/H} f \right) \left( p \right) 
        :=
        \frac{1}{\mu_H(H)}\; \int_G 
        k \left( g^{-1}  p \right)
        \ 
        f \left( g  p_0 \right)
        \d\mu_G(g)
        ,
    \end{equation}
    where $\mu_H$ and $\mu_G$ are the left-invariant Haar measures (determined up to scalar-multiplication) on $H$ respectively $G$.
\end{definition}
\begin{remark}
In the remainder of this article we refer the to the left-invariant Haar measure on $G$ as `the Haar measure on $G$' as right-invariant Haar measures on $G$ do not play a role in our framework.
\end{remark}

\begin{remark} \label{rem:Weil}
Compactness of $H$ is crucial as otherwise the integral in the righthand side of (\ref{eq:linear_group_convolution}) does not converge.
To this end we note that one can
always decompose (by Weil's integral formula  \cite[Lem. 2.1]{fuhr2005abstract}) the Haar measure $\mu_{G}$ on the group as a product of a measure on the quotient $G/H$ times the measure on the subgroup $H$. As Haar-measures are determined up to a constant we take the following convention:
we normalize the Haar-measure $\mu_G$ such that
\begin{equation} \label{eq:PFmeasure}
\mu_{G}\left( \pi^{-1}(A) \right) = \mu_H(H) \, \mu_{\mathcal{G}} \left( A \right),
\qquad \forall A \subset G/H,
\end{equation}
where $\mu_{\mathcal{G}}$ is the Riemannian measure induced by $\mathcal{G}$ and $\mu_H$ is a choice of Haar measure on $H$.
Thereby (\ref{eq:PFmeasure}) boils down to
 Weil's integration formula:
\begin{equation}
\mu_H(H) \int_{G/H} f(p) \; {\rm d}\mu_{\mathcal{G}}(p) =\int_{G} f(gH)\; {\rm d}\mu_G(g)
\end{equation}
whenever $f$ is measurable. Since $H$ is compact we can indeed normalize the Haar measure $\mu_H$ so that $\mu_H(H)=1$.
\end{remark}
In general an exact analytic expression for the fundamental solution $K_{t}^{\alpha}$ requires complicated steerable filter operators \cite[Thm. 1 \& 2]{duits2019fourier} and for that reason we contend ourselves with more easily computable approximations. 
For now let us construct our approximations and address their quality and the involved asymptotics later.
\begin{remark}
In the approximations we will make use of logarithmic map as the inverse of the Lie group exponential map $\exp_G$. 
Locally, such inversion can always be done by the inverse function theorem.
Specifically, there is always a neighborhood $V \subset T_e G$ of the origin so that $\exp_G \vert_V$ is a diffeomorphism between $V$ and $W=\exp_G(V) \subset G$, where $W$ is a neighborhood of $e$.
Then we define the logarithmic map $\log_G :W \to V$
by $\textrm{exp}_G \circ \log_{G} =
\textrm{id}_W$
and $
\log_{G}
\circ \left.\textrm{exp}_G \right|_{V}=\textrm{id}_{V}$.
For the moment, for simplicity, we assume $V=T_e(G)$ in the general setting\footnote{In our primary case of interest $G=SE(2)$
we have $V=\{\sum_{k=1}^3 c^k A_k \;|\; c^{3} \in [-\pi,\pi)\}$.
}. 
\end{remark}

The idea is that instead of basing our kernels on the metric $d_{\mathcal{G}}$ (which is hard to calculate \cite{duits2018optimal}) we approximate it using the seminorm from Def.~\ref{def:seminorm} (which is easy to calculate). We can use this seminorm on elements of the homogeneous space by using the group's logarithmic map $\log_G$. We can take the group logarithm of all the group elements that constitute a particular equivalence class of $G/H$ and then pick the group element with the lowest seminorm:
\begin{equation} \label{eq:approxd}
    d_{\mathcal{G}} \left( p_0, p \right)
    \approx
    \inf_{g \in p} \left\Vert \log_G g \right\Vert_{\tilde{\mathcal{G}}}.
\end{equation}
Henceforth, we write this estimate as $d_{\mathcal{G}} \left( p_0, p \right) \approx \rho_{\mathcal{G}}(p)$ relying on the following definition.
\begin{definition}[Logarithmic metric estimate]
    \label{def:metric_estimate}
    Let $\mathcal{G}$ be a $G$-invariant metric tensor field on the homogeneous space $G/H$, then we define
    \begin{equation} \label{eq:rho}
        \begin{split}
        \rho_{\mathcal{G}}(p) 
        &:= 
        \inf_{g \in p} \left\Vert \log_G g \right\Vert_{\tilde{\mathcal{G}}}
        \\
        &:=
        \inf_{g \in p}
        \sqrt{\mathcal{G}\left( \pi_* \log_G g,\ \pi_* \log_G g \right)}
        ,
        \end{split}
    \end{equation}
\end{definition}
where $\pi_*$ is the push-forward of the projection map $\pi$ given by (\ref{eq:quotientmap}).

We can interpret this metric estimate as finding all exponential curves in $G$ whose actions on the homogeneous space connect $p_0$ (at $t=0$) to $p$ (at $t=1$) and then from that set we choose the exponential curve that has the lowest constant velocity according to the seminorm in Def. \ref{def:seminorm} and use that velocity as the distance estimate. 

Summarizing, Def.~\ref{def:metric_estimate} and Eq.~(\ref{eq:approxd}), can be intuitively reformulated as: `instead of the length of the geodesic connecting two points of $G/H$ we take the length of the shortest exponential curve connecting those two points'.

The following lemma quantifies how well our estimate approximates the true metric.

\begin{lemma}[Bounding the logarithmic metric estimate] \label{th:core}
    \noindent
    \label{lem:bounding-metric}
   For all $p \in G/H$ sufficiently close to $p_0$ we have
   \begin{equation*}
       d_{\mathcal{G}}(p_0,p)^2
       \leq
       \rho_{\mathcal{G}}(p)^2
       \leq
       d_{\mathcal{G}}(p_0,p)^2
       +
       O \left( d_{\mathcal{G}}(p_0,p)^4 \right)
       ,
   \end{equation*}
   which has as a weaker corollary that for all compact neighborhoods of $p_0$ there exists a $C_{\mathrm{metr}}>1$ so that
   \begin{equation*}
        d_{\mathcal{G}}(p_0, p) \leq \rho_{\mathcal{G}}(p) \leq C_{\mathrm{metr}} \, d_{\mathcal{G}}(p_0,p)
    \end{equation*}
    for all $p$ in that neighborhood.
    Note that the constant $C_{\mathrm{metr}}$ depends on both the choice of compact neighborhood and the metric tensor field. 
\end{lemma}

The proof of this lemma can be found in Appendix~\ref{appendix:proof-of-bounding-lemma}. 

\begin{remark}[Logarithmic metric estimate in principal homogeneous spaces]
    When we take a principal homogeneous space such as $\mathbb{M}_2 \equiv SE(2)$ with a left-invariant metric tensor field the metric estimate simplifies to
    \begin{equation*}
        \label{eq:metric_estimate}
        \rho_{\mathcal{G}}(g)
        =
        \left\Vert
            \log_G g
        \right\Vert_{\mathcal{G}\vert_{e}}
        ,
    \end{equation*}
    hence we see that this construction generalizes the logarithmic estimate, as used in \cite{portegies2015new,terelst1998weighted}, to homogeneous spaces other than the principal.
\end{remark}

\begin{remark}[Logarithmic metric estimate for $\mathbb{M}_2$]
    Using the $(x,y,\theta)$ coordinates for $\mathbb{M}_2$ and a left-invariant metric tensor field of the form \eqref{eq:metric_m2} we formulate the metric estimate in terms of the following auxiliary functions called the exponential coordinates of the first kind:
    \begin{align*} \label{eq:ctjes}
        c^1(x,y,\theta)
        &:=
        \begin{cases}
            \frac{\theta}{2} \left( y + x \cot{\frac{\theta}{2}} \right)
            \qquad &\text{if } \theta \neq 0,
            \\
            x
            &\text{if } \theta = 0,
        \end{cases}
        \\
        c^2(x,y,\theta)
        &:=
        \begin{cases}
            \frac{\theta}{2} \left( -x + y \cot{\frac{\theta}{2}} \right)
            \qquad &\text{if } \theta \neq 0,
            \\
            y
            &\text{if } \theta = 0,
        \end{cases}
        \\
        c^3(x,y,\theta) &:= \theta.
    \end{align*}
    The logarithmic metric estimate for $SE(2)$ is then given by
    \begin{gather*}
        \rho_{\mathcal{G}}(x,y,\theta)
        =
        \\
        \sqrt{w_M \ c^1(x,y,\theta)^2 + w_L \ c^2(x,y,\theta)^2 + w_A \ c^3(x,y,\theta)^2},
    \end{gather*}
    this estimate is illustrated in Fig. \ref{fig:distance_comparison} where it is contrasted against the exact metric.
\end{remark}

\begin{figure}[ht]
    \centering
    \begin{subfigure}{0.45\linewidth}
        \centering
        \includegraphics[width=0.9\linewidth]{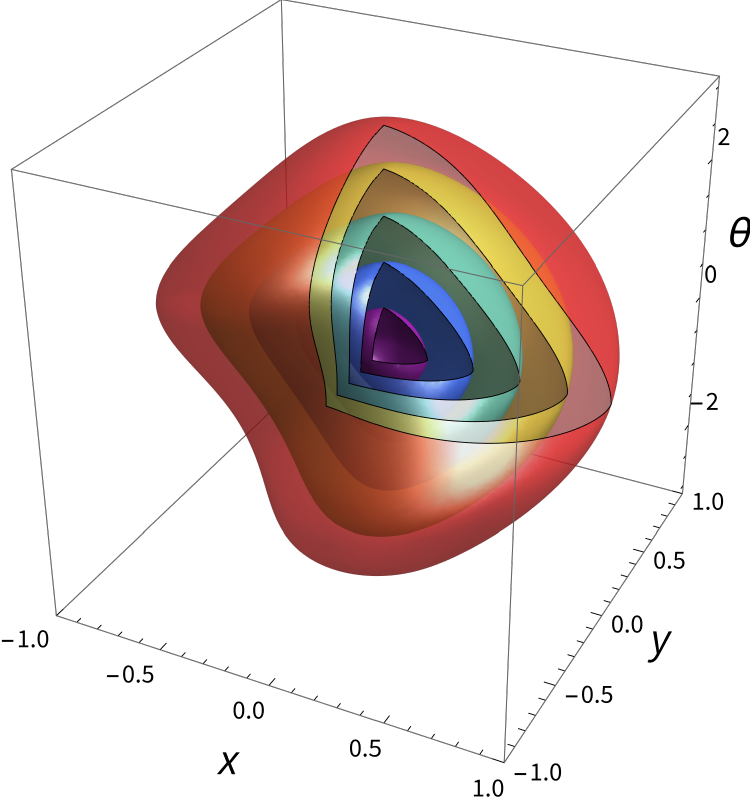}
    \end{subfigure}
    \begin{subfigure}{0.45\linewidth}
        \centering
        \includegraphics[width=0.9\linewidth]{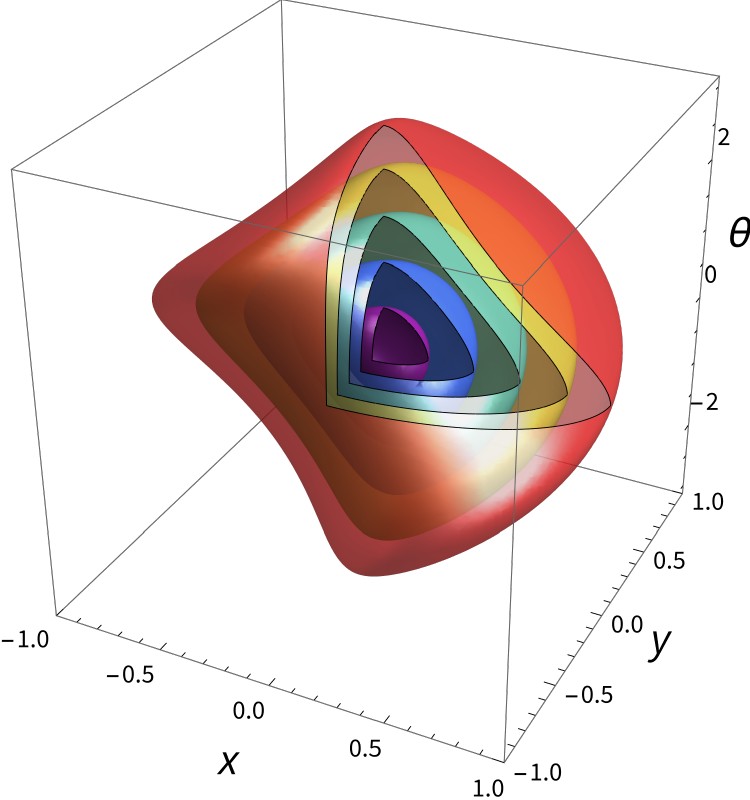}
    \end{subfigure}
    \includegraphics[width=0.4\linewidth]{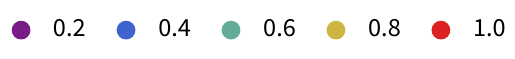}
    \caption{Comparing the `exact' Riemannian distance (left) obtained through numerically solving the Eikonal equation \cite{bekkers2015pde} versus the logarithmic metric estimate (right) on $SE(2)$ endowed with a left-invariant Riemannian metric tensor field \eqref{eq:metric_m2} with $w_M=1$, $w_L=2$, $w_A=1/\pi$. The relative $L^1$ error in the plotted volume is $0.20$.}
    \label{fig:distance_comparison}
\end{figure}

We can see that the metric estimate $\rho_{\mathcal{G}}$ (and consequently any function of $\rho_{\mathcal{G}}$) has the necessary compatibility property to be a kernel used in convolutions per Def.~\ref{def:linear_convolution}.
\begin{proposition}[Kernel compatibility of $\rho_{\mathcal{G}}$]
    \label{prop:compatibility}
   Let $\mathcal{G}$ be a $G$-invariant metric tensor field on $G/H$, then we have
    \begin{equation}
        \forall p \in G/H,\ \forall h \in H : \rho_{\mathcal{G}}(hp) = \rho_{\mathcal{G}}(p).
    \end{equation}
\end{proposition}

Note that, since we use left cosets, $p h = p$ but $h p \neq p$ in general, this requirement is not trivial.
Proof of this proposition is included in Appendix \ref{appendix:proof-of-compatibility}.

Now that we have developed and analyzed the logarithmic metric estimate we can use it to construct an approximation to the diffusion kernel for $\alpha=1$.

\begin{definition}[Approximate $\alpha=1$ kernel]
    \label{def:diffusion_kernel_1}
    \begin{gather}
        K_t^{1,\mathrm{appr}}(p) 
        := \eta_t \exp{\left( - \frac{\rho_{\mathcal{G}}(p)^2}{4t} \right)}
    \end{gather}
    where $\eta_t$ is a normalization constant for a given $t$, this can either be the $L^1$ normalization constant or in the case of groups of polynomial growth one typically sets
    $\eta_t=\mu_{\mathcal{G}}\left( B(p_0,\sqrt{t}) \right)^{-1}$,
    see the definition of polynomial growth below
    .
\end{definition}

On Lie groups of polynomial growth this approximate kernel be bounded from above and below by the exact kernels.

\begin{definition}[Polynomial growth]
    A Lie group $G$ with left-invariant Haar measure $\mu_G$ is of polynomial growth
    when the volume of a sphere of radius $r$ around $g \in G$:
    \begin{equation*}
    B(g,r)=\left\{ g' \in G \ \big\vert\  d_{\tilde{\mathcal{G}}}(g,g') < r \right\},
    \end{equation*}
    can be polynomialy bounded as follows: there exists constants $\delta>0$ and $C_{\mathrm{grow}}>0$ so that
    \begin{equation*}
        \frac{1}{C_{\mathrm{grow}}} r^\delta \leq \mu_G \left( B(g, r) \right) \leq C_{\mathrm{grow}} r^\delta
        , \qquad r \geq 1,
    \end{equation*}
    take note that the exponent $\delta$ is the same on both the lower and upper bound.
    Since $\mu_G$ is left-invariant the choice of $g$ does not matter.
\end{definition}

\begin{lemma}
\label{lem:heat_kernel_squeeze}
Let $G$ be of polynomial growth and let $K_t^1$ be the fundamental solution to the $\alpha=1$ diffusion equation on $G/H$ then there exists constants $C\geq 1$ , $D_1 \in (0,1)$ and $D_2>D_{1}$
so that for all $t>0$ the following holds:
\begin{align}
    \begin{split}
        \frac{1}{C} K^1_{D_1 t}(p)
        \leq
        K_t^{1,\mathrm{appr}}(p)
        \leq
        C K^1_{D_2 t}(p).
    \end{split}
\end{align} 
for all $p \in G/H$.
\end{lemma}

\begin{proof}
    On a group of polynomial growth we have
    $\eta_t=\mu_{\mathcal{G}}\left( B(p_0,\sqrt{t}) \right)^{-1}$.
    If $G$ is of polynomial growth we can apply \cite[Thm. 2.12]{Maheux} to find that there exists constants $C_1,C_2 >0$ and for all $\varepsilon>0$ there exists a constant $C_\varepsilon$ so that:
    \begin{equation*}
        \label{eq:enclosure_1}
    \begin{split}
        &C_1 \eta_t \exp{\left( -\frac{d_{\mathcal{G}}(p_0,p)^2}{4C_2 t} \right)}
        \leq
        K^1_t(p)
        \\
        &\leq
        C_{\varepsilon} \eta_t \exp{\left( -\frac{d_{\mathcal{G}}(p_0,p)^2}{4(1+\varepsilon)t} \right)}
        .
    \end{split}
    \end{equation*}
    
    \begin{remark}[Left vs. right cosets]
    Note that Maheux \cite{Maheux} uses right cosets while we use left cosets. We can translate the results easily by inversion in view of $(gH)^{-1}=H^{-1} g^{-1}=H g^{-1}$. 
    We then apply the result of Maheux to the correct (invertible) $G$-invariant metric tensor field on $G/H$.
    
    Also note the different (but equivalent) way Maheux relates distance on the group with distance on the homogeneous space. While we use a pseudometric on $G$ induced by a metric on $G/H$, Maheux uses a metric on $G/H$ induced by a metric on $G$ by:
    \begin{equation}
        \label{eq:metric_maheux}
        \begin{split}
        d_{G/H}^{\mathrm{maheux}}(p_1,p_2) 
        &=\inf_{g_1 \in p_1}\inf_{g_2 \in p_2} d_{G}^{\mathrm{maheux}}(g_1, g_2)
        \\
        &=\inf_{g_2 \in p_2} d_{G}^{\mathrm{maheux}}(q_1, g_2), 
        \end{split}
    \end{equation}
    for any choice of $q_1 \in p_1$. We avoid having to minimize inside the cosets as in \eqref{eq:metric_maheux} thanks to the inherent symmetries in our pseudometric.
    \end{remark}

    Now using the inequalities from Lemma \ref{lem:bounding-metric} we obtain:
    \begin{equation*}
        \label{eq:enclosure_2}
        \begin{split}
            &C_1 \eta_t \exp{\left( -\frac{\rho_{\mathcal{G}}(p)^2}{4C_2 t} \right)}
            \leq
            K^1_t(p)
            \\
            &\leq
            C_{\varepsilon} \eta_t \exp{\left( -\frac{\rho_{\mathcal{G}}(p)^2}{4C_{\mathrm{metr}}^2 (1+\varepsilon)t} \right)}
            ,
        \end{split}
    \end{equation*}
    which leads to:
    \begin{equation*}
        \label{eq:enclosure_3}
        \begin{split}
            &C_1 \frac{\eta_t}{\eta_{c_2 t}} K^{1,\mathrm{appr}}_{C_2 t}(p)
            \leq
            K^1_t(p)
            \\
            &\leq
            C_{\varepsilon} \frac{\eta_t}{\eta_{C^2_{\mathrm{metr}}(1+\varepsilon)t}} 
            K^{1,\mathrm{appr}}_{C^2_{\mathrm{metr}}(1+\varepsilon)t}(p)
            .
        \end{split}
    \end{equation*}
    
    The group $G$ being of polynomial growth also implies $G/H$ is a doubling space \cite[Thm. 2.17]{Maheux}. Using the volume doubling and reverse volume doubling property of doubling spaces \cite[Prop. 3.2 and 3.3]{grigor2009heat} we find that there exist constants $C_3,C_4,\beta,\beta'>0$ so that:
    \begin{gather*}
         \frac{\eta_t}{\eta_{c_2 t}}
         \geq
         C_3 \left( \frac{\sqrt{t}}{\sqrt{C_2 t}} \right)^\beta
         =
         C_3 C_2^{-\beta/2}
         ,
         \\
         \begin{split}
         \frac{\eta_t}{\eta_{C^2_{\mathrm{metr}}(1+\varepsilon)t}}
         &\leq
         C_4 \left( \frac{\sqrt{t}}{\sqrt{C^2_{\mathrm{metr}}(1+\varepsilon)t}} \right)^{\beta'}
         \\
         &=
         C_4 \left( C^2_{\mathrm{metr}}(1+\varepsilon)  \right)^{-\beta'/2}
        .
        \end{split}
    \end{gather*}
    Applying these inequalities we get: 
    \begin{equation*}
        C_1' := C_1 C_3 C_2^{-\beta/2}
    \end{equation*} 
    and 
    \begin{equation*}
        C_\varepsilon' := C_\varepsilon C_4 \left( C^2_{\mathrm{metr}}(1+\varepsilon)  \right)^{-\beta'/2}
    \end{equation*} we obtain:
    \begin{equation*}
        \label{eq:enclosure_4}
            C_1'  K^{1,\mathrm{appr}}_{C_2 t}(p)
            \leq
            K^1_t(p)
            \leq
            C_{\varepsilon}'
            K^{1,\mathrm{appr}}_{C^2_{\mathrm{metr}}(1+\varepsilon)t}(p)
            .
    \end{equation*}
    Reparametrising $t$ in both inequalities gives:
    \begin{equation*}
        \frac{1}{C_\varepsilon'} K^1_{t/(C^2_{\mathrm{metr}}(1+\varepsilon))}(p)
        \leq
        K_t^{1,\mathrm{appr}}(p)
        \leq
        \frac{1}{C_1^\prime} K^1_{C_2^{-1}t}(p).
    \end{equation*}
    Finally we fix $\varepsilon>0$ and relabel constants:
    \begin{align*}
        &C := \max\left\{ C_1^{\prime -1},\ C_\varepsilon',\ 1 \right\}
        ,
        \\
        &D_1 := \frac{1}{C^2_{\mathrm{metr}}(1+\varepsilon)}
        ,
        \\
        &D_2 := \frac{1}{C_2},
    \end{align*}
    observe that since $\varepsilon>0$ and $C_{\mathrm{metr}} \geq 1$ we have $0<D_1< 1$.
    \qed
\end{proof}

Depending on the actually achievable constants, Lem.~\ref{lem:heat_kernel_squeeze} provides a very strong or very weak bound on how much our approximation deviates from the fundamental solution. 
Fortunately in the $SE(2)$ case our approximation is very close to the exact kernel in the vicinity of the origin, as can be seen in Fig.~\ref{fig:heat_kernel_approx}.
In our experiments we sample the kernel on a grid around the origin, hence this approximation is good for reasonable values of the metric parameters, which we may expect from Lemma~\ref{lem:bounding-metric} providing a second order relative error.

\begin{figure}[ht!]
    \centering
    \begin{subfigure}{0.45\linewidth}
        \centering
        \includegraphics[width=0.9\linewidth]{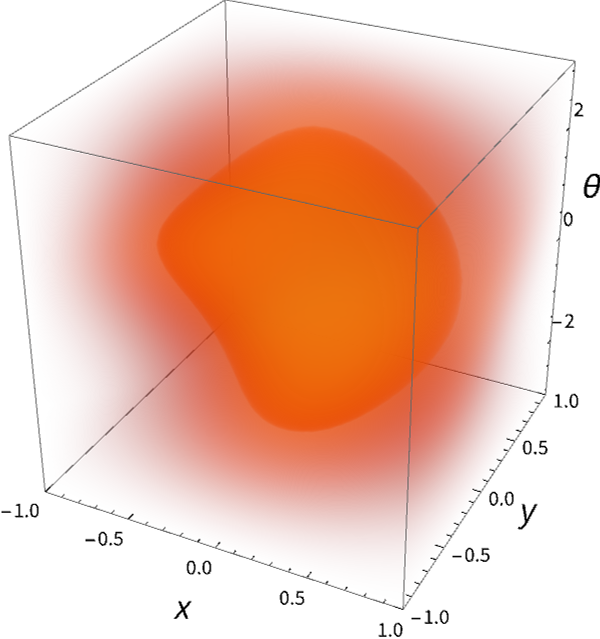}
    \end{subfigure}
    \begin{subfigure}{0.45\linewidth}
        \centering
        \includegraphics[width=0.9\linewidth]{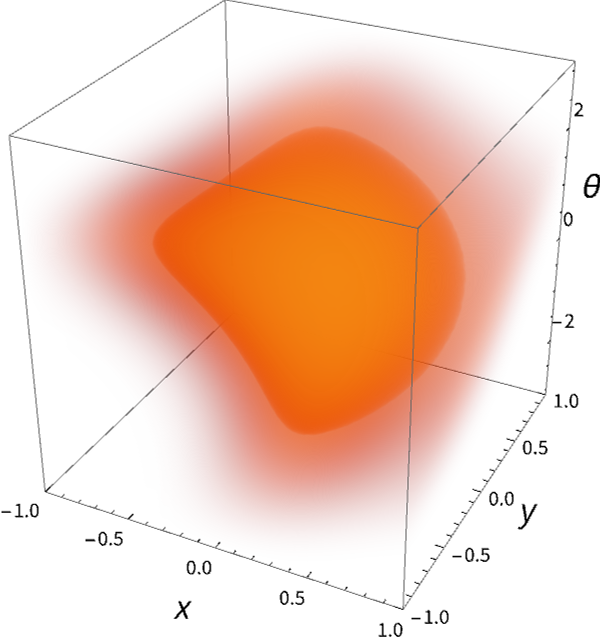}
    \end{subfigure}
    \caption{Comparing the numerically computed heat kernel $K_t^1$ (left) with our approximation $K_t^{1,\mathrm{appr}}$ based on the logarithmic norm estimate (right) for $G/H=SE(2)$. Shown here at $t=1$ with the same metric as in Fig.~\ref{fig:distance_comparison}. Especially in deep learning applications where discretization is very coarse our approximation is sufficiently accurate as long as the spatial anisotropies $w_M/w_L$ and $w_L/w_M$ do not become too high. In this case with $w_L/w_M=2$ we have a relative $L^2$ error of $0.23$ in the plotted volume.}
    \label{fig:heat_kernel_approx}
\end{figure}

Now let us develop an approximation for values of $\alpha$ other than $1$. From semi-group theory \cite{yosida1968functional} it follows that semi-groups generated by taking fractional powers of the generator (in our case $\Delta_{\mathcal{G}} \rightarrow -(-\Delta_{\mathcal{G}})^{\alpha}$) amounts to the following key relation between the $\alpha$-kernel and the diffusion kernel:
\begin{equation}
\label{keyrelation}
K_{t}^{\alpha}(p) := \int_{0}^{\infty} q_{t,\alpha}(\tau) \, K_{\tau}^{1}(p) \; {\rm d}\tau,
\end{equation}
for $\alpha \in (0,1)$ and $t>0$ where
$q_{t,\alpha}$ is defined as follows.
\begin{definition}
    Let $\mathcal{L}^{-1}$ be the inverse Laplace transform then
    \begin{equation*}
        q_{t,\alpha}(\tau)
        :=
        \mathcal{L}^{-1} \left( r \mapsto e^{-t r^{\alpha}} \right)(\tau)
        \qquad
        \hfill\mathrm{for }\ \tau \geq 0
        .
    \end{equation*}
\end{definition}
For explicit formulas of this kernel see \cite[Ch.~IX:11 eq.~17]{yosida1968functional}. Since $e^{-t r^\alpha}$ is positive for all $r$ it follows that $q_{t,\alpha}$ is also positive everywhere.

Now instead of integrating $K_t^1$ to obtain the exact fundamental solution, we can replace it with our approximation $K_t^{1,\mathrm{appr}}$ to obtain an approximate $\alpha$-kernel.

\begin{definition}[Approximate $\alpha \in (0,1)$ kernel]
    \label{def:alpha_kernel_approx}
    Akin to (\ref{keyrelation}) we set $\alpha \in (0,1)$, $t>0$ and define:
    \begin{equation} 
    \label{eq:alpha_kernel_approx}
    K_{t}^{\alpha,\mathrm{appr}}(p) := \int_{0}^{\infty} q_{t,\alpha}(\tau) \, K_{\tau}^{1,\mathrm{appr}}(p) \; {\rm d}\tau \geq 0,
    \end{equation}
    for $p \in G/H$.
\end{definition}

The bounding of $K_t^1$ we obtained in Lem.~\ref{lem:heat_kernel_squeeze} transfers directly to our approximation for other $\alpha$. 

\begin{theorem}
    \label{thm:alpha_kernel_squeeze}
Let $G$ be of polynomial growth and let $K_t^\alpha$ be the fundamental solution to the $\alpha\in(0,1]$ diffusion equation on $G/H$, then there exists constants $C\geq 1$ , $D_1 \in (0,1)$ and 
$D_{2}>D_{1}$
so that for all $t>0$ and $p \in G/H$ the following holds:
\begin{align}
    \begin{split}
        \frac{1}{C} K^\alpha_{D_1^\alpha t}(p)
        \leq
        K_t^{\alpha,\mathrm{appr}}(p)
        \leq
        C K^\alpha_{D_2^\alpha t}(p).
    \end{split}
\end{align}   
\end{theorem}

\begin{proof}
    This is an consequence of Lem.~\ref{lem:heat_kernel_squeeze} and the fact that $q_{t,\alpha}$ is positive, applying the integral from \eqref{keyrelation} yields:
    \begin{equation*}
    \begin{split}
        K_t^{\alpha,\mathrm{appr}}(p) 
        \ &= \  
        \int_0^\infty q_{t,\alpha}(\tau) K_\tau^{1,\mathrm{appr}}(p) \d\tau
        \\
        {\scriptstyle\mathrm{(Lem.~\ref{lem:heat_kernel_squeeze})}} &\leq
        C \int_0^\infty q_{t,\alpha}(\tau) K_{D_2\tau}^1(p) \d\tau
        \\
        {\scriptstyle(\tau'= D_2 \tau)}&=
        C \int_0^\infty \frac{1}{D_2} q_{t,\alpha}\left( \frac{\tau'}{D_2} \right) K_{\tau'}^1(p) \d\tau'
        \\
        {\scriptstyle\left(\substack{\mathrm{Bromwich}\\ \mathrm{integral}}\right)}&=
        C\int_0^\infty q_{D_2^\alpha t,\alpha}\left( \tau' \right) K_{\tau'}^1(p) \d\tau'
        \\
        &=
        C K^{\alpha}_{D_2^\alpha t}(p).
    \end{split}
    \end{equation*}
    The other inequality works the same way.
    \qed
\end{proof}

Although the approximation \eqref{eq:alpha_kernel_approx} is helpful in the proof above it contains some integration and is not an explicit expression. 
Our initial experiments with diffusion for $\alpha=1$ showed that (for the applications under consideration at least) adding diffusion did not improve performance.
For that reason we chose not to focus further on diffusion in this work.
We leave developing a more explicit and computable approximation for diffusion kernels for $0<\alpha<1$ for future work.

\begin{figure}[ht]
    \centering
    \includegraphics[width=0.8\linewidth]{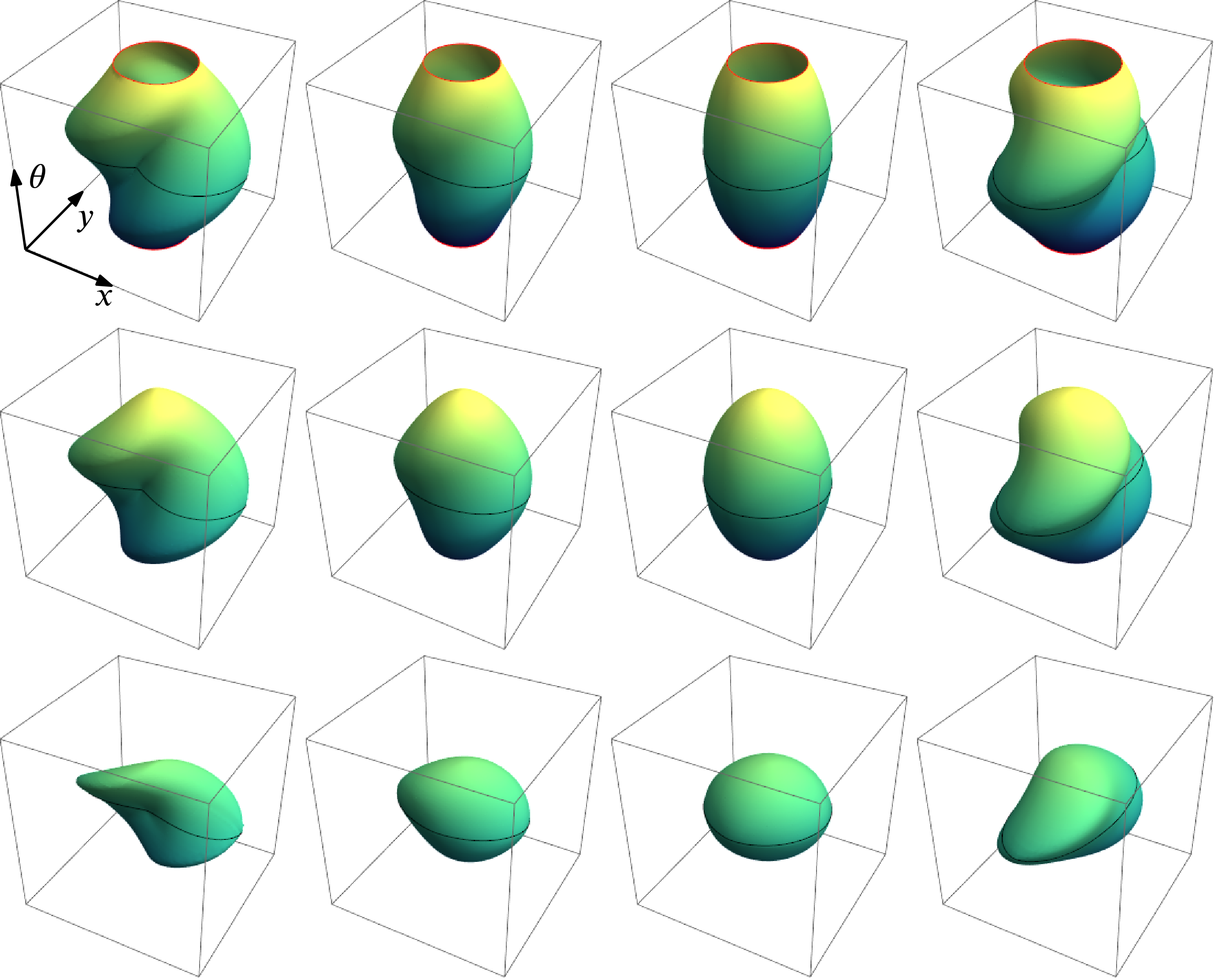}
    \caption[$SE(2)$ kernel shapes]{Shapes of the level sets of the kernels on $\mathbb{M}_2$ for solving fractional diffusion ($K^{\alpha}_t$) and dilation/erosion ($k^{\alpha}_t$) for various values of the trainable metric tensor field parameters $w_M,w_L$ and $w_A$. This shape is essentially what is being optimized during the training process of a metric tensor field on $\mathbb{M}_2$.}
    \label{fig:shapes}
\end{figure}

\subsection{Dilation and Erosion}
The dilation/erosion step solves the PDE
\begin{equation}
\label{eq:erosion_system}
    \begin{cases}
    \frac{\partial W_3}{\partial t}(p,t)
    =\pm
    \left\Vert \nabla_{\mathcal{G}_2^\pm} W_3 (p,t) \right\Vert^{2 \alpha}_{\mathcal{G}_2^\pm} \ &
    \text{for } p \in G/H,
    \\
    & t \geq 0,
    \\[1em]
    W_3(p,0) = U_3(p)&
    \text{for } p \in G/H.
    \end{cases}
\end{equation}

By a generalization of the Hopf-Lax formula \cite[Ch.10.3]{evans2010partial}, the solution is given by morphological convolution 
\begin{equation}
\label{eq:solve_dilation}
W_3(p,t) = -\left(k^{\alpha}_{t} \morph_G -U_3\right)(p)
\end{equation}
for the $(+)$ (dilation) variant and 
\begin{equation}
\label{eq:solve_erosion}
W_3(p,t) = \left(k^{\alpha}_{t} \morph_G U_3\right)(p)
\end{equation}
for the $(-)$ (erosion) variant,
where the kernel $k^{\alpha}_{t}$ (also called the structuring element in the context of morphology) is given as follows.

\begin{definition}[Dilation/erosion kernels]
    \label{def:morph_kernel}
    The morphological convolution kernel $k^{\alpha}_{t}$ for small times $t$ and $\alpha\in\left( \sfrac{1}{2},1\right]$ is given by
    \begin{equation}
        \label{eq:dilation_kernel}
        k^{\alpha}_{t} (p)
        :=
        \nu_{\alpha}
       t^{-\frac{1}{2 \alpha -1}}
        \d_{\mathcal{G}_2}(p_0,p)^{\frac{2 \alpha}{2 \alpha -1}},
    \end{equation}
    with $\nu_{\alpha}:=\left( \frac{2 \alpha - 1}{ \left(2 \alpha\right)^{2\alpha/(2\alpha-1)}} \right)$
    and for $\alpha=\sfrac{1}{2}$ by
    \begin{equation}
        \label{eq:dilation_kernel_1/2}
        k^{\sfrac{1}{2}}_{t} (p)
        =
        \begin{cases}
            0 \qquad &\text{if } \d_{\mathcal{G}_2}(p_0,p) \leq t,
            \\
            \infty &\text{if } \d_{\mathcal{G}_2}(p_0,p) > t.
        \end{cases}
    \end{equation}
\end{definition}
In the above definition and for the rest of the section we write $\mathcal{G}_2$ for either $\mathcal{G}_2^+$ or $\mathcal{G}_2^-$ depending on whether we are dealing with the dilation or erosion variant.
The morphological convolution $\morph_G$ (alternatively: the infimal convolution) is specified as follows.
\begin{definition}[Morphological group convolution]
    \label{def:morphological_convolution}
    Let $f \in L^\infty\left( G/H \right)$, let $k: G/H \to \mathbb{R} \cup \left\{ \infty \right\}$ be proper (not everywhere $\infty$) then we define:
    \begin{equation*}
    \begin{array}{ll}
        \left( k \morph_G f \right) (p)
        &:=
        \inf \limits_{g \in G}
        \left\{ k\left( g^{-1}  p \right) + f\left( g  p_0  \right) \right\}
        \\
         &=
           \inf \limits_{g \in G}
        \left\{ k\left( g^{-1}  p \right) + f\left( g H  \right) \right\}
        .
        \end{array}
    \end{equation*}
\end{definition}

\begin{remark}[Grayscale morphology]
    Morphological convolution is related to the grayscale morphology operations $\oplus$ (dilation) and $\ominus$ (erosion) on $\mathbb{R}^d$ as follows:
    \begin{align*}
        f_1 \oplus f_2 &= - \left( - f_1 \morph_{\mathbb{R}^d} -f_2 \right),
        \\
        f_1 \ominus f_2 &= f_1 \morph_{\mathbb{R}^d} \left[x \mapsto - f_2 \left( -x \right) \right],
    \end{align*}
    where $f_1$ and $f_2$ are proper functions on $\mathbb{R}^d$.
    Hence our use of the terms dilation and erosion, but mathematically we will only use $\morph_G$ as the actual operation to be performed and avoid $\oplus$ and $\ominus$.
\end{remark}

Combining morphological convolution with the structuring element $k_t^\alpha$ allows us to solve \eqref{eq:erosion_system}.

\begin{theorem}
\label{thm:erosion_dilation_solution}
Let $G$ be of polynomial growth, let $\alpha\in(\sfrac{1}{2},1]$ and let $U_3 : G/H \to \mathbb{R}$ be Lipschitz. Then $W_3 :  G/H \times (0,\infty) \to \mathbb{R}$ given by
\[
W_3(p, t) := (k_t^\alpha \morph_G U_3)(p)
\]
is Lipschitz and solves the $(-)$-variant, the erosion variant, of the system (\ref{eq:erosion_system}) in the sense of Theorem 2.1 in \cite{balogh2012functional}, while
\[
W_3(p, t) := -(k_t^\alpha \morph_G -U_3)(p)
\]
is Lipschitz and solves the $(+)$-variant, the dilation variant, of system (\ref{eq:erosion_system}) in the sense of Theorem 2.1 in \cite{balogh2012functional}. The kernels satisfy the semigroup property
\[
k_t^{\alpha} \morph_{G} k_{s}^{\alpha} = k_{t+s}^{\alpha}
\]
for all $s,t\geq 0$ and $\alpha \in (\sfrac{1}{2},1]$.
\end{theorem}

\begin{proof}
The Riemannian manifold $(G/H, \mathcal{G}_2)$ is a proper length space, and therefore the theory of \cite{balogh2012functional} applies.
Moreover since $G$ is of polynomial growth we have that $G/H$ is a doubling space \cite[Thm.~2.17]{Maheux} and also admits a Poincaré constant \cite[Thm.~2.18]{Maheux}. So we meet the additional requirements of \cite[Thm.~2.3 (vii) and (viii)]{balogh2012functional}.

The Hamiltonian $\mathcal{H}: \mathbb{R}_+ \to \mathbb{R}_+$ in \cite{balogh2012functional} is given by $\mathcal{H}(x)= x^{2\alpha}$. This Hamiltonian is indeed superlinear, convex, and satisfies $\mathcal{H}(0) = 0$. The corresponding Lagrangian $\mathcal{L}: \mathbb{R}_+ \to \mathbb{R}_+$ becomes
\[
\mathcal{L}(x) = 
\nu_{\alpha}\, 
x^{\frac{2\alpha}{2\alpha - 1}}.
\]
According to \cite{balogh2012functional} the solution (in the sense of their Theorem 2.1) to the $(-)$-variant of system (\ref{eq:erosion_system}) is given by
\[
\begin{split}
W_3(p, t) &= \inf_{x \in G / H} \left\{ t \mathcal{L}\left( \frac{\d_{\mathcal{G}_2}(p, x)}{t} \right)  + U_3 (x) \right\}
\\
&= \inf_{g \in G} \left\{ t \mathcal{L}\left( \frac{\d_{\mathcal{G}_2}(p, g p_0)}{t} \right)  + U_3 (g p_0) \right\}
\\
&= \inf_{g \in G} \left\{ t \mathcal{L}\left( \frac{\d_{\mathcal{G}_2}(g^{-1} p, p_0)}{t} \right)  + U_3 (g p_0) \right\}
\\
&= \inf_{g \in G} \left\{ \nu_\alpha \frac{\d_{\mathcal{G}_2}(g^{-1} p, p_0)^{\frac{2\alpha}{2\alpha-1}}}{t^{\frac{2\alpha}{2\alpha-1}-1}} + U_3(g p_0) \right\}
\\
&= \inf_{g \in G} \left\{ \nu_\alpha t^{1-\frac{2\alpha}{2\alpha-1}}
\d_{\mathcal{G}_2}(g^{-1} p, p_0)^{\frac{2\alpha}{2\alpha-1}}  + U_3(g p_0) \right\}
\\
&= \inf_{g \in G} \left\{ \nu_\alpha t^{\frac{-1}{2\alpha-1}}
\d_{\mathcal{G}_2}(g^{-1} p, p_0)^{\frac{2\alpha}{2\alpha-1}}  + U_3(g p_0) \right\}
\\
&= (k_t^{\alpha} \morph_G U_3 )(p)
.
\end{split}
\]
The $(+)$-variant is proven analogously. 

The semigroup property follows directly from \cite[Thm 2.1(ii)]{balogh2012functional}.
\qed
\end{proof}

\begin{remark}[Solution according to Balogh et al.]
    This theorem builds on the work by Balogh et al.\cite{balogh2012functional} who provide a solution concept that is (potentially) different from the strong, weak or viscosity solution.
    The point of departure is to replace the norm of the gradient (i.e. the dual norm of the differential) with a metric subgradient, i.e. we replace $\left\| \nabla_{\mathcal{G}_2} W (p,t) \right\|_{\mathcal{G}_2}$ by:
    \begin{equation*}
        \limsup_{p' \to p}{\frac{\max{\left( W(p,t) - W(p',t) ,\ 0\right)}}{d_{\mathcal{G}_2}(p,p')}},
    \end{equation*}
    and we get a solution concept in terms of this slightly different notion of a gradient.
\end{remark}

\begin{remark}[Unique viscosity solutions]
    For the case $\alpha=\sfrac{1}{2}$ we lose the superlinearity of the Hamiltonian and can no longer apply Balogh et al.'s approach \cite{balogh2012functional}. The solution for $\alpha>\sfrac{1}{2}$ \eqref{eq:dilation_kernel} converges pointwise to the solution for $\alpha=\sfrac{1}{2}$ \eqref{eq:dilation_kernel_1/2} as $\alpha \downarrow \sfrac{1}{2}$. However, the solution concept changes from that of Balogh et al. to that of a viscosity solution \cite{dragoni2007metric,azagra2005nonsmooth}. In the general Riemannian homogeneous space setting the result by Azagra \cite[Thm~6.24]{azagra2005nonsmooth} applies. It states that viscosity solutions of Eikonal PDEs on complete Riemannian manifolds are given by the distance map departing from the boundary of a given open and bounded set. As Eikonal equations  directly relate to geodesically equidistant wavefront propagation on manifolds (\cite[ch.~3]{rund1966hamilton},\cite[ch.~4,app.~E]{bekkers2015pde}, \cite{evans2010partial}) one expects that the solutions (\ref{eq:solve_dilation}),(\ref{eq:solve_erosion}) of (\ref{eq:erosion_system}) are indeed the viscosity solutions (for resp. the $+$ and $-$-case) for $\alpha=\sfrac{1}{2}$. 
    
    In many matrix Lie group quotients, like the Heisenberg group $H(2d+1)$ studied in \cite{manfredi2002version}, or in our case of interest: the homogeneous space $\mathbb{M}_d$ of positions and orientations) this is indeed the case. One can describe $G$-invariant vector fields via explicit coordinates and transfer HJB systems on $G/H$ directly towards HJB-systems on $\R^n$ or $\R^d \times S^q$, with $n=d+q=\textrm{dim}(G/H)$. Then one can directly apply results by Dragoni \cite[Thm.4]{dragoni2007metric} and deduce that \emph{our solutions, the dilations in (\ref{eq:solve_dilation}) resp. erosions in (\ref{eq:solve_erosion}), are indeed the unique viscosity solutions of HJB-PDE system (\ref{eq:erosion_system}) for the $+$ and $-$-case, for all $\alpha \in [\sfrac{1}{2},1]$}. Details are left for future research.
\end{remark}

To get an idea of how the kernel in \eqref{eq:dilation_kernel} operates in conjunction with morphological convolution we take $G=G/H=\mathbb{R}$ and see how the operation evolves simple data, the kernels and results at $t=1$ are shown in Fig. \ref{fig:soft_max_pooling}. Observe that with $\alpha$ close to $\sfrac{1}{2}$ (kernel and result in red) we obtain what amounts to an equivariant version of max/min pooling.
\begin{figure*}[ht]
    \centering
    \includegraphics[width=1.0\linewidth]{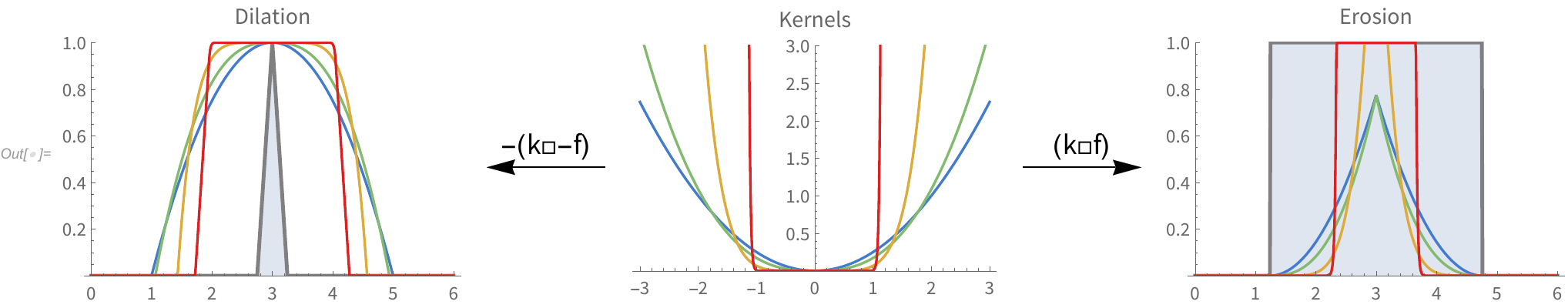}
    \includegraphics[width=0.4\linewidth]{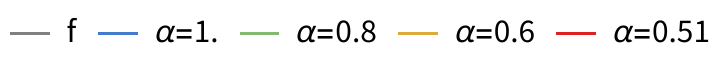}
	\caption{In the center we have kernels of the type \eqref{eq:dilation_kernel} in $\mathbb{R}$ (or the signed distance on a manifold of choice) for some $\alpha \in \left( \sfrac{1}{2}, 1\right]$ and $t=1$, which solves dilation/erosion. For $\alpha \to \sfrac{1}{2}$ this kernel converges to the type in \eqref{eq:dilation_kernel_1/2}, i.e. the solution is obtained by max/min pooling. On the left we morphologically convolve a spike (in gray) with a few of these kernels, we see that if $\alpha \to \sfrac{1}{2}$ we get max pooling, conversely we can call the case $\alpha>\sfrac{1}{2}$  \emph{soft} max pooling. On the right we similarly erode a plateau, which for $\alpha \to \sfrac{1}{2}$ yields min pooling. The effects of these operations in the image processing context can also be seen in the last two columns of Fig.~\ref{fig:interpretation_r2}.}
	\label{fig:soft_max_pooling}
\end{figure*}

The level sets of the kernels $k^{\alpha}_{t}$ for $\alpha > \sfrac{1}{2}$ are of the same shape as for the approximate diffusion kernels, see Fig.~\ref{fig:shapes}, for $\alpha=\sfrac{1}{2}$ these are the stencils over which we would perform min/max pooling. 
\begin{remark}
The level sets in Fig.~\ref{fig:shapes} are balls in $G/H=\mathbb{M}_2$ that do not depend on $\alpha$. It is only the growth of the kernel values when passing through these level sets that depends on $\alpha$. As such the example $G/H=\R$ and Fig.~\ref{fig:soft_max_pooling} is very representative to the general $G/H$ case. In the general $G/H$ case Fig.~\ref{fig:shapes} still applies when one replaces the horizontal $\R$-axis with a signed distance along a minimizing geodesic in $G/H$ passing through the origin. In that sense $\alpha \in [\sfrac{1}{2},1]$ regulates soft-max pooling over Riemannian balls in $G/H$.
\end{remark}

We can now define a more tangible approximate kernel by again replacing the exact metric $d_{\mathcal{G}_2}$ with the logarithmic approximation $\rho_{\mathcal{G}_2}$.
\begin{definition}[Approximate dilation/erosion kernel]
    \label{prop:approximate_morph_kernel}
    The approximate morphological convolution kernel $k^{\alpha,\mathrm{appr}}_{t}$ for small times $t$ and $\alpha\in\left( \sfrac{1}{2},1\right]$ is given by
    \begin{equation}
        \label{eq:approx_dilation_kernel}
        k^{\alpha,\mathrm{appr}}_{t} (p)
        :=
        \nu_{\alpha}
       t^{-\frac{1}{2 \alpha -1}}
        \rho_{\mathcal{G}_2}(p)^{\frac{2 \alpha}{2 \alpha -1}},
    \end{equation}
    with $\nu_{\alpha}:=\left( \frac{2 \alpha - 1}{ \left(2 \alpha\right)^{2\alpha/(2\alpha-1)}} \right)$
    and for $\alpha=\sfrac{1}{2}$ by
    \begin{equation}
        \label{eq:approx_dilation_kernel_1/2}
        k^{\sfrac{1}{2},\mathrm{appr}}_{t} (p)
        =
        \begin{cases}
            0 \qquad &\text{if } \rho_{\mathcal{G}_2}(p) \leq t,
            \\
            \infty &\text{if } \rho_{\mathcal{G}_2}(p) > t
        \end{cases}.
    \end{equation}
\end{definition}
We used this approximation in our parallel GPU-algorithms (for our PDE-G-CNNs experiments in Section~\ref{sec:experiments}). It is highly preferable over the `exact' solution based on the true distance as this would require Eikonal PDE solvers (\cite{bekkers2015pde,mirebeau2019hamiltonian} which would not be practical for parallel GPU implementations of PDE-G-CNNs. Again the approximations are reasonable as long as the spatial anisotropy does not get too high, see Fig.~\ref{fig:distance_comparison} for an example.

Next we formalize the  theoretical underpinning of the approximations in the upcoming corollary.

An immediate consequence of Def.~\ref{prop:approximate_morph_kernel} and Lem.~\ref{lem:bounding-metric} (keeping in mind that the kernel expressions in Def.~\ref{prop:approximate_morph_kernel} are monotonic w.r.t. $\rho:=\rho_{\mathcal{G}_{2}}(p)$) is that we can enclose our approximate morphological kernel with the exact morphological kernels in the same way as we did for the (fractional) diffusion kernel in Theorem~\ref{thm:alpha_kernel_squeeze}. This proves the following Corollaries:

\begin{corollary}
    Let $\alpha \in \left( \sfrac{1}{2}, 1 \right]$, then for all $t>0$
    \begin{equation*}
        k_t^\alpha(p)
        \leq 
        k_t^{\alpha,\mathrm{appr}}(p) 
        \leq
        C_{\mathrm{metr}}^{\frac{2\alpha}{2\alpha-1}} k_t^{\alpha}(p)
        \qquad\hfill
        \mathrm{for } \ p \in G/H.
    \end{equation*}
    For the case $\alpha=\sfrac{1}{2}$ the approximation is exact in an inner and outer region:
    \begin{align*}
        &k_t^{\sfrac{1}{2},\mathrm{appr}}(p) = k_t^{\sfrac{1}{2}}(p) = 0 
        \qquad &\mathrm{if }\ \rho_{\mathcal{G}_2}(p)^{2 \alpha} \leq t,
        \\
        &k_t^{\sfrac{1}{2},\mathrm{appr}}(p) = k_t^{\sfrac{1}{2}}(p) = \infty
        &\mathrm{if }\ \d_{\mathcal{G}_2}(p_0,p)^{2\alpha} > t ,
    \end{align*} 
    but in the intermediate region where $\rho_{\mathcal{G}_2}(p)^{2 \alpha} > t$ and $\d_{\mathcal{G}_2}(p_0,p)^{2\alpha} \leq t$ we have $k_t^{\sfrac{1}{2},\mathrm{appr}}=\infty$ while $k_t^{\sfrac{1}{2}}=0$.
\end{corollary}

Alternatively, instead of bounding by value we can bound in time, in which case we do not need to distinguish different cases of $\alpha$.
\begin{corollary}
Let $\alpha \in \left[ \sfrac{1}{2}, 1 \right], t>0$ then for all $p \in G/H$ one has
    \begin{equation*}
        k_t^\alpha(p)
        \leq 
        k_t^{\alpha,\mathrm{appr}}(p) 
        \leq
        k_{C_{\mathrm{metr}}^{-2\alpha}t}^{\alpha}(p)
        \end{equation*}
\end{corollary}

With these two bounds on our approximate morphological kernels we end our theoretical results.

\section{Generalization of (Group-)CNNs}
\label{sec:G-CNN-generalization}
    
In this section we point out the similarities between common (G-)CNN operations and our PDE-based approach.
Our goal here is not so much claiming that our PDE approach serves as a useful model for analyzing (G-)CNNs, but that modern CNNs already bear some resemblance to a network of PDE solvers.
Noticing that similarity, our approach is then just taking the next logical step by structuring a network to explicitly solve a set of PDEs.

\subsection{Discrete Convolution as Convection \& Diffusion}

Now that we have seen how PDE-G-CNNs are designed we show how they generalize conventional G-CNNs. Starting with an initial condition $U$ we show how group convolution with a general kernel $k$ can be interpreted as a superposition of solutions \eqref{eq:convection_solution} of convection PDEs:
\begin{align*}
    \left( k *_{G/H} U \right)(p)
    &=
    \frac{1}{\mu_G(H)}
    \int_G
    k \left( g^{-1} p \right)
    U \left( g p_0 \right) \d\mu_G(g)
    \\
    =&
    \frac{1}{\mu_G(H)}
    \int_G
    k \left( g^{-1} g_p p_0 \right)
    U \left( g p_0 \right) \d\mu_G(g)
    ,
\intertext{for any $g_p \in p$, now change variables to $q=g_p^{-1} g$ and recall that $\mu_G$ is left invariant:}
    =&
    \frac{1}{\mu_G(H)}
    \int_G
    k \left( q^{-1} p_0 \right)
    U \left( g_p q p_0 \right)
    \d\mu_G(q)
    .
\end{align*}
In this last expression we recognize \eqref{eq:convection_solution} and see that we can interpret $p \mapsto U \left( g_p \,q  p_0 \right)$ as the solution of the convection PDE \eqref{eq:convection_pde} at time $t=1$ for a convection vector field $\bm{c}$ that has flow lines given by $\gamma_{\bm{c}}(t)=\exp_G \left( - t \log_G q \right) p_0$ so that $(\gamma_{\mathbf{c}}(1))^{-1} p_0= q p_0$.
As a result the output $k *_{G/H} U$ can then be seen as a weighted sum of solutions over all possible left invariant convection vector fields.

Using this result we can consider what happens in the discrete case where we take the kernel $k$ to be a linear combination of displaced diffusion kernels $K_t^\alpha$ (for some choice of $\alpha$) as follows:
\begin{equation}
    k(p) = \sum_{i=1}^n k_i \, K^\alpha_{t_i} \left( g_i^{-1} p \right),
\end{equation}
where for all $i$ we fix a weight $k_i \in \mathbb{R}$, diffusion time $t_i \geq 0$ and a displacement $g_i \in G$. Convolving with this kernel yields:
\begin{align*}
    &\left( k *_{G/H} U \right)(p)
    \\
    &=
    \int_G
    \sum_{i=1}^n k_i \, K^\alpha_{t_i} \left( g_i^{-1} \, g^{-1} p \right)
    U \left( g p_0 \right)
    \d\mu_G(g)
    \\
    &=
    \sum_{i=1}^n
    k_i
    \int \limits_G
    K^\alpha_{t_i} \left( g_i^{-1} \, g^{-1} p \right)
    U \left( g p_0 \right)
    \d\mu_G(g),
\intertext{we change variables to $h=g\, g_i$:}
    &=
    \sum_{i=1}^n
    k_i
    \int \limits_G
    K^\alpha_{t_i} \left( h^{-1} p \right)
    U \left( h \, g_i^{-1} p_0 \right)
    \d\mu_G(h)
    \\
    &=
    \sum_{i=1}^n
    k_i
    \left(
        K^\alpha_{t_i}
        *_{G/H}
        \left[
            \vphantom{\Big[}
            q \mapsto U \left( g_q \, g_i^{-1} p_0 \right)
        \right]
    \right)(p)
    .
\end{align*}
Here again we recognize $q \mapsto U \left( g_q \, g_i^{-1} p_0 \right)$ as the solution \eqref{eq:convection_solution} of the convection PDE at $t=1$ with flow lines induced by $\gamma_{\bm{c}}(t)=\exp_G(t \log_G g_i)$. Subsequently we take these solutions and convolve them with a (fractional) diffusion kernel with scale $t_i$, i.e. after convection we apply the fractional diffusion PDE with evolution time $t_i$ and finally make a linear combination of the results.

We can conclude that G-CNNs fit in our PDE-based model by looking at a single discretized group convolution as a set of single-step PDE units working on an input, without the morphological convolution and with specific choices made for the convection vector fields and diffusion times.

\subsection{Max Pooling as Morphological Convolution}

The ordinary max pooling operation commonly found in convolutional neural networks can also be seen as a morphological convolution with a kernel for $\alpha=\sfrac{1}{2}$.

\begin{proposition}[Max pooling]
    \label{thm:max_pooling}
    Let $f \in L^\infty\left( G/H \right)$, let $S \subset G/H$ be non empty and define $k_{S} : G/H \to \mathbb{R} \cup \left\{ \infty \right\}$ as:
    \begin{equation}
        \label{eq:max_pooling_kernel}
        k_{S} (p)
        :=
        \begin{cases}
            0 \qquad &\text{if } p \in S,
            \\
            \infty \qquad&\text{else}.
        \end{cases}
    \end{equation}
    Then:
    \begin{equation}
        -\left( k_S \morph -f \right)(p)
        =
        \sup_{g \in G : g^{-1}  p \in S}
        f \left( g  p_0 \right)
        .
    \end{equation}
\end{proposition}
\noindent
We can recognize the morphological convolution as a generalized form of max pooling of the function $f$ with stencil~$S$.

\begin{proof}
    Filling in \eqref{eq:max_pooling_kernel} into Def. \ref{def:morphological_convolution} yields:
    \begin{align*}
        &-\left( k_S \morph -f \right) (p)
        \\
        &=
        -\inf \left\{
            \inf_{g \in G: g^{-1}  p \in S} -f \left( g  p_0 \right)
            ,
            \inf_{g \in G: g^{-1}  p \notin S} -f \left( g  p_0 \right) + \infty
        \right\}
        \\[1em]
        &=
        -\inf_{g \in G: g^{-1}  p \in S} -f \left( g  p_0 \right)
        \\[1em]
        &=
        \sup_{g \in G: g^{-1}  p \in S} f \left( g  p_0 \right)
        .
    \end{align*}
    \qed
\end{proof}

In particular cases we recover a more familiar form of max pooling as the following corollary shows.
\begin{corollary}[Euclidean Max Pooling]
    Let $G=G/H=\mathbb{R}^n$ and let $f \in C^0 \left( \mathbb{R}^n \right)$ with $S \subset \mathbb{R}^n$ compact then:
    \begin{equation*}
        -\left( k_S \morph_{\mathbb{R}^n} -f \right)(x)
        =
        \max_{y \in S} f \left( x-y \right)
        ,
    \end{equation*}
    for all $x \in \mathbb{R}^n$.
\end{corollary}

The observation that max pooling is a particular limiting case of morphological convolution allows us to think of the case with $\alpha > \sfrac{1}{2}$ as a \emph{soft} variant of max pooling, one that is better behaved under small perturbations in a discretized context.  

\subsection{ReLUs as Morphological Convolution}

Max pooling is not the only common CNN operation that can be generalized by morphological convolution as the following proposition shows.

\begin{proposition}
    Let $f$ be a compactly supported continuous function on $G/H$. Then dilation with the kernel
    \begin{equation*}
        k_{\mathrm{ReLU},f}(p) :=
        \begin{cases}
            \ 0 \qquad & \text{if } p = p_0,
            \\
            \ \sup\limits_{q \in G/H} f(q) & \text{else},
        \end{cases}
    \end{equation*}
    equates to applying a Rectified Linear Unit to the function $f$:
    \begin{equation*}
        -\left( k_{\text{ReLU},f} \morph -f \right)(p)
        =
        \max \left\{ 0, f(p) \vphantom{\big\vert}\right\}
        .
    \end{equation*}
\end{proposition}
\begin{proof}
    Filling in $k$ into the definition of morphological convolution:
    \begin{align*}
        &-\left( k_{\text{ReLU},f} \morph_G -f \right)(p)
        \\
        &=
        - \inf_{g \in G} k_{\text{ReLU}}(g^{-1} p) - f(g.p_0)
        \\
        &=
        - \inf_{g \in G} \left\{ \inf_{g^{-1} p = p_0} -f(g p_0) ,  \inf_{g^{-1} p \neq p_0} - f(g p_0) +    \sup_{y \in G/H} f(y)  \right\}
        \\
        &=
        \sup \left\{ f(p) ,\  \sup_{z \in G/H:z \neq p} f(z) -  \sup_{y \in G/H} f(y)  \right\},
        \intertext{due to the continuity and compact support of $f$ its supremum exists and moreover we have $\sup_{z \in G/H:z \neq p_0} f(z) =  \sup_{y \in G/H} f(y)$ and thereby we obtain the required result}
        &=
        \max \left\{f(p) \vphantom{\big\vert} ,\ 0\right\}
        .
    \end{align*}
    \qed
\end{proof}

We conclude that morphological convolution allows us to:
\begin{itemize}
    \item do pooling in an equivariant manner with transformations other then translation,
    \item do \emph{soft}  pooling that is continuous under domain transformations (illustrated in Fig. \ref{fig:soft_max_pooling}),
    \item learn the pooling region by considering the kernel $k$ as trainable,
    \item effectively fold the action of a ReLU into trainable non-linearities.
\end{itemize}

\subsection{Residual Networks}

So called \emph{residual networks} \cite{he2016deep} were introduced mainly as a means of dealing with the vanishing gradient problem in very deep networks, aiding trainability. 
These networks use so-called residual blocks, illustrated in Fig.~\ref{fig:residual}, that feature a skip connection to group a few layers together to produce a delta-map that gets added to the input.

\begin{figure}[ht!]
    \centering
    \includegraphics[width=0.4\linewidth]{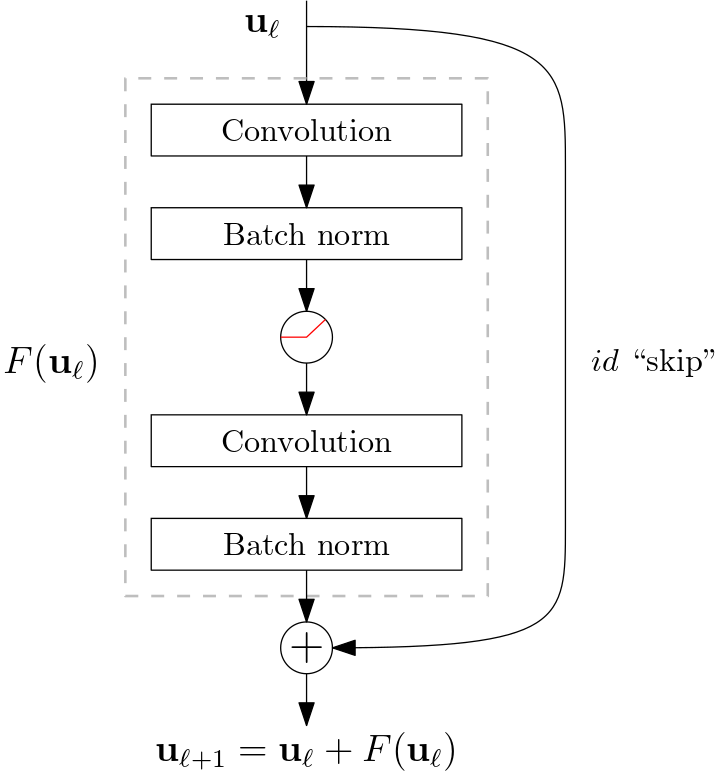}
    \caption{A residual block, like in \cite{he2016deep}, note the resemblance to a forward Euler discretization scheme.}
    \label{fig:residual}
\end{figure}

This \emph{identity + delta} structure is very reminiscent of a forward Euler discretization scheme. If we had an evolution equation of the type
\begin{equation*}
    \begin{cases}
        \frac{\partial U}{\partial t}(p,t) 
        = \mathcal{F}\left( U(\cdot,t),  p \right)
        \qquad & \text{ for } p \in M, t\geq 0,
        \\
        U(p,0) = U_0(p) & \text{ for } p \in M, 
    \end{cases}
\end{equation*}
with some operator $\mathcal{F}: L^{\infty}(M) \times M \to \mathbb{R}$, we could solve it approximately by stepping forward with:
\begin{equation*}
    U(p,t+\Delta t) = U(p,t) + \Delta t \ \mathcal{F}\left( U(\cdot, t),  p \right),
\end{equation*}
for some time step $\Delta t > 0$.
We see that this is very similar to what is implemented in the residual block in Fig.~\ref{fig:residual} once we discretize it.

The correspondence is far from exact given that multiple channels are being combined in residual blocks, so we can not easily describe a residual block with a PDE. 
Still, our takeaway is that residual networks and skip connections have moved CNNs from networks that \emph{change} data to networks that \emph{evolve} data.

For this reason we speculate that deep PDE-G-CNNs will not need (or have the same benefit from) skip connections, we leave this subject for future investigation. More discussion on the relation between residual networks and PDEs can be found in \cite{alt2021translating}.

\section{Experiments}
\label{sec:experiments}

To demonstrate the viability of PDE-based CNNs we perform two experiments where we compare the performance of PDE-G-CNNs against G-CNNs and classic CNNs. We will be doing a vessel segmentation and digit classification problem: two straightforward applications of CNNs. Examples of these two applications are illustrated in Fig.~\ref{fig:experiments}. 

The goal of the experiments is to compare the basic building blocks of the different types of networks in clearly defined feed-forward network architectures.
So we test networks of modest size only and do not just aim for the performance that would be possible with large-scale networks.

\begin{figure}[ht]
    \centering
    \begin{subfigure}[t]{0.45\linewidth}
        \centering
        \includegraphics[width=0.9\linewidth]{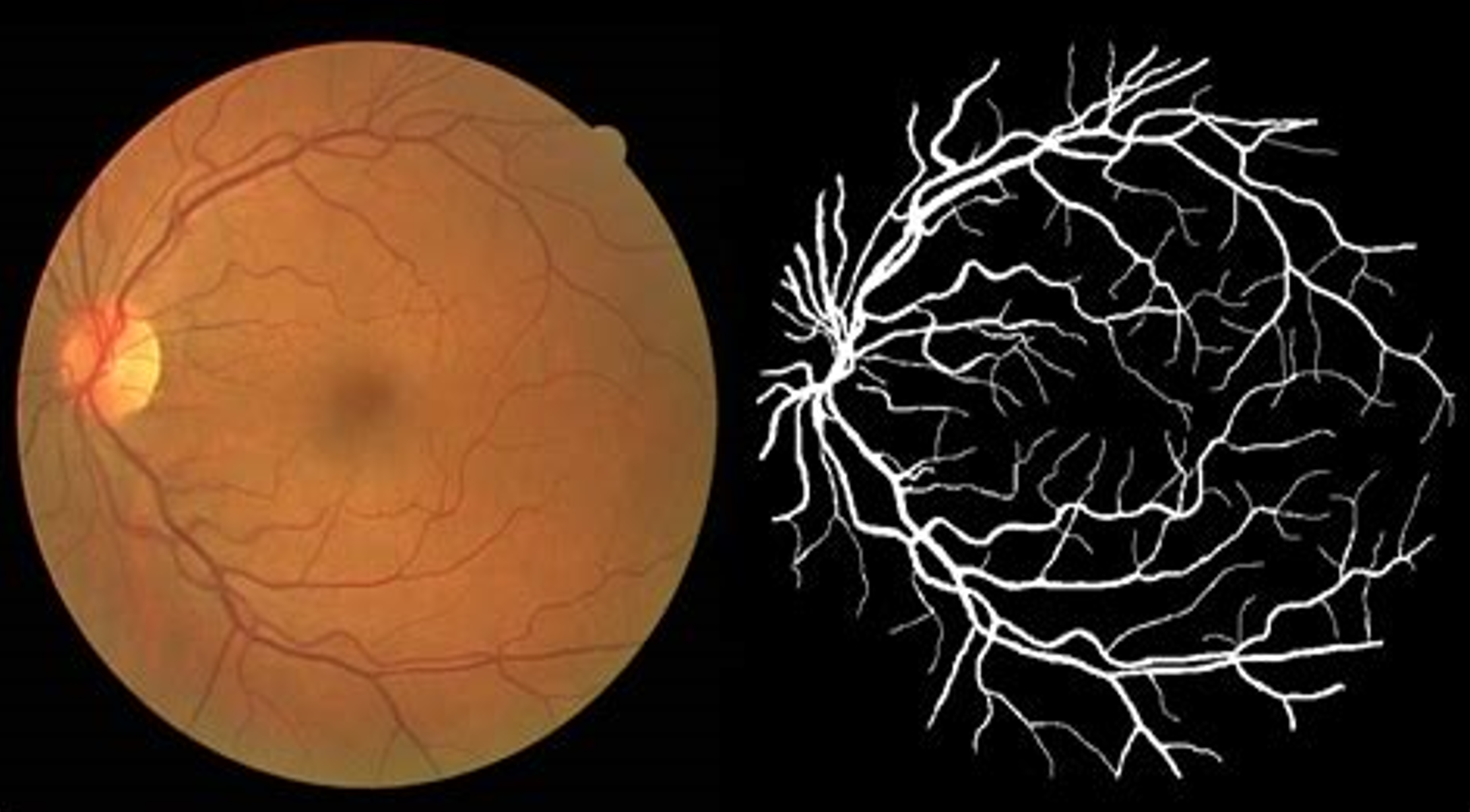}
        \caption{Example from the DRIVE \cite{staal2004ridge} dataset, showing a retinal image and its vessel segmentation.}
        \vspace*{1em}
    \end{subfigure}
    \begin{subfigure}[t]{0.45\linewidth}
        \centering
        \includegraphics[width=0.9\linewidth]{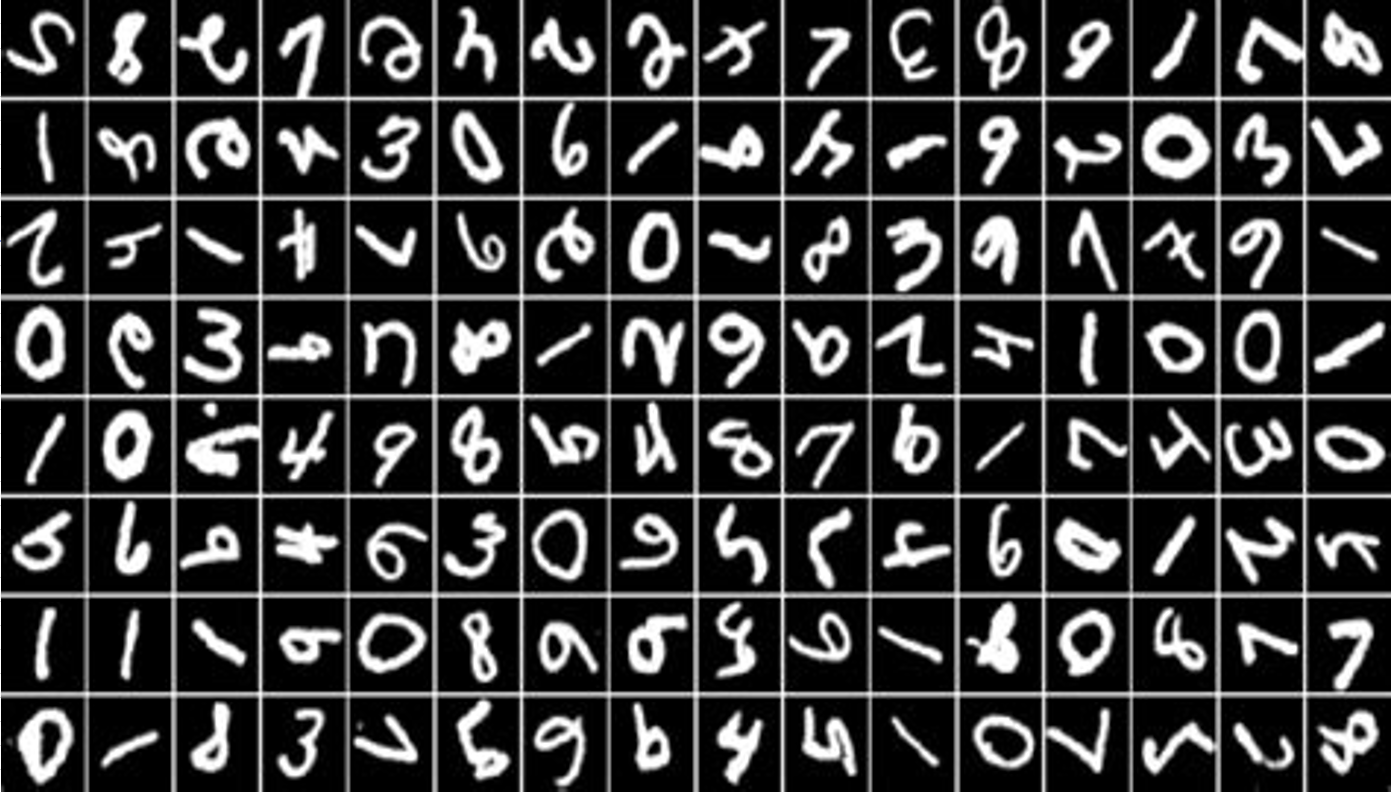}
        \caption{Examples from the RotNIST \cite{Baweja2018} dataset.}
    \end{subfigure}
    \caption{We perform a segmentation experiment on retinal vessel images and a classification experiment on rotation augmented digits.}
    \label{fig:experiments}
\end{figure}

\subsection{Implementation}

We implemented our PDE-based operators in an extension to the PyTorch deep learning framework \cite{paszke2019pytorch}. Our package is called \emph{LieTorch} and is open source. It is available at \url{https://gitlab.com/bsmetsjr/lietorch}.

The operations we have proposed in the paper have been implemented in C++ for CPUs and CUDA for Nvidia GPUs but can be used from Python through PyTorch. Our package was also designed with modularity in mind: we provide a host of PyTorch modules that can be used together to implement the PDE-G-CNNs we proposed but that can also be used separately to experiment with other architectures.

All the modules we provide are differentiable and so our PDE-G-CNNs are trainable through stochastic gradient descent (or its many variants) in the usual manner. In our experiments we have had good results with using the ADAM \cite{kingma2014adam} optimizer.

All the network models and training scripts used in the experiments are also available in the repository.

\subsection{Design Choices}

Several design choices are common to both experiments, we will go over these now.

First, we choose $G/H=\mathbb{M}_2$ for our G-CNNs and PDE-G-CNNs and so go for roto-translation equivariant networks. In all instances we lift to 8 orientations.

Second, we use the convection, dilation and erosion version of \eqref{eq:cnn_pde}, hence we refer to these networks as PDE-CNNs of the \emph{CDE}-type. 
Each PDE-layer is implemented as in Fig.~\ref{fig:traditional} with the single-pass PDE solver from Fig.~\ref{fig:cdde-layer} without the convolution.
So no explicit diffusion is used and the layer consists of just resampling and two morphological convolutions. 
Since we do the resampling using trilinear interpolation this does introduce a small amount of implicit diffusion.

\begin{remark}[Role of diffusion]
In these experiments we found no benefit to adding diffusion to the networks. 
Diffusion likely would be of benefit when the input data is noisy but neither datasets we used are noisy and we have not yet performed experiments with adding noise. We leave this investigation for future work.
\end{remark}

Third, we fix $\alpha=0.65$. 
We came to this value empirically; the networks performed best with $\alpha$-values in the range $0.6-0.7$.
Looking at Fig.~\ref{fig:soft_max_pooling} we can conjecture that $\alpha=0.65$ is the ``sweet spot'' between sharpness and smoothness. When the kernel is too sharp ($\alpha$ close to $\sfrac{1}{2}$) minor perturbations in the input can have large effects on the output, when the kernel is too smooth ($\alpha$ close to $1$) the output will be smoothed out too much as well.

Fourth, all our networks are simple feed-forward networks.

Finally, we use the ADAM optimizer \cite{kingma2014adam} together with $L^2$ regularization uniformly over all parameters with a factor of 0.005.

\subsection{DRIVE Retinal Vessel Segmentation}

The first experiment uses the DRIVE retinal vessel segmentation dataset \cite{staal2004ridge}. The object is the generate a binary mask indicating the location of blood vessels from a color image of a retina as illustrated in Fig.~\ref{fig:experiments}(a).

We test 6- and 12-layer variants of a CNN, a G-CNN and a CDE-PDE-CNN. The layout of the 6-layer networks is shown in Fig.~\ref{fig:drive_models}, the 12-layer networks simply add more convolution, group convolution or CDE layers. All the networks were trained on the same training data and tested on the same testing data.

The output of the network is passed through a sigmoid function to produce a 2D map $a$ of values in the range $[0,1]$ which we compare against the known segmentation map $b$ with values in $\left\{ 0,1 \right\}$. We use the continuous DICE coefficient as the loss function:
\begin{equation*}
    \mathrm{loss}(a,b) = 1 - \frac{2 \sum a  b + \varepsilon}{\sum a + \sum b + \varepsilon},
\end{equation*}
where the sum $\sum$ is over all the values in the 2D map. A relatively small $\varepsilon=1$ is used to avoid divide-by-zero issues and the $a \equiv b \equiv 0$ edge case.

The 6-layer networks were trained over 60 epochs, starting with a learning rate of 0.01 that we decay exponentially with a gamma of 0.95. The 12-layer networks were trained over 80 epochs, starting from the same learning rate but with a learning rate gamma of 0.96.

We measure the performance of the network by the DICE coefficient obtained on the 20 images of the testing dataset. We trained each model 10 times, the results of which are summarized in Tbl.~\ref{tbl:drive} and Fig.~\ref{fig:experiment_performance}(a). 

We achieve similar or better performance than CNNs or G-CNNs but with a vast reduction in parameters. Scaling from 6 to 12 layers even allows us to reduce the total number of parameters of the PDE-G-CNN while still increasing performance, this is achieved by reducing the number of channels (i.e. the width) of the network, see also Tbl.~\ref{tbl:params}.

\begin{figure}[ht]
    \centering
    \includegraphics[width=0.7\linewidth]{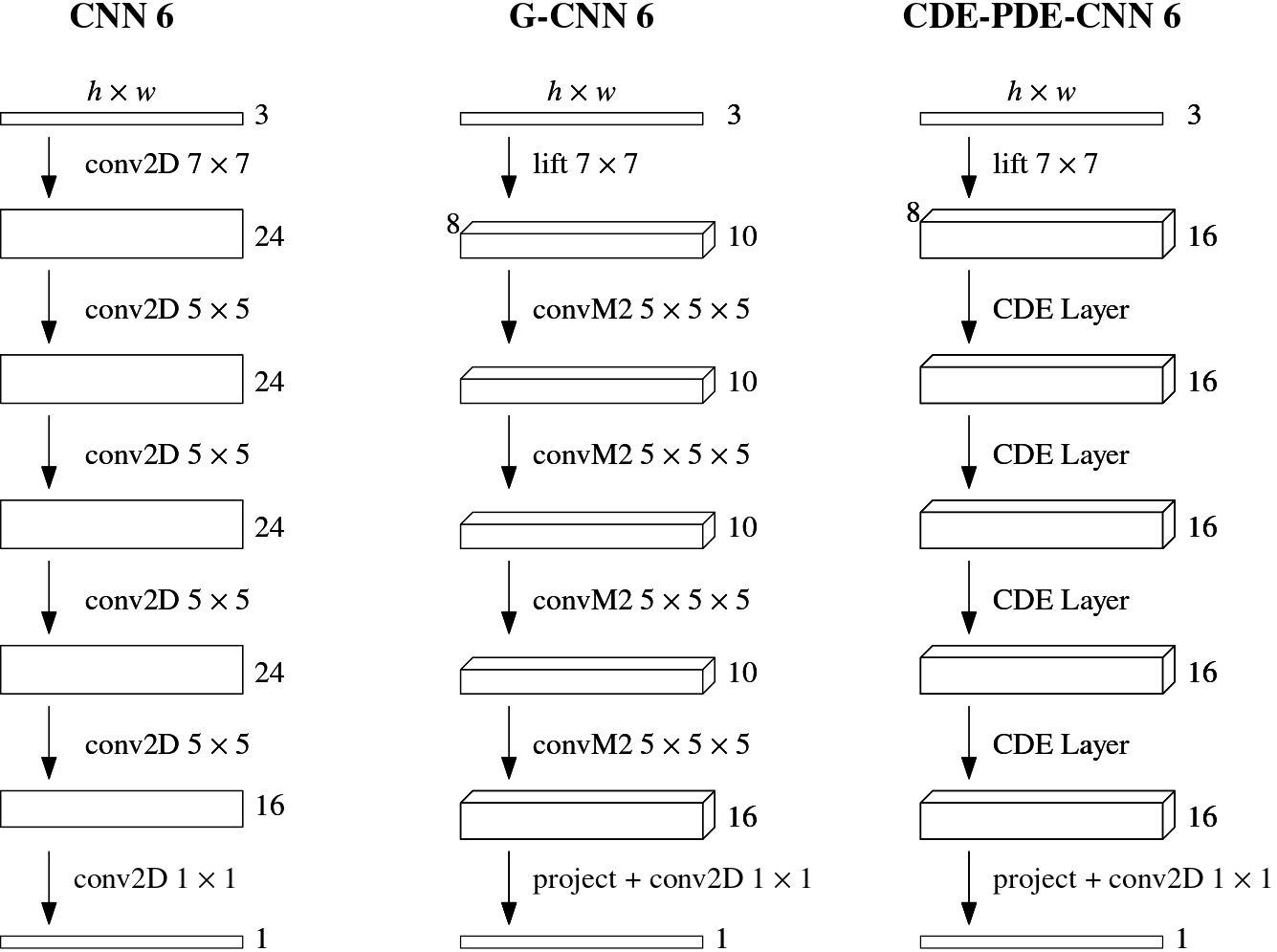}
    \caption{Schematic of the 6-layer models used on our segmentation experiments. Kernel sizes and number of feature channels in each layer are indicated, depth indicates that the data lives on $\mathbb{M}_2$. Omitted are activation functions, batch normalization, padding and dropout modules. The 12-layer models are essentially the same but with double the number of layers but with reduced number of channels per layer (i.e. reduced width) for the CDE-PDE-CNN (hence the reduction in parameters going from 6 to 12 layers).}
    \label{fig:drive_models}
\end{figure}

\begin{table}[ht]
    \centering
    \renewcommand{\arraystretch}{1.3}
    \begin{tabular}{@{}l r r@{}}
        Model &  Parameters &  DICE score $\pm$ std.dev.
        \\
        \hline
        CNN 6 & $47352$  & $0.8058 \pm 0.0017$
        \\
        G-CNN 6 & $39258$  & $0.8085 \pm 0.0022$
        \\
        CDE-PDE-CNN 6 & $4128$  & $0.8115 \pm 0.0018$
        \\[0.7em]
        CNN 12 & $129432$ & $0.8189 \pm 0.0005$
        \\
        G-CNN 12 & $114378$ & $0.8192\pm 0.0012$
        \\
        CDE-PDE-CNN 12 & $3678$ & $0.8220 \pm 0.0007$
    \end{tabular}
    \caption{Average DICE coefficient achieved on the 20 images of the testing dataset and the number of trainable parameters of each model. The G-CNNs and CDE-PDE-CNNs are roto-translation equivariant by construction. Note the vast reduction in parameters allowed by using PDE-based networks.}
    \label{tbl:drive}
\end{table}

\begin{table}[ht]
    \centering
    \renewcommand{\arraystretch}{1.3}
    \begin{tabular}{@{}l r r@{}}
        Type of parameter & CDE-PDE-CNN 6 & CDE-PDE-CNN 12
        \\
        \hline
        Lifting layer & $2352$  & $1470$
        \\
        Convection & $192$  & $300$
        \\
        Dilation & $192$  & $300$
        \\
        Erosion & $192$ & $300$
        \\
        Linear combinations & $1040$ & $1076$
        \\
        Batch normalization & $160$ & $232$
    \end{tabular}
    \caption{Allocation of parameters for the 6- and 12-layer CDE-PDE-CNNs used in the vessel segmentation experiment. The added depth of the networks allows us to shrink the width. With the network having less channels over all we can also shrink the number of channels in the lifting layer, which drastically reduces the total number of parameters.}
    \label{tbl:params}
\end{table}

\begin{figure}[ht]
    \centering
    \begin{subfigure}[t]{0.45\linewidth}
        \centering
        \includegraphics[width=0.9\linewidth]{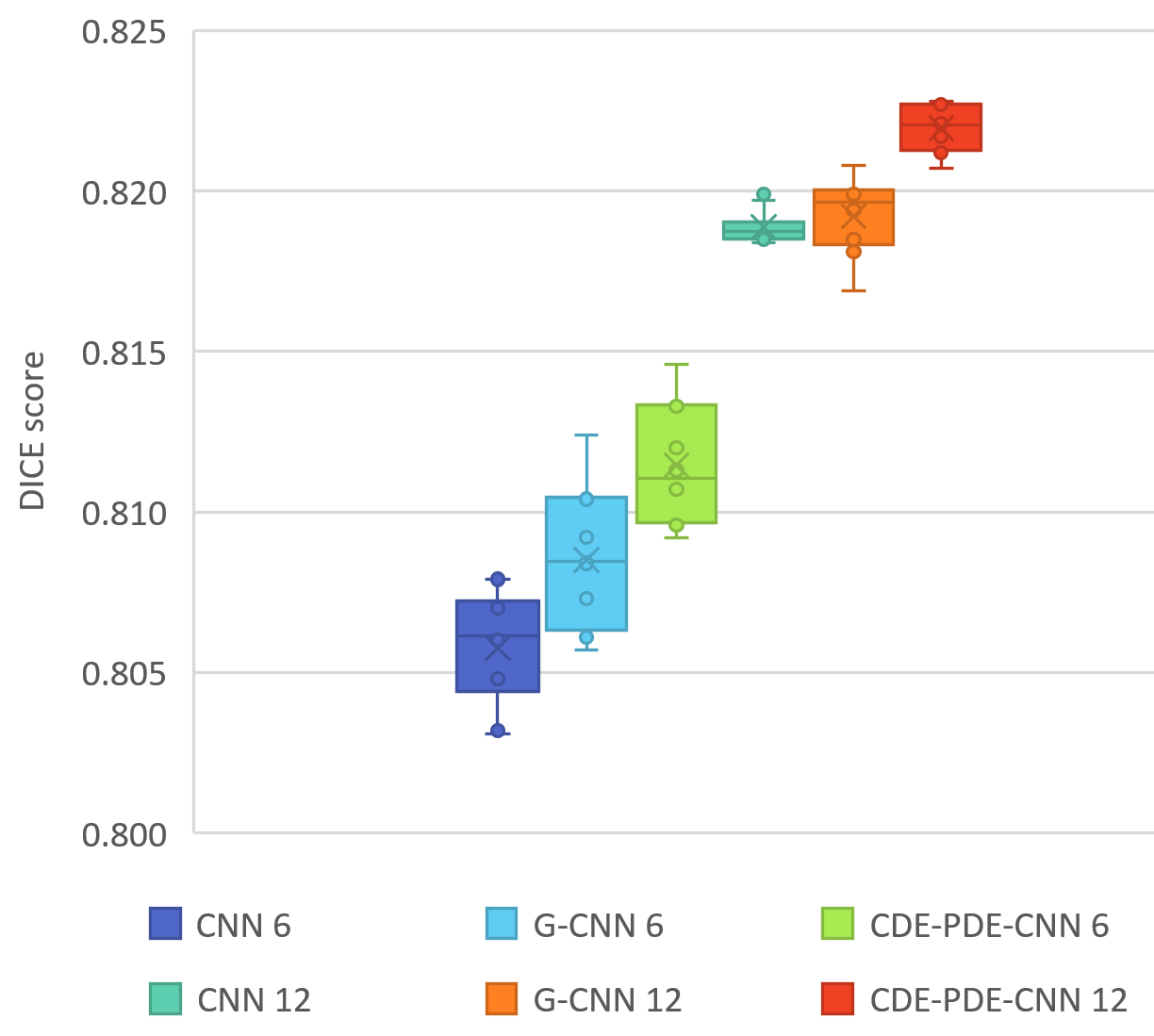}
        \caption{Performance on retinal vessel segmentation. We test 6- and 12-layer variants of conventional CNNs, G-CNNs and our PDE-CNNs, each network is trained 10 times, the chart shows the distribution of DICE performances on the test dataset.}
        \vspace*{1em}
    \end{subfigure}
    ~\hspace{5mm}~
    \begin{subfigure}[t]{0.45\linewidth}
        \centering
        \includegraphics[width=0.9\linewidth]{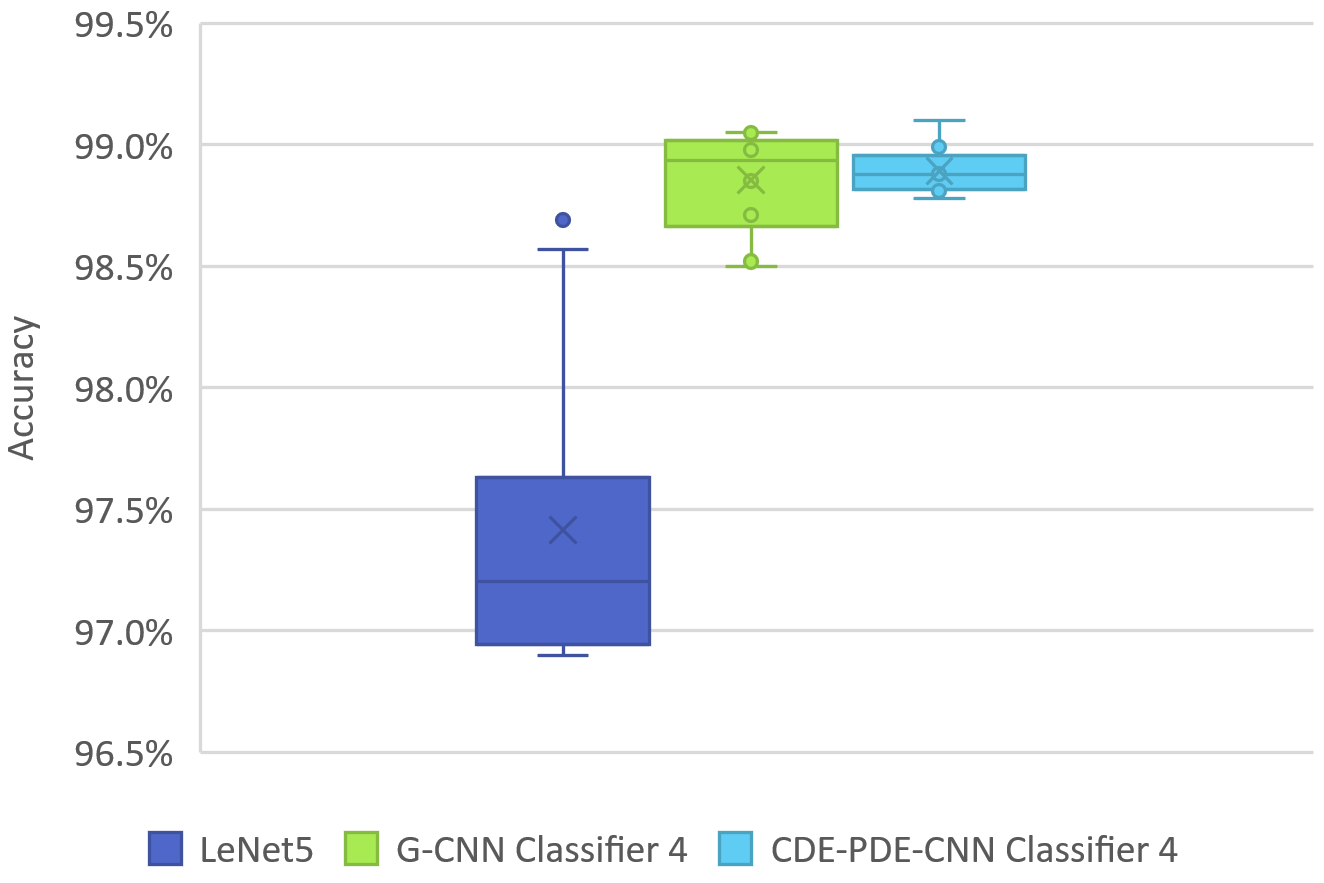}
        \caption{Performance of digit classification on the RotNIST dataset. We compare the classic 5-layer LeNet against a 4-layer G-CNN and PDE-CNN. LeNet was trained for 120 epochs, the other two for 60 epochs.}
    \end{subfigure}
    \caption{Comparison of PDE-based networks against conventional CNNs and group CNNs on segmentation and classification tasks. }
    \label{fig:experiment_performance}
\end{figure}

\subsection{RotNIST Digit Classification}

The second experiment we performed is the classic digit classification experiment. Instead of using the plain MNIST dataset we did the experiment on the RotNIST dataset \cite{Baweja2018}. RotNIST contains the same images as MNIST but rotated to various degrees. 
Even though classifying rotated digits is a fairly artificial problem we include this experiment to show that PDE-G-CNNs also work in a context very different from the first segmentation experiment.
While our choice of PDEs derives from more traditional image processing methods, this experiment shows their utility in a basic image classification context.

We tested three networks: the classic LeNet5 CNN \cite{lecun1989backpropagation} as a baseline, a 4-layer G-CNN and a 4-layer CDE-PDE-CNN. The architectures of these three networks are illustrated in Fig.~\ref{fig:rotnist_models}.

All three networks were trained on the same training data and tested on the same testing data. 
We train with a learning rate of 0.05 and a learning rate gamma of 0.96. 
We trained the LeNet5 model for 120 epochs and the G-CNN and CDE-PDE-CNN models for 60 epochs.

We measure the performance of the network by its accuracy on the testing dataset. 
We trained each model 10 times, the results of which are summarized in Tbl.~\ref{tbl:rotnist} and Fig.~\ref{fig:experiment_performance}(b). 

We manage to get better performance than classic or group CNNs with far fewer parameters.

\begin{figure}[ht]
    \centering
    \includegraphics[width=0.7\linewidth]{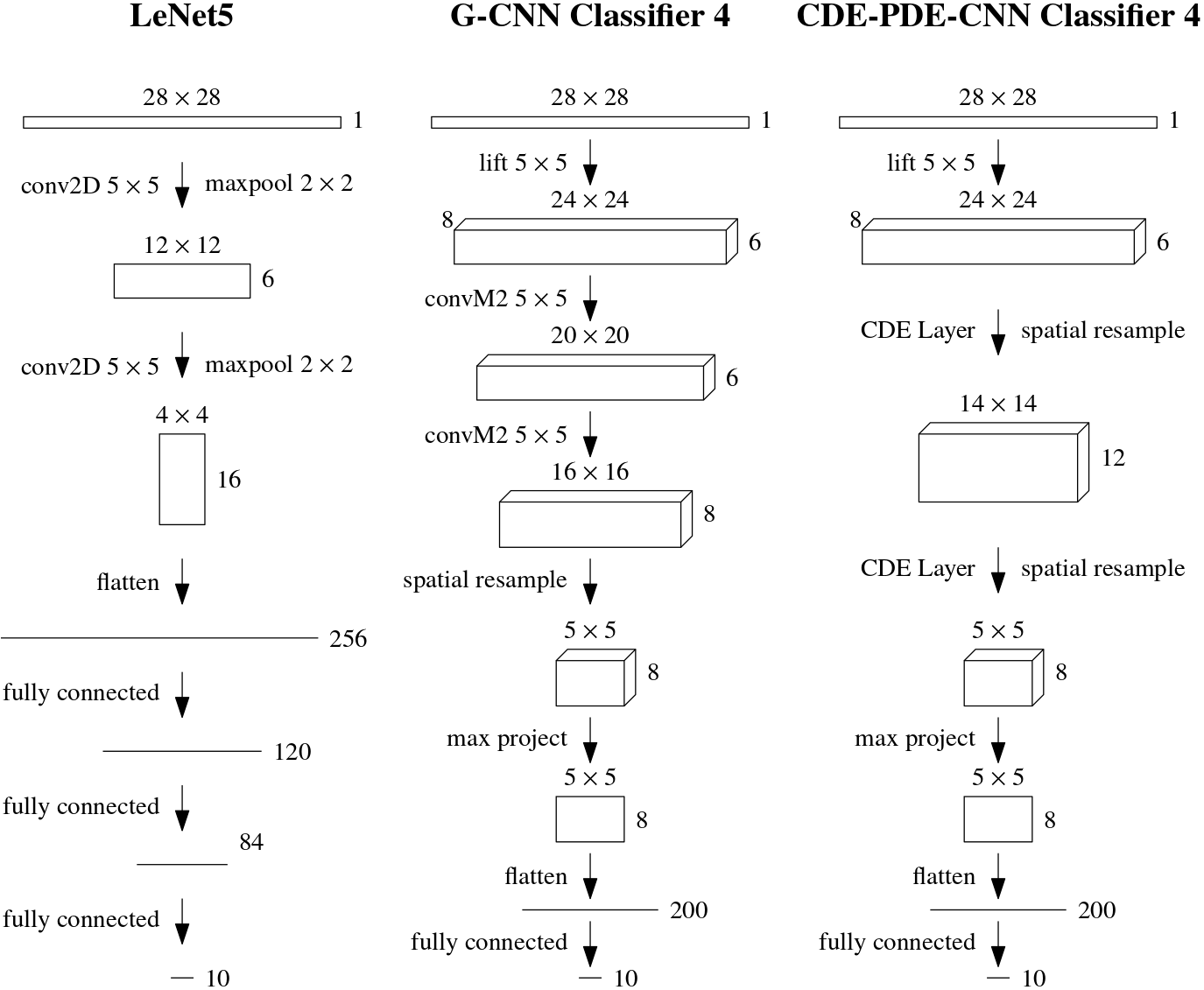}
    \caption{Schematic of the three models tested with the RotNIST data. Kernel sizes and number of feature channels in each layer are indicated. Omitted are activation functions, batch normalization and dropout modules.}
    \label{fig:rotnist_models}
\end{figure}

\begin{table}[ht]
    \centering
    \renewcommand{\arraystretch}{1.3}
    \begin{tabular}{@{}l r r@{}}
        Model &  Parameters &  Error rate $\pm$ std.dev.
        \\
        \hline
        CNN (LeNet5) & $44426$  & $2.59\% \pm 0.66\%$
        \\
        G-CNN Classifier 4 & $12700$  & $1.14\% \pm 0.21\%$
        \\
        CDE-PDE-CNN Classifier 4 & $2542$  & $1.10\% \pm 0.10\%$
    \end{tabular}
    \caption{Accuracy of the digit classification models on the testing dataset and number of parameters for each model.}
    \label{tbl:rotnist}
\end{table}

\subsection{Computational Performance}

Care was taken in optimizing the implementation to show that PDE-based networks can still achieve decent running times despite their higher computational complexity. In Tbl.~\ref{tab:speed} we summarized the inferencing performance of each model we experimented with.

Our approach simultaneously gives us equivariance, a decrease in parameters and higher performance but at the cost of an increase in flops and memory footprint.
While our implementation is reasonably optimized it has had far less development time dedicated to it than the traditional CNN implementation provided by PyTorch/cuDNN, so we are confident more performance gains can be found.

In comparison with G-CNNs our PDE-based networks are generally a little bit faster. Our G-CNN implementation is however less optimized compared to out PDE-G-CNN implementation. Were our G-CNN implementation equally optimized we expect G-CNNs to be slightly faster than the PDE-G-CNNs in our experiments.

\begin{table}[ht!]
    \centering
    \renewcommand{\arraystretch}{1.3}
    \begin{tabular}{@{}l r r r@{}}
        & CNN & G-CNN & PDE-CNN
        \\
        \hline
        DRIVE 6-layer &  1.7s & 6.5s & 6.8s
        \\
        DRIVE 12-layer & 2.2s & 14.1s & 9.8s
        \\
        RotNIST & 0.1s & 0.9s & 0.7s
    \end{tabular}
    \caption{Time in seconds it took to run each model on the testing dataset of its respective experiment. The DRIVE testing dataset contains 20 images while the RotNIST testing dataset contains 10000 digits.}
    \label{tab:speed}
\end{table}

\section{Conclusion}
\label{sec:concluding}

In this article we presented the general mathematical framework of geometric PDEs on homogeneous spaces that underlies our PDE-G-CNNs. PDE-G-CNNs allow for a geometric and probabilistic interpretation of CNNs opening up new avenues for the study and development of these types of networks. We showed that additionally, PDE-G-CNNs have increased performance with a reduction of parameters.

PDE-G-CNNs ensure equivariance by design. The trainable parameters are geometrically relevant: they are left-invariant vector and tensor fields. 

PDE-G-CNNs have three types of layers: convection, diffusion and erosion/dilation layers. We have shown that these layers implicitly include standard nonlinear operations in CNNs such as max pooling and ReLU activation.

To efficiently evaluate PDE evolution in the layers, we provided tangible analytical approximations to the relevant kernel operators on homogeneous spaces. In this article we have underpinned the quality of the approximations in Theorem~\ref{thm:alpha_kernel_squeeze} and Theorem~\ref{thm:erosion_dilation_solution}.

With two experiments we have verified that PDE-G-CNNs can improve performance over G-CNNs in the context of automatic vessel segmentation and digit classification.
Most importantly, the performance increase is achieved with a vast reduction in the amount of trainable parameters.

\section*{Acknowledgements}

We gratefully acknowledge the Dutch
Foundation of Science NWO for financial support {\small (Duits: Geometric
learning for Image Analysis, VI.C 202–031, Bekkers:
Context Aware AI, VI nr.17290).}

\renewcommand{\appendixname}{Appendix}
\begin{appendices}
\section{}

\subsection{Proof of Lemma \ref{lem:bounding-metric}}
\label{appendix:proof-of-bounding-lemma}

The left inequality follows directly from the observation that $\rho_{\mathcal{G}}(p)$ is exactly the Riemannian length of the curve 
\begin{equation*}
t \mapsto\exp_G(t \log_G(g_p)) p_0
\end{equation*}
for $t \in [0,1]$ and $g_p = \argmin_{g \in p} \left\Vert \log_G g \right\Vert_{\tilde{\mathcal{G}}}$. This continuous curve connects $p_0$ with $p$ and as such has a greater length than the minimal-length curve between those two points.

For the right inequality, consider the function $F:T_e G \to \mathbb{R}$ given by
\begin{equation*}
    F(v)
    :=
    d_{\mathcal{G}} (p_0, \pi \circ \exp_G (v))^2
    ,
\end{equation*}
where we recall that $\pi : G \to G/H$ was given by \eqref{eq:quotientmap}.
With the goal of making a Taylor expansion for this function we note that:
\begin{itemize}
    \item at the origin we have $F(0)=0$,
    \item due to the chain rule applied to the squaring we have $\d F \vert_0 = 0$.
\end{itemize}
Moreover, due to the to the $G$-invariance of $d_{\mathcal{G}}$, the function $F$ is even and consequently the 3rd order term of the Taylor expansion of $F$ is zero.

For the second order term, we are looking for the Hessian $\mathcal{H}$ of $F$ at $v=0$. We split $F$ into $F_1 := \pi \circ \exp_G$ and $F_2(p) := d_{\mathcal{G}}(p_0,p)^2$ and find the Hessian of the composed function is
\begin{equation*}
    \begin{split}
        &\mathcal{H} (F_2 \circ F_1) \vert_0 (v,w)
        \\
        &=
        \mathcal{H} F_2 \vert_{p_0}
        \left(
            \d F_1 \vert_0 (v)
            ,\,
            \d F_1 \vert_0 (w)
        \right)
        \\
        &=
        2\, \mathcal{G} \vert_{p_0}
        \left(
            \d\pi \vert_e \circ \d\exp_G \vert_0 (v)
            ,\,
            \d\pi \vert_e \circ \d\exp_G \vert_0 (w)
        \right)
        \\
        &=
        2\, \tilde{\mathcal{G}} \vert_e
        \left(
            \d\exp_G \vert_0 (v)
            ,\,
            \d\exp_G \vert_0 (w)
        \right)
        \\
        &=
        2\, \tilde{\mathcal{G}} \vert_e
        \left(
            v
            ,\,
            w
        \right)
        .
    \end{split}
\end{equation*}

Putting these facts together we find:
\begin{equation}
    \label{eq:taylorF}
    F(v) = \tilde{\mathcal{G}} \vert_e (v,v) + O (\| v \|^4)
    ,
\end{equation}
where  $\| \cdot \|$ denotes some arbitrary norm on $T_e G$.

Now we take a linear subspace $V$ of $T_e G$ that is independent from $T_e H$ but so that the span of $T_e H$ and $V$ equals the entire $T_e G$, so that $T_e H \oplus V = T_e G$.
Note that $\tilde{\mathcal{G}}\vert_e$ is only degenerated along $T_e H$, and so is a full norm when restricted to $V$, i.e. for all $v \in V$, $\tilde{\mathcal{G}}\vert_e(v,v) = 0$ only if $v=0$.
Therefore, there exists a $c>0$ such that for all $v \in V$,
\begin{equation*}
    \tilde{\mathcal{G}} \vert_e (v,v) > c \| v \|^2
    ,
\end{equation*}
and so by \eqref{eq:taylorF} we have:
\begin{equation}
    \label{eq:F(v)}
    F(v)
    =
    d_{\mathcal{G}} (p_0, \pi \circ \exp_G (v))^2
    >
    c \| v \|^2 + O \left( \| v \|^4 \right)
    .
\end{equation}
Hence, for all $v \in V$ close enough to 0 we have:
\begin{equation}
    \label{eq:d-norm}
    d_{\mathcal{G}} (p_0, \pi \circ \exp_G (v))^2
    >
    \frac{c}{2} \| v \|^2
    .
\end{equation}

In a neighborhood of the origin the Lie group exponential map $\exp_G:T_e G \to G$ is a diffeomorphism to a neighborhood of $e$, at the same time $\pi : G \to G/H$ is a smooth submersion by the Homogeneous Space Construction Theorem \cite[Thm 21.17]{lee2013smooth}.
Consequently $\d (\pi \circ \exp_G)\vert_0:V \to T_{p_0}(G/H)$ has full rank since $\pi \circ \exp_G$ is a local diffeomorphism between two spaces ($V$ and $G/H$) with the same dimension, it follows by the inverse function theorem that there exists a neighborhood $V_0$ of $0$ in $V$ and a neighborhood $P_0$ of $p_0$ in $G/H$ such that $\pi \circ \exp_G$ is a diffeomorphism from $V_0$ to $P_0$.
By possibly choosing $V_0$ smaller, we may assume by \eqref{eq:F(v)} that there exists a $C'>0$ such that for all $v \in V_0$:
\begin{equation*}
    \tilde{\mathcal{G}} \vert_e (v,v)
    \leq
    F(v)
    + C' \, \| v \|^4
    ,
\end{equation*}
which, by using \eqref{eq:d-norm} yields
\begin{equation*}
    \tilde{\mathcal{G}} \vert_e (v,v)
    \leq
    d_{\mathcal{G}} (p_0, \pi \circ \exp_G (v))^2
    +
    C
    d_{\mathcal{G}} (p_0, \pi \circ \exp_G (v))^4
    ,
\end{equation*}
for all $v \in V_0$ and $C=C'\frac{c}{2}>0$.

Now take a $p \in P_0$, then there exists a $w \in V_0$ so that
\begin{equation*}
    \pi \circ \exp_G w = p.
\end{equation*}
Call $g_p = \exp_G w$, then the previous inequality gives
\begin{equation*}
    \left\| \log_G g_p \right\|_{\tilde{\mathcal{G}}}^2
    \leq
    d_{\mathcal{G}}(p,p_0)^2
    +
    C
    d_{\mathcal{G}} (p,p_0)^4\
    .
\end{equation*}
Clearly $g_p \in p$. 
Since $\rho_{\mathcal{G}}(p)$ is the infinum of $\| \log_G g \|$ for all $g \in p$ it follows that $\rho_{\mathcal{G}}(p)$ must also satisfy:
\begin{equation*}
    \rho_{\mathcal{G}}(p)^2
    \leq
    \left\| \log_G g_p \right\|_{\tilde{\mathcal{G}}}^2
    \leq
    d_{\mathcal{G}}(p,p_0)^2
    +
    C
    d_{\mathcal{G}} (p,p_0)^4
    ,
\end{equation*}
for all $p \in P_0$, i.e. all $p$ sufficiently close to $p_0$.

As a corollary we get that for any compact neighborhood $K \subset G/H$ of $p_0$
\begin{equation*}
    \rho_{\mathcal{G}}(p)
    \leq
    C_{\mathrm{metr}} \, d_{\mathcal{G}}(p_0,p)
\end{equation*}
for all $p \in K$.
We can see this by choosing $C_1^2 = 1+C\,\sup_{p \in P_)} d_{\mathcal{G}}(p_0,p)^2$, then for all 
by choosing $p \in P_0$ we have
\begin{equation*}
    \rho_{\mathcal{G}}(p)
    \leq
    C_1 d_{\mathcal{G}}(p_0,p).
\end{equation*}
Let $K \subset G/H$ be compact so that it contains $P_0$. 
Then on $\overline{K \setminus P_0}$ we have that both $\rho_{\mathcal{G}}$ and $d_{\mathcal{G}}(p_0,\cdot)$ are strictly positive, continuous and so bounded functions.
Consequently
\begin{equation*}
    \rho_{\mathcal{G}}(p)
    \leq 
    \sup_{p \in K \setminus P_0} \rho_{\mathcal{G}}(p)
    =:
    M
    <
    \infty,
\end{equation*}
and
\begin{equation*}
    d_{\mathcal{G}}(p_0,p)
    \geq
    \sup_{p \in P_0} d_{\mathcal{G}}(p_0,p)
    =: m > 0,
\end{equation*}
for all $p \in K\setminus P_0$.
Which leads to
\begin{equation*}
    \begin{split}
    \rho_{\mathcal{G}}(p)
    \leq 
    M
    \leq
    \frac{M}{d_{\mathcal{G}}(p_0,p)} d_{\mathcal{G}}(p_0,p)
    \leq
    \frac{M}{m}
    d_{\mathcal{G}}(p_0,p)
    .
    \end{split}
\end{equation*}
for all $p \in K \setminus P_0$.
Now choose $C_{\textrm{metr}} = \max\{ C_1, C_2 \}$ then we obtain the corollary.
Remark that $C_{\textrm{metr}}$ depends on both the parameters of the metric tensor field and the choice of $K$, and so may become very large indeed.

\ \qed

\subsection{Proof of Proposition \ref{prop:compatibility}}
\label{appendix:proof-of-compatibility}

As a preliminary we prove the following lemma.

\textbf{Lemma}
\noindent
For all $g \in G$ let $L_g: G \to G$ be the left group multiplication given by $L_g h = g h$ and let $R_g: G \to G$ be the right group multiplication given by $R_g h = h g$.
Let $H$ be a closed subgroup of $G$ with the projection map $\pi:G \to G/H$ given by $\pi(g)= g H$.

Then for all $h \in H$ we have the following relations for the push forwards:
\begin{enumerate}
    \item $\pi_* \circ \left( R_h \right)_* = \pi_*$,
    \item $\left( L_h \right)_* \circ \pi_* = \pi_* \circ \left( L_h \right)_*$.
\end{enumerate}

\begin{proof}
    \ 
    \begin{enumerate}
        \item $\pi \circ R_h = \pi$ since $g h H = g H$,
        \item $L_h \circ \pi = \pi \circ L_h$ since $h ( g H)= (h g) H$.
    \end{enumerate}
    \qed
\end{proof}

Now for the proof of Proposition~\ref{prop:compatibility}.
Consider the set of all exponential curves in the group whose action connects $p_0 \in G/H$ to $p \in G/H$:
\begin{equation*}
    \begin{split}
    \Gamma_{p_0,p}
    =
    \Big\{
        & \gamma \in \mathrm{Lip}([0,1],G)
        \ \Big\vert \ 
        \\
        & \gamma(0)=e
        ,\ 
        \gamma(1) p_0 = p
        ,\ 
        \gamma(t+s)=\gamma(t) \gamma(s)
    \Big\}
    .
    \end{split}
\end{equation*}
We can then restate $\rho_{\mathcal{G}}$ equivalently in terms of these curves as
\begin{equation*}
    \rho_{\mathcal{G}}(p)
    :=
    \inf_{g \in P}
    \| \log_G g \|_{\tilde{\mathcal{G}}}
    =
    \inf_{\gamma \in \Gamma_{p_0,p}}
    \| \dot{\gamma}(0) \|_{\tilde{\mathcal{G}}}
\end{equation*}
since for each $g \in p$ we have an exponential curve $t \mapsto \exp_G(t \, \log_G g)$ in $\Gamma_{p_0,p}$ and for each exponential curve $\gamma$ in $\Gamma_{p_0,p}$ we have $\gamma(1) \in p$.

Let $\gamma \in \Gamma_{p_0,p}$ and let $h \in H$ then
\begin{enumerate}
    \item $h \gamma(0) h^{-1} = h e h^{-1} =e$,
    \item $h \gamma(1) h^{-1} p_0 =h \gamma(1) p_0 = h p$,
    \item $h \gamma(a+b) h^{-1} = h \gamma(a) \gamma(b) h^{-1}=(h \gamma(a) h^{-1})(h \gamma(b) h^{-1})$.
\end{enumerate}
From which we conclude that $h \gamma(\cdot) h^{-1} \in \Gamma_{p_0,h p}$ and so there is a bijection between $\Gamma_{p_0,p}$ and $\Gamma_{p_0,h p}$ given by
\begin{equation*}
    \Gamma_{p_0, h  p}
    =
    h \Gamma_{p_0,p} h^{-1}
    .
\end{equation*}

Moreover, the bijection preserves the seminorm due to the G-invariance of $\mathcal{G}$:
\begin{align*}
    \left\|
        \left( h \gamma(\cdot) h^{-1} \right)(0)
    \right\|_{\tilde{\mathcal{G}}}
    &=
    \left\|
        \left(L_h\right)_*
        \left(R_{h^{-1}}\right)_*
        \dot{\gamma}(0)
    \right\|_{\tilde{\mathcal{G}}}
    \\
    &=
    \left\|
        \pi_*
        \left(L_h\right)_*
        \left(R_{h^{-1}}\right)_*
        \dot{\gamma}(0)
    \right\|_{\mathcal{G}}
\intertext{(using the previous lemma)}
    &=
    \left\|
        \left(L_h\right)_*
        \pi_*
        \left(R_{h^{-1}}\right)_*
        \dot{\gamma}(0)
    \right\|_{\mathcal{G}}
    \\
    &=
    \left\|
        \left(L_h\right)_*
        \pi_*
        \dot{\gamma}(0)
    \right\|_{\mathcal{G}}
\intertext{(using the G-invariance of $\mathcal{G}$)}
    &=
    \left\|
        \pi_*
        \dot{\gamma}(0)
    \right\|_{\mathcal{G}}
    \\
    &=
    \left\|
        \dot{\gamma}(0)
    \right\|_{\tilde{\mathcal{G}}}
    .
\end{align*}

It follows that
\begin{align*}
    \rho_{\mathcal{G}}(p)
    &=
    \inf_{\gamma \in \Gamma_{p_0,p}}
    \| \dot{\gamma}(0) \|_{\tilde{\mathcal{G}}}
    \\
    &=
    \inf_{\gamma \in \Gamma_{p_0,p}}
    \| h \dot{\gamma}(0) h^{-1} \|_{\tilde{\mathcal{G}}}
    \\
    &=
    \inf_{\gamma \in \Gamma_{p_0,h p}}
    \| \dot{\gamma}(0) \|_{\tilde{\mathcal{G}}}
    \\
    &=
    \rho_{\mathcal{G}}(h p).
\end{align*}

\ \qed
\end{appendices}

\printbibliography

\end{document}